%% file: arXiV_main.tex
\documentclass{article}

\usepackage{arxiv}

\usepackage[utf8]{inputenc} 
\usepackage[T1]{fontenc}    
\usepackage{hyperref}       
\usepackage{url}            
\usepackage{booktabs}       
\usepackage{amsfonts}       
\usepackage{nicefrac}       
\usepackage{microtype}      
\usepackage{lipsum}		
\usepackage{graphicx}
\usepackage{natbib}
\usepackage{doi}

\renewcommand{\cite}{\citep}
\usepackage{times}
\usepackage[utf8]{inputenc} 
\usepackage[T1]{fontenc}    
\usepackage{hyperref}       
\usepackage{url}            
\usepackage{booktabs}       
\usepackage{amsfonts}       
\usepackage{nicefrac}       
\usepackage{caption}
\usepackage{microtype}      
\usepackage{xcolor}         
\usepackage{algorithm}
\usepackage{algpseudocode}
\usepackage[english]{babel}
\usepackage{amsmath,amsthm,amssymb}
\usepackage{xspace}
\usepackage{graphicx}
\usepackage{subcaption}
\newtheorem{theorem}{Theorem}[section]

\newtheorem{proposition}[theorem]{Proposition}

\usepackage{multirow}

\usepackage{xspace}
\usepackage{marginnote}

\usepackage{color}

\usepackage{cancel}

\usepackage{soul}
\setul{0ex}{}

\newcommand{\tr}[1]{\textcolor{blue}{TR:  #1}}
\newcommand{\trr}[1]{\marginpar[]{\textcolor{blue}{TR:  #1}}}
\newcommand{\vg}[1]{\marginpar[]{\textcolor{red}
{VG:  #1}}}
\newcommand{\sva}[1]{\textcolor{brown}{SVA: }\textcolor{teal}{ #1 }}
\newcommand{\nips}[1]{\textcolor{red}{\cancel{#1}}}
\newcommand{\svam}[1]{\marginpar[]{\textcolor{brown}{SVA:}\textcolor{teal}{#1}}}
\newcommand{\svar}[1]{{[\color{olive} {\color{purple}\textbf{Shivvrat}}: #1] }}

\renewcommand{\tr}[1]{}
\renewcommand{\trr}[1]{}
\renewcommand{\vg}[1]{}
\renewcommand{\sva}[1]{}
\renewcommand{\nips}[1]{}
\renewcommand{\svam}[1]{}

\newcommand{\cmpe}{\texttt{CMPE}\xspace}
\newcommand{\mincmpe}{Minimization Version of CMPE\xspace}

\newcommand{\eat}[1]{}

\newcommand{\ilp}{ILP\xspace}
\newcommand{\mae}{MAE\xspace}
\newcommand{\maepen}{MAE+Penalty\xspace}
\newcommand{\mse}{MSE\xspace}
\newcommand{\msepen}{SL+Penalty\xspace}
\newcommand{\ssl}{SSL+Penalty\xspace}
\newcommand{\pdl}{\texttt{PDL}\xspace}
\newtheorem{example}{Example}

\newcommand{\slpen}{\texttt{SL}$_{pen}$\xspace}
\newcommand{\sslpen}{\texttt{SSL}$_{pen}$\xspace}
\newcommand{\ours}{\sscmpe\xspace}
\newcommand{\sscmpe}{\texttt{SS-CMPE}\xspace}
\newcommand{\sscmpepen}{\sscmpe$_{pen}$\xspace}
\usepackage[normalem]{ulem}

\useunder{\uline}{\ul}{}
\makeatletter
    \newcommand{\coloruline}[2]{%
        \newcommand\temp@reduline{\bgroup\markoverwith
            {\textcolor{#1}{\rule[-0.5ex]{2pt}{1.0pt}}}\ULon}%
        \temp@reduline{#2}%
    }
\makeatother

\title{Learning to Solve the Constrained Most Probable Explanation Task in Probabilistic Graphical Models}

\author{Shivvrat Arya$^*$ \\
	Department of Computer Science\\
	The University of Texas at Dallas\\
	Richardson, TX 75252 \\
	\texttt{shivvrat.arya@utdallas.edu}
	\And
	Tahrima Rahman$^*$ \\
	Department of Computer Science\\
	The University of Texas at Dallas\\
	Richardson, TX 75252 \\
	\texttt{tahrima.rahman@utdallas.edu} 
 	\And
	Vibhav Gogate \\
	Department of Computer Science\\
	The University of Texas at Dallas\\
	Richardson, TX 75252 \\
\texttt{vibhav.gogate@utdallas.edu}  
}

\date{}

\renewcommand{\shorttitle}{Learning to Solve the CMPE Task in PGMs}

\hypersetup{
pdftitle={Learning to Solve the Constrained Most Probable Explanation Task in Probabilistic Graphical Models},
pdfsubject={cs.AI, cs.LG},
pdfauthor={Shivvrat Arya, Tahrima Rahman,  Vibhav Gogate},
pdfkeywords={Probabilistic Modeling, Constrained Inference, Neuro-Symbolic Inference},
}

\begin{document}
\maketitle

\begin{abstract}
We propose a \textit{self-supervised learning} approach for solving the following constrained optimization task in log-linear models or Markov networks. Let $f$ and $g$ be two log-linear models defined over the sets $\mathbf{X}$ and $\mathbf{Y}$ of random variables respectively. Given an assignment $\mathbf{x}$ to all variables in $\mathbf{X}$ (evidence) and a real number $q$, the constrained most-probable explanation (CMPE) task seeks to find an assignment $\mathbf{y}$ to all variables in $\mathbf{Y}$ such that $f(\mathbf{x}, \mathbf{y})$ is maximized and $g(\mathbf{x}, \mathbf{y})\leq q$. 
In our proposed self-supervised approach, given assignments $\mathbf{x}$ to $\mathbf{X}$ (data), we train a deep neural network that learns to output near-optimal solutions to the CMPE problem without requiring access to any pre-computed solutions. The key idea in our approach is to use first principles and approximate inference methods for CMPE to derive novel loss functions that seek to push infeasible solutions towards feasible ones and feasible solutions towards optimal ones. We analyze the properties of our proposed method and experimentally demonstrate its efficacy on several benchmark problems. 
\end{abstract}

\section{INTRODUCTION}
\def\thefootnote{*}\footnotetext{These authors contributed equally to this work.}

\label{sec:intro}
\input{content/1-intro}

\input{content/notation_and_background}
\eat{
\section{PRELIMINARIES}
\label{sec:prelim}
\input{content/3-preliminary}

}
\section{ \label{sec:our-method}A NOVEL SELF-SUPERVISED CMPE SOLVER}
\input{content/4-method}

\section{EXPERIMENTAL EVALUATION}
\label{sec:experiments}
\input{content/5-experiment}

\eat{\section{RELATED WORK}
\label{section:related_work}
\input{content/2-related-work} 
}
\section{CONCLUSION AND FUTURE WORK}
\label{sec:conclusion}
\input{content/6-conclusion}

\section*{ACKNOWLEDGMENTS}
This work was supported in part by the DARPA Perceptually-Enabled Task Guidance (PTG) Program under contract number HR00112220005, by the DARPA Assured Neuro Symbolic Learning and Reasoning (ANSR) under contract number HR001122S0039 and by the National Science Foundation grant IIS-1652835.

\bibliography{main}

\clearpage
\input{supplement_content}

\end{document}

%% file: content/1-intro.tex
      Probabilistic graphical models (PGMs) such as Bayesian and Markov networks \cite{koller2009probabilistic,darwiche09} compactly represent joint probability distributions over random variables by factorizing the distribution according to a graph structure that encodes conditional independence among the variables. Once learned from data, these models can be used to answer various queries, such as computing the marginal probability distribution over a subset of variables (MAR) and finding the most likely assignment to all unobserved variables, which is referred to as the most probable explanation (MPE) task.

      Recently, \citet{rouhani2020novel} proposed an extension to the MPE task in PGMs by introducing constraints. More specifically, given two PGMs $f$ and $g$ defined over the set of random variables $\mathbf{X}$ and a real number $q$, the constrained most probable explanation (CMPE) task seeks to find the most likely state $\mathbf{X}=\mathbf{x}$ w.r.t. $f$ such that the constraint $g(\mathbf{x})\leq q$ is satisfied. Even though both MPE and CMPE are NP-hard in general, CMPE is considerably more difficult to solve in practice than MPE. Notably, CMPE is NP-hard even on PGMs having no edges, such as zero treewidth or independent PGMs, while MPE can be solved in linear time. \citet{rouhani2020novel} and later \citet{rahman2021novel} showed that several probabilistic inference queries are special cases of CMPE. This includes queries such as finding the decision preserving most probable explanation \cite{choi_same-decision_2012}, finding the nearest assignment \cite{rouhaniRG18} and robust estimation \cite{darwiche2023complete,darwiche_reasons_2020}.   

    Our interest in the CMPE task is motivated by its extensive applicability to various \textit{neuro-symbolic inference} tasks. Many of these tasks can be viewed as specific instances of CMPE. Specifically, when $f(\mathbf{x})$ represents a function encoded by a neural network and $g(\mathbf{x}) \leq q$ signifies particular symbolic or weighted constraints that the neural network must adhere to, the neuro-symbolic inference task involves determining the most likely prediction with respect to $f$ while ensuring that the constraint $g(\mathbf{x}) \leq q$ is satisfied. Another notable application of CMPE involves transferring abstract knowledge and inferences from simulations to real-world contexts. For example, in robotics, numerous simulations can be employed to instruct the robot on various aspects, such as object interactions, robot-world interactions, and underlying physical principles, encapsulating this abstract knowledge within the constraint $g(\mathbf{x}) \leq q$. Subsequently, with a neural network $f$ trained on a limited amount of real-world data, characterized by richer feature sets and objectives, $g$ can be used to reinforce the predictions made by $f$, ensuring that the robot identifies the most likely prediction with respect to $f$ while satisfying the constraint $g(\mathbf{x}) \leq q$. This strategy enhances the reliability of the robot's predictions and underscores the practical significance of CMPE.


      
      

      \eat{Inference queries like MPE and CMPE, which require optimizations over the states of random variables} 

In this paper, we explore novel machine learning (ML) approaches for solving the CMPE task, drawing inspiration from recent developments in \textit{learning to optimize} \cite{donti2021dc3,fioretto2020predicting,park2022self,zamzam2019learning}. The main idea in these works is to train a deep neural network that takes the parameters, observations, etc. of a constrained optimization problem \textit{as input} and outputs a near-optimal solution to the optimization problem. 


In practice, a popular approach for solving optimization problems is to use search-based solvers such as Gurobi and SCIP. However, a drawback of these off-the-shelf solvers is their inability to efficiently solve large problems, especially those with dense global constraints, such as the CMPE problem. In contrast, neural networks are efficient because once trained, the time complexity of solving an optimization problem using them scales linearly with the network's size. This attractive property has also driven their application in solving probabilistic inference tasks such as MAR and MPE inference \cite{gilmer2017neural, kuck2020belief, zhang2020factor, satorras2021neural}. However, all of these works require access to exact inference techniques in order to train the neural network. As a result, they are feasible only for small graphical models on which exact inference is tractable. Recently, \citet{cui2022variational} proposed to solve the MPE task by training a variational distribution that is parameterized by a neural network in a self-supervised manner (without requiring access to exact inference methods). To the best of our knowledge, there is no prior work on using neural networks for solving the CMPE problem.

 In this paper, we propose a new self-supervised approach for training neural networks which takes observations or evidence as input and outputs a near optimal solution to the CMPE task. Existing self-supervised approaches \cite{fioretto2020predicting,park2022self} in the \textit{learning to optimize} literature either relax the constrained objective function using Lagrangian relaxation and then use the Langragian dual as a loss function or use the Augmented Lagrangian method.
 We show that these methods can be easily adapted to solve the CMPE task. Unfortunately, an issue with them is that an optimal solution to the Lagrangian dual is not guaranteed to be an optimal solution to the CMPE task (because of the non-convexity of CMPE, there is a duality gap). To address this issue, we propose a new loss function based on first principles and show that an optimal solution to the loss function is also an optimal solution to the CMPE task. Moreover, our new loss function has several desirable properties, which include: (a) during training, when the constraint is violated, it focuses on decreasing the strength of the violation, and (b) when constraints are not violated, it focuses on increasing the value of the objective function associated with the CMPE task. 
      

       \eat{\cite{gilmer2017neural, kuck2020belief, zhang2020factor, satorras2021neural} proposed neural networks that can perform probabilistic inference tasks like marginal and MPE inference. These methods require exact inference results to train the neural network to perform inference on future probabilistic queries. As a result these methods are feasible only for small graphical models from which exact results are obtainable. Recently, \cite{cui2022variational} proposed to solve MPE by training a neural network to approximate the parameters of a variational distribution using which an MPE state will be exactly be found. According to their method a neural network is trained to minimize a variational free energy without requiring any exact inference results. To the best of our knowledge, there is no prior work related to solve CMPE using neural networks.}  
      
\eat{    In summary, this paper makes the following contributions 
    \begin{itemize}
        \item To the best of our knowledge, this work is the first to explore machine learning approaches to solve constrained optimization problems in PGMs and more generally for multilinear optimization. 
        \item We propose a self-supervised approach for training a deep neural network to solve the CMPE task. At the core of our method is a novel loss function that is designed using first principles, has the same set of global optima as the CMPE task and several desirable properties. 
        \item Empirically, we evaluate several supervised and self-supervised approaches and compare them to our proposed method. To the best of our knowledge, these are the first empirical results on using machine learning, either supervised or self-supervised, to solve the CMPE task in PGMs. On a number of benchmark models, we show that neural networks trained using our proposed loss function are more efficient and accurate compared to models trained to minimize competing supervised and self-supervised loss functions from literature.
    \end{itemize}

}  

We conducted a comprehensive empirical evaluation, comparing several supervised and self-supervised approaches to our proposed method. To the best of our knowledge, these are the first empirical results on using machine learning, either supervised or self-supervised, to solve the CMPE task in PGMs. On a number of benchmark models, our experiments show that neural networks trained using our proposed loss function are more efficient and accurate compared to models trained to minimize competing supervised and self-supervised loss functions from the literature.

%% file: content/notation_and_background.tex
\section{Notation and Background}
We denote random variables by upper-case letters (e.g., $X$, $Y$, $Z$, etc.), their corresponding assignments by lower-case letters (e.g., $x$, $y$, $z$, etc.), sets of random variables by bold upper-case letters (e.g., $\mathbf{X}$, $\mathbf{Y}$, $\mathbf{Z}$, etc.) and assignments to them by bold lower-case letters (e.g., $\mathbf{x}$, $\mathbf{y}$, $\mathbf{z}$, etc.). $\mathbf{z}_{\mathbf{X}}$ denotes the projection of the complete assignment $\mathbf{z}$ on to the subset $\mathbf{X}$ of $\mathbf{Z}$. For simplicity of exposition, we assume that discrete and continuous random variables take values from the set \{0,1\} and [0,1] respectively. 


We use the multilinear polynomial representation \citep{Sherali&Adams09,sherali1992global,horst1996global} to concisely describe our proposed method as well as for specifying discrete, continuous, and mixed constrained optimization problems. Let $\mathbf{Z}=\{Z_1,\ldots,Z_n\}$ be a set of random variables. Let $[n]=\{1,\ldots,n\}$ and $i \in [n]$ be an index over the variables of $\mathbf{Z}$. Let $2^{[n]}$ denote the set of subsets of indices of $[n]$; thus each element of $2^{[n]}$ denotes a (unique) subset of $\mathbf{Z}$. Let $\mathcal{I} \subseteq 2^{[n]}$ and let $w_I \in \mathbb{R}$ where $I \in \mathcal{I}$ be a real number (weight) associated with each element $I$ of $\mathcal{I}$. Then, a multilinear polynomial is given by
\begin{align}
    \label{eq:mult-poly}
    f(\mathbf{z}) = f(z_1,\ldots,z_n)= \sum_{I \in \mathcal{I}} w_I \prod_{i \in I}z_i
\end{align}
where $\mathbf{z}=(z_1,\ldots,z_n)$ is an assignment to all variables in $\mathbf{Z}$. We will call $f(\mathbf{z})$ the weight of  $\mathbf{z}$.
\eat{
\begin{figure*}[t]
\begin{center}
{\footnotesize
    \centering
    \begin{tabular}{|c|c|c|}
    \hline
     $X_1$ & $X_2$ & $f_{1,2}(x_1,x_2)$   \\
     \hline
     0 & 0 & 3 \\
     0 & 1 & 4 \\
     1 & 0 & 5 \\
     1 & 1 & 8 \\
     \hline
     \multicolumn{3}{c}{(a)}\\
    \end{tabular}
    ~~
    \begin{tabular}{|c|c|c|}
    \hline
     $X_2$ & $X_3$ & $f_{2,3}(x_2,x_3)$   \\
     \hline
     0 & 0 & 4 \\
     0 & 1 & 2 \\
     1 & 0 & 9 \\
     1 & 1 & 1 \\
     \hline
     \multicolumn{3}{c}{(b)}\\
    \end{tabular}
    ~~
    \begin{tabular}{c}
    $f_{1,2}(x_1,x_2)=3(1-x_1)(1-x_2)+$\\
    $4(1-x_1)(x_2)+5(x_1)(1-x_2)+$ 
    $8(x_1)(x_2)$\\
    $f_{2,3}(x_2,x_3)=4(1-x_2)(1-x_3)+$\\
    $2(1-x_2)(x_3)+9(x_2)(1-x_3) + 1(x_2)(x_3)$\\
    \\
    (c)\\
    \end{tabular}
}
    \caption{Two log-potentials (a) $f_{1,2}(x_1,x_2)$ and (b) $f_{2,3}(x_2,x_3)$ associated with a Markov network defined over three variables $\{X_1,X_2,X_3\}$. The multilinear representation for the log-potentials is given in (c). The reader can verify that the weight of a full assignment $(x_1,x_2,x_3)$ w.r.t. the given log-linear model can be computed using the following multilinear polynomial $f(x_1,x_2,x_3)=f_{1,2}(x_1,x_2)+f_{2,3}(x_2,x_3)=7+2x_1+6x_2-2x_3+2x_1x_2-6x_2x_3$ (obtained by adding the two polynomials and simplifying).}
    \label{fig:pgm-example}
 \vspace{-0.1in}
\end{center}
\end{figure*}
}

\begin{figure}[t]
    \begin{center}
\includegraphics[scale=0.6]{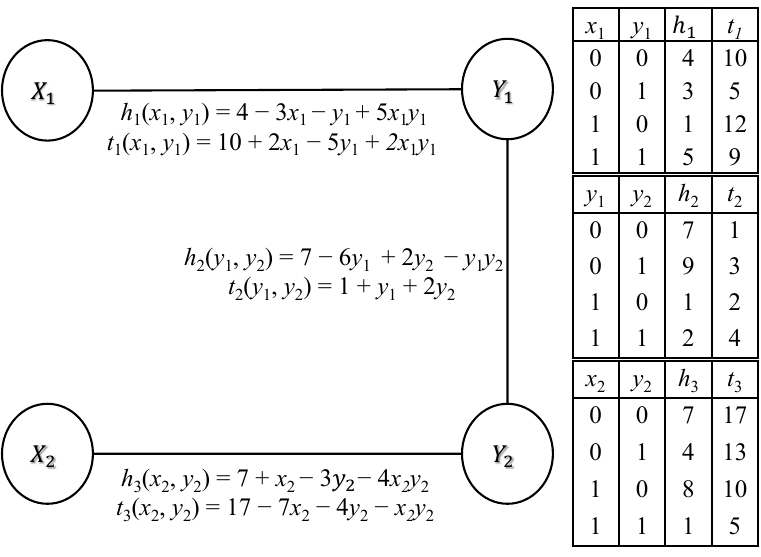} 
    \end{center}
    \caption{\label{fig:poly}Two Markov networks $\mathcal{M}_1$ and $\mathcal{M}_2$ having the same chain-like structure and defined over the same set $\{X_1,X_2,Y_1,Y_2\}$ of variables . $\mathcal{M}_1$ is defined by the set of log-potentials $\{h_1,h_2,h_3\}$ and $\mathcal{M}_2$ is defined by the set of log-potentials $\{t_1,t_2,t_3\}$. Each log-potential can be expressed as a local multilinear polynomial function. The global multilinear function representing $\mathcal{M}_1$ and $\mathcal{M}_2$ are $h(x_1,x_2,y_1,y_2) = 18-3x_1+x_2-7y_1-y_2+5x_1y_1-4x_2y_2-y_1y_2$ and $t(x_1,x_2,y_1,y_2) = 28+2x_1-7x_2-4y_1-2y_2+2x_1y_1-x_2y_2$ respectively which are obtained by adding the local functions associated with the respective models and then simplifying, i.e., $h(x_1,x_2,y_1,y_2) = h_1(x_1,y_1)+h_2(y_1,y_2)+h_3(x_2,y_2)$. $t(x_1,x_2,y_1,y_2)$ is obtained similarly.} 
\end{figure}

It is known that weighting functions, namely the sum of log of conditional probability tables and log-potentials associated with Bayesian and Markov networks respectively can be expressed as multilinear polynomials (see for example \cite{koller2009probabilistic}). \eat{Figure \ref{fig:pgm-example} shows a multilinear representation for a Markov network having two potentials and three variables.}
\begin{example}
    \eat{The multilinear functions representing the Markov networks $\mathcal{M}_1$ and $\mathcal{M}_2$ in figure \ref{} are $h(X_1,X_2,Y_1,Y_2) = 18-3X_1+X_2-7Y_1-Y_2+5X_1Y_1-4X_2Y_2-Y_1Y_2$ and $t(X_1,X_2,Y_1,Y_2) = 28+2X_1-7X_2-4Y_1-2Y_2+2X_1Y_1-X_2Y_2$ respectively which are obtained by adding the multilinear polynomials corresponding to the log-potential functions associated with the respective models and then simplifying, i.e., $h(X_1,X_2,Y_1,Y_2) = h_1(X_1,Y_1)+h_2(Y_1,Y_2)+h_3(X_2,Y_2)$. $t(X_1,X_2,Y_1,Y_2)$ is obtained similarly.} Figure \ref{fig:poly} shows a multilinear representation for a Markov network. The weight of the assignment $(X_1=0,X_2=1,Y_1=0,Y_2=1)$ is 14 and 18 w.r.t. $\mathcal{M}_1$ and $\mathcal{M}_2$ respectively.
\end{example}

\eat{
\textbf{Example:}    Consider a log-linear model having three variables $\{X_1,X_2,X_3\}$ and two log-potentials $h_{1,2}(X_1,X_2)=(v_{0,0},v_{0,1},v_{1,0},v_{1,1})$ and $h_{2,3}(X_2,X_3)=(w_{0,0},w_{0,1},w_{1,0},w_{1,1})$ where $v_{i,j}$, $i \in \{0,1\}$ and $j \in \{0,1\}$ (similarly $w_{ij}$) denotes the weight associated with the assignment $(X_1=i,X_2=j)$ (similarly $(X_2=i,X_3=j)$). The weight associated with the assignment $(x_1,x_2,x_3)$ to all the variables is given by
    \begin{align}
    \nonumber
        h(x_1,x_2,x_3)= \prod_{i\in \{0,1\}} \prod_{j\in \{0,1\}} v_{i,j}x_1^i (1-x_1)^{1-i }x_2^j(1-x_2)^{1-j} +\\
    \prod_{i\in \{0,1\}} \prod_{j\in \{0,1\}} w_{i,j}x_2^i (1-x_2)^{1-i }x_3^j(1-x_3)^{1-j}
    \end{align} 
The above expression is a multilinear polynomial because the powers are either zero or one. 
}

\subsection{Constrained Most Probable Explanation}
We are interested in solving the following constrained most probable explanation (CMPE) task. Let $\mathbf{X}$ and $\mathbf{Y}$ be two subsets of $\mathbf{Z}$ such that $\mathbf{Z}=\mathbf{X} \cup \mathbf{Y}$ and $\mathbf{X} \cap \mathbf{Y} =\emptyset$. We will refer to $\mathbf{Y}$ as \textit{decision variables} and $\mathbf{X}$ as \textit{evidence variables}. Given assignments $\mathbf{x}$ and $\mathbf{y}$, let $(\mathbf{x},\mathbf{y})$ denote their composition. Let $h$ and $t$ denote two multilinear polynomials over $\mathbf{Z}$ obtained from two Markov networks $\mathcal{M}_1$ and $\mathcal{M}_2$ respectively that represent two (possibly different) joint probability distributions over $\mathbf{Z}$. Then given a real number $q$ and an assignment $\mathbf{x}$ to all variables in $\mathbf{X}$, the CMPE task is to find an assignment $\mathbf{y}^*$ to all the variables in $\mathbf{Y}$ such that $h(\mathbf{x},\mathbf{y}^*)$ is maximized (namely the probability of the assignment w.r.t. $\mathcal{M}_1$ is maximized) and $t(\mathbf{x,y}^*)\leq q$ (namely the probability of the assignment w.r.t. $\mathcal{M}_2$ is bounded by a constant). Formally, 
\eat{
\sva{
\begin{equation}\label{eqn:cmpe}
    P(\mathbf{x}|f,h,t): \underset{\mathbf{y}}{\text{maximize}} \:\: h(\mathbf{x},\mathbf{y}) \:\: \text{s.t.} \:\: t(\mathbf{x}, \mathbf{y}) \leq q
\end{equation}}

\sva{Replaced the $P(\mathbf{x}|f,g,q)$ with $P(\mathbf{x}|h,t,q)$ in the equation \ref{eqn:cmpe}.}
}
\begin{equation}\label{eqn:cmpe}
 \underset{\mathbf{y}}{\text{maximize}} \:\: h(\mathbf{x, y}) \:\: s.t. \:\: t(\mathbf{x, y})\leq q
\end{equation}
\eat{For brevity, we will abuse notation and use $\mathcal{P}(\mathbf{x})$ instead of \eat{$\mathcal{P}(\mathbf{x}|f,g,q)$} $\mathcal{P}(\mathbf{x}, h, t)$. } For brevity, we will abuse notation and use $h_{\mathbf{x}}(\mathbf{y})$ and $t_{\mathbf{x}}(\mathbf{y})$ to denote $h(\mathbf{x,y})$ and $t(\mathbf{x,y})$ respectively. The most probable explanation (MPE) task in probabilistic graphical models \citep{koller2009probabilistic} is a special case of CMPE; MPE is just CMPE without the constraint $t_{\mathbf{x}}(\mathbf{y}) \leq q$. The goal in MPE is to find an assignment $\mathbf{y}^*$ to $\mathbf{Y}$ such that the weight $h_{\mathbf{x}}(\mathbf{y}^*)$ of the assignment is maximized given evidence  $\mathbf{x}$. Similar to MPE, CMPE is NP-hard in general, with the caveat that CMPE is much harder than MPE. Specifically, CMPE is NP-hard even on independent graphical models (having zero treewidth), where MPE can be solved in linear time by independently maximizing each univariate function \citep{rouhani2020novel}. 
\begin{example}
Given $X_1 = 1$, $X_2=1$ and $q=20$, the CMPE solution of the example problem in figure 1 is $(y_1^*, y_2^*)=(0,1)$ with a value $h(1,1,0,1) = 11$, whereas the MPE solution is $(y_1^*, y_2^*)=(0,0)$ with value $h(1,1,0,0)=16$.
\end{example}

\eat{
Applications of CMPE include learning robust classifiers (see for example \cite{bertsimas2019robust,madry2017towards}) where $h$ is a distance function and $t$ is a logistic regression model or a neural network; generating explanations that conform with a user's mental model where $h$ is a generative model learned from data, and $t$ is a user model \cite{rahman2021novel,gunning_darpas_2019}; and generating nearest assignments where $h$ is a log-linear model, $h$ is equal to $t$ and $q$ is the log-probability of the desired nearest assignment \cite{rouhaniRG18}. 
}

\eat{
\textbf{Remarks:}\begin{itemize}
    \item If all variables in CMPE take values from the domain $\{0,1\}$ (or discrete in general), the CMPE task can be expressed as an equivalent integer linear programming (ILP) task by introducing auxiliary integer variables for each $\prod_{i \in I}z_i$ term and adding appropriate constraints to model equivalence between the auxiliary variables and the term. In other words, CMPE can be solved optimally using ILP techniques. 
    \item If CMPE has continuous variables, it is not possible to reduce CMPE to an equivalent ILP/LP. In other words, in general, a CMPE task having continuous variables (and has $i \in I$ such that $|i| \geq 2$) cannot be solved optimally using ILP/LP techniques.
\end{itemize}
}

\eat{
A log-linear model $\mathcal{M}_f = \langle \mathbf{Z}, f \rangle$ represents a joint probability distribution over a set of random variables $\mathbf{Z}$ using a set of log-potential functions, $f=\{f_1,\ldots,f_n\}$. Each potential function is defined over a subset of variables $S(f_i)$ of $\mathbf{Z}$ called its \textit{scope}. Each $f_i:\mathbf{z}_{S(f_i)}\rightarrow \mathbb{R}$ takes an assignment $\mathbf{z}_{S(f_i)}$ to the variables in its scope and maps it to a real value. The unnormalized probability of any complete assignment $\mathbf{z}$ is given by 
\begin{equation*}\label{eqn:log-linear}
    P(\mathbf{z}) \propto \exp \left (\sum_i f_i(\mathbf{z}_{S(f_i)}) \right ) 
\end{equation*}
\eat{We call $\sum_i f_i(\mathbf{z}_{S(f_i)})$ the weight $\omega_{\mathcal{M}_f}(\mathbf{z})$  of the assignment $\mathbf{z}$ w.r.t. the model $\mathcal{M}_f$. }For notational convenience, we write $f(\mathbf{z})$ to denote $\sum_i f_i(\mathbf{z}_{S(f_i)})$. \eat{and $\omega_{f}(\mathbf{z})$ to denote $\omega_{\mathcal{M}_f}(\mathbf{z})$.}
 Let $\mathbf{X}$ and $\mathbf{Y}$ be two subsets of $\mathbf{Z}$ such that $\mathbf{Z}=\mathbf{X} \cup \mathbf{Y}$ and $\mathbf{X} \cap \mathbf{Y} =\emptyset$. Given a log-linear model defined over the random variables $\mathbf{Z}$, most probable explanation or MPE, is the inference task to find an assignment $\mathbf{y}^*$ to all variables in $\mathbf{Y}$ given an observation (evidence) $\mathbf{x}$ over $\mathbf{X}$ such that $P(\mathbf{x}, \mathbf{y}^*)$ \eat{or $\omega_f(\mathbf{x},\mathbf{y}^*)$} is maximized. MPE is NP-hard in general, but can be solved in polynomial time on certain classes of  probabilistic circuits \cite{choi2020probabilistic} (typically having no latent variables) such as cutset networks \cite{rahman2014cutset}, AND/OR graphs having small context size \cite{dechter2007and} and low tree-width Bayesian and Markov networks \cite{bach2001thin}.  

\subsection{Constrained Most Probable Explanation}
Constrained most probable explanation or CMPE is an MPE task with constraints which restrict the values the decision variables in $\mathbf{Y}$ can take. Given two log-linear models $\mathcal{M}_f$ and $\mathcal{M}_g$ having log-potentials $f$ and $g$ respectively, such that both of them are defined over the same sets of variables $\mathbf{X}$ and $\mathbf{Y}$, an observation $\mathbf{x}$ over $\mathbf{X}$, and a constant $q \in \mathbb{R}$, the CMPE task is to find the most likely assignment $\mathbf{y}^*$ to all the variables $\mathbf{Y}$ w.r.t. $\mathcal{M}_f$ such that the constraint $g(\mathbf{x,y}^*)\leq q$ is satisfied w.r.t. $\mathcal{M}_g$. Formally, it is the following maximization problem.
\begin{equation}\label{eqn:cmpe}
    \underset{\mathbf{y}}{\text{maximize}} \:\: f(\mathbf{x,y}) \:\: s.t. \:\: g(\mathbf{x, y}) \leq q
\end{equation}
For brevity, we will denote $f(\mathbf{x, y})$ as $f_\mathbf{x}(\mathbf{y})$ and $g(\mathbf{x, y})$ as $g_\mathbf{x}(\mathbf{y})$. We will often use  $f$ and $g$ in place of $\mathcal{M}_f$ and $\mathcal{M}_g$ when it is clear from the context. Constraints over the decision variables in $\mathbf{Y}$ can be specified in a variety of ways. In our experiments, we have used both log-linear models and neural network classifiers as $g$. Here, we assume that the functions in $f$ is strictly positive. \tr{Add a figure may be?}
}
Since we are interested in machine learning approaches to solve the CMPE task and such approaches employ loss functions, it is convenient to express CMPE as a minimization task with a "$\leq 0$" constraint. This can be accomplished by negating $h$ and subtracting $q$ from $t$. Formally, let $f_{\mathbf{x}}(\mathbf{y})=-h_{\mathbf{x}}(\mathbf{y})$ and $g_{\mathbf{x}}(\mathbf{y})=t_{\mathbf{x}}(\mathbf{y}) - q$. Then Eq. \eqref{eqn:cmpe} is equivalent to the following minimization problem:
 \begin{align} 
 \label{eqn:cmpe_min}
 \eat{Q(\mathbf{x}|f,g,q):}\underset{\mathbf{y}}{\text{minimize}} \:\: f_{\mathbf{x}}(\mathbf{y}) \:\:s.t.\:\: g_{\mathbf{x}}(\mathbf{y}) \leq 0
 \end{align}

Let $\mathbf{y}^*$ be the optimal solution to the problem given in Eq. \eqref{eqn:cmpe_min} and let $p^*_{\mathbf{x}}=f_{\mathbf{x}}(\mathbf{y}^*)$. Also, without loss of generality, we assume that $f_{\mathbf{x}}$ is strictly positive, i.e., $\forall \mathbf{y}$, $f_{\mathbf{x}}(\mathbf{y}) > 0$. 

\eat{It is important to note that any upper bounding method described in section \ref{sec:cmpe_ub}, for example, the methods proposed by \cite{rahman2021novel} can readily be applied to obtain lower bound approximations to the above minimization problem. \tr{We will be using both of these formulations in our proposed method described in section 4.}} 

\eat{\subsection{Reformulating CMPE as a Loss Function}
\sva{Given that machine learning methods frequently employ loss functions, it is convenient to express CMPE as a minimization task (\mincmpe) with a ``$\leq 0$ constraint''. This can be accomplished by negating $h$ and subtracting $q$ from $t$. Formally, let $f(\mathbf{x},\mathbf{y})=-h(\mathbf{x},\mathbf{y})$ and $g(\mathbf{x},\mathbf{y})=t(\mathbf{x},\mathbf{y}) -q$.  For brevity, we will abuse notation and use $f_{\mathbf{x}}(\mathbf{y})$ and $g_{\mathbf{x}}(\mathbf{y})$ to denote $f(\mathbf{x},\mathbf{y})$ and $g(\mathbf{x},\mathbf{y})$ respectively. Then Eq. \eqref{eqn:cmpe} is equivalent to:
\sva{Defined the following minimization problem as $Q(\mathbf{x}|f,g,q)$}
 \begin{align} 
 \label{eqn:cmpe_min}
 Q(\mathbf{x}|f,g,q):\underset{\mathbf{y}}{\text{minimize}} \:\: f_{\mathbf{x}}(\mathbf{y}) \:\:s.t.\:\: g_{\mathbf{x}}(\mathbf{y}) \leq 0
 \end{align}}

 \sva{Given the formulation of \cmpe (and \mincmpe), we can ensure that $f$ is always positive and $g$ can assume any real value. It is noteworthy that a multi-linear objective function can be re-parameterized to be strictly positive without altering the relative ordering or rank of its solutions. To be explicit, scaling $g$ is not permissible as it would affect the constraint. For example, the solutions for $\min_x 7x-100$ where $x \in \{0,1\}$ (with $x=0$ being optimal and $x=1$ as the second-best solution) maintain the same relative ordering as those for $\min_x 7x+900$.}
}


\eat{If all variables in $P(\mathbf{x})$ are binary (or discrete in general), namely they take values from the domain $\{0,1\}$ then $P(\mathbf{x})$}

If all variables in $\mathbf{Y}$ are binary (or discrete in general), Eq. \eqref{eqn:cmpe_min} 
can be formulated as an (equivalent) integer linear programming (ILP) problem by introducing auxiliary integer variables for each multilinear term (e.g., $y_{1,2}=y_1y_2$, $y_{2,3}=y_2y_3$, etc.) and adding appropriate constraints to model the equivalence between the auxiliary variables and multilinear terms (see for example \cite{koller2009probabilistic}, Chapter 13). Therefore, in practice, (\ref{eqn:cmpe_min}) can be solved optimally using mixed integer linear programming (MILP) solvers such as Gurobi \cite{gurobi2021gurobi} and SCIP \cite{achterberg2008constraint, achterberg2009scip}. 

Unfortunately, due to a presence of a dense global constraint, namely $g_{\mathbf{x}}(\mathbf{y}) \leq 0$ in Eq. \eqref{eqn:cmpe_min}, the MILP solvers often perform poorly. Instead, in practice, application designers often use efficient, specialized algorithms that exploit problem structure for lower bounding $p^*_{\mathbf{x}}$, and then using these lower bounds in an \textit{anytime} branch-and-bound algorithm to obtain an upper bound on $p^*_{\mathbf{x}}$. 


\subsection{Specialized Lower Bounding Algorithms}
\label{sec:lower_bound}
Recently, \citet{rahman2021novel} proposed two new approaches for computing upper bounds on the optimal value of the maximization problem given in Eq. \eqref{eqn:cmpe}. These methods can be easily adapted to obtain a lower bound on $p^*_{\mathbf{x}}$; because an upper bound on the maximization problem is a lower bound on the corresponding minimization problem. We present the adaptations of Rahman et al.'s approach next.

The first approach is based on the Lagrangian relaxation method that introduces a \textit{Lagrange multiplier} $\mu\geq 0$ to transform the constrained minimization problem to the following unconstrained problem: $\text{minimize}_{\mathbf{y}} f_{\mathbf{x}}(\mathbf{y}) + \mu g_{\mathbf{x}}(\mathbf{y})$. Let $d^*_{\mu}$ denote the optimal value of the unconstrained problem. Then, it is easy to show that $d^*_{\mu} \leq p^*_{\mathbf{x}}$. The largest upper bound is obtained by finding a value of $\mu$ that maximizes $d^*_{\mu}$. More formally,
\begin{align}
\label{eq:lagrange-cmpe_min}
    \max_{\mu \geq 0} d^*_{\mu} = \max_{\mu \geq 0} \min_{\mathbf{y}} f_{\mathbf{x}}(\mathbf{y}) + \mu g_{\mathbf{x}}(\mathbf{y}) \leq p^*_{\mathbf{x}}
\end{align}
\citet{rahman2021novel} proposed to solve the inner minimization problem using exact techniques from the graphical models literature such as variable/bucket elimination \cite{dechter99}, branch and bound search and best-first search \cite{marinescu&dechter12,marinescu&dechter09,wu2020map}. When exact inference is not feasible, \citet{rahman2021novel} proposed to solve the inner problem using approximate inference techniques such as mini-bucket elimination, dual-decomposition and join-graph based bounding algorithms \cite{choi_relax_2010,dechter2003mini,wainwright2005map,globerson2007fixing,komodakis2007mrf,ihler2012join}. The outer maximization problem is solved using sub-gradient ascent. 

The second approach by \citet{rahman2021novel} uses the Lagrangian decomposition method to transform the problem into a multi-choice knapsack problem (MCKP) and then utilizes off-the-shelf MCKP solvers. In our experiments, we use the Lagrange relaxation approach given in Eq. \eqref{eq:lagrange-cmpe_min}.

If the set $\mathbf{Y}$ contains continuous variables, then it is not possible to reduce it to an equivalent MILP/LP \cite{horst1996global,sherali1992global}. However, by leveraging linearization methods \cite{sherali1992global,Sherali&Adams09} and solving the resulting problem using linear programming (LP) solvers, we can still obtain good lower bounds on $p^*_{\mathbf{x}}$.

\eat{
Optimization tasks over log-linear models like MPE and CMPE can be cast as mixed-integer linear programs (MILPs) and solved by MILP solvers like Gurobi \cite{gurobi2021gurobi} and SCIP \cite{achterberg2008constraint, achterberg2009scip}. Due to the dense global constraints in these problems, the MILP solvers often perform poorly. Instead researchers often try to invent algorithms to efficiently find upper or lower bounds to the optimal solutions of these problems by exploiting the underlying problem structure. The computed bounds can then be used to guide a search like branch and bound or even serve as near-optimal, heuristic solutions. \cite{rahman2021novel} proposed two approaches for computing upper bounds to the CMPE problem in \ref{eqn:cmpe}. \eat{Both approaches use method of multipliers to relax the original constrained problem.} The first approach is based on the Lagrangian dual method that introduces a \textit{Lagrange multiplier} $\lambda\geq 0$ to transform the constrained maximization of \ref{eqn:cmpe} to an unconstrained maximization of $f_\mathbf{x}(\mathbf{y}) + \lambda g_\mathbf{x}(\mathbf{y})$. The optimal value for $\lambda$ is then obtained by solving the dual problem using iterative sub-gradient methods. The second approach uses the Lagrangian Decomposition to transform the problem into a multi-choice knapsack problem and then utilize off-the-shelf MCKP solvers using Gurobi or SCIP. Both of these approaches yield upper bounds to the optimal solution of CMPE that can be used in subsequent search methods over the decision variables in $\mathbf{Y}$. 
}

\input{content/new-section3}

\eat{
 \section{Solving CMPE using Learning to Optimize Methods in Literature}
 \label{sec:ML-CMPE}
 In this section, we show how techniques developed in the learning to optimize literature \cite{donti2021dc3,fioretto2020predicting,park2022self,zamzam2019learning} which seeks to develop machine learning approaches for solving constrained optimization problems can be leveraged to solve the CMPE task. The main idea is to train a deep neural network $\mathcal{F}_\Theta:\mathbf{X}\rightarrow \mathbf{Y}$ parameterized by the set $\Theta\in\mathbb{R}^M$ such that at test time given evidence $\mathbf{x}$, the network is able to predict an (near) optimal solution $\hat{\mathbf{y}}$ to the CMPE problem. While no prior work exists on solving CMPE using deep neural networks, we explore potential supervised and self-supervised techniques that can be adapted from the \textit{learning to optimize} literature.
 
 \eat{
 Let $\mathcal{F}_\Theta:\mathbf{X}\rightarrow \mathbf{Y}$ be a deep neural network parameterized by the set $\Theta \in \mathbb{R}^M$ which takes as input an assignment $\mathbf{x}$ over $\mathbf{X}$ and outputs an assignment $\hat{\mathbf{y}}$ over $\mathbf{Y}$ that is a good approximation to the optimal solution of the CMPE problem in \ref{eqn:cmpe}.}

\subsection{Supervised Methods} Learning the parameters of $\mathcal{F}_\Theta$ in a supervised manner will require labeled data in the form $\mathcal{D}=\{\langle \mathbf{x}_i, \mathbf{y}_i\rangle \}_{1}^{N}$ where each label $\mathbf{y}_i$ is an optimal solution to the problem in Eq. \eqref{eqn:cmpe_min} given $\mathbf{x}_i$. In such a setting, we can choose to minimize a loss between the true solution $\mathbf{y}_i$ and the predicted solution $\hat{\mathbf{y}_i} = \mathcal{F}_\Theta (\mathbf{x}_i)$. \citet{zamzam2019learning} proposed to use two standard loss functions: Mean-Squared-Error (MSE)  $\bigl(\sum_i(\mathbf{y}_i-\mathbf{\hat{y}}_i)^2/N\bigr)$ and Mean-Absolute-Error (MAE) $\bigl(\sum_i|\mathbf{y}_i-\mathbf{\hat{y}}_i|/N\bigr)$, and we can easily leverage them to solve the CMPE task. 
Experimentally (see the supplementary material), we found that neural networks trained using the MAE and MSE loss functions often output infeasible assignments. To address this issue,  in prior work, \citet{nellikkath2021physicsinformed} proposed to add a penalty to the loss function for violating the constraint. More formally, the authors propose to  add $\lambda \times \max\{0,g_{\mathbf{x}}(\hat{\mathbf{y}})\}$ to the loss function where $\lambda$ is a penalty coefficient (hyperparameter). It was observed in prior work that the quality of the solutions greatly depends on the value chosen for the penalty coefficient, and it is not straightforward to choose it optimally because it varies for each $\mathbf{x}$. To circumvent this issue, Fioretto et al. \cite{fioretto2020predicting} proposed a Lagrangian dual approach to optimally determine the penalty coefficient via subgradient descent. In our experiments, we evaluated both the naive and the penalty based supervised loss approaches (for CMPE) and found that the penalty method with MSE loss yields the best results. Therefore, in our experiments, we use it as a strong supervised baseline.
\subsection{Self-Supervised Methods} Supervised methods require pre-computed solutions to several NP-hard problem instances, each of which is computationally expensive to obtain. An additional challenge pertaining to the acquisition of optimal solutions in supervised methods stems from the potential variability in outputs produced by contemporary solvers, as elucidated by \citet{kotary2021learning}. Due to the neural networks' capacity to learn the mapping from input to output, the presence of divergent outputs for a given input within the dataset can lead to confusion, as it imparts disparate gradient information to the network for identical input instances. We, therefore focus on training the neural network in a self-supervised manner that does not rely on pre-computed results. 
\eat{We can easily transform our maximization problem in \ref{eqn:cmpe} to the following minimization problem in order to construct suitable loss function for the self-supervised training of the neural network. We take each potential function value of the log-linear model $f$ and negate it. Since this transformation can be efficiently done inside a program, we omit using any explicit notations for the transformed potentials and continue to use $f_\mathbf{x}(\mathbf{{y}})$ to denote $-f_\mathbf{x}(\mathbf{{y}})$.
\begin{align}\label{eqn:cmpe_min}
     \underset{\hat{\mathbf{y}}}{\text{minimize}} \:\: f_{\mathbf{x}}(\hat{\mathbf{y}}) \:\:s.t.\:\: g_{\mathbf{x}}(\hat{\mathbf{y}}) \leq 0
\end{align}
}

\citet{donti2021dc3} proposed to use a Lagrange multiplier to relax the constrained optimization problem into an unconstrained one and use the corresponding Lagrangian as a loss function: $f_\mathbf{x}(\hat{\mathbf{y}}) + \lambda \max\{0, g_\mathbf{x}(\hat{\mathbf{y}})\}$.\tr{The authors do not use the "Lagrangian" method.} This is similar to a supervised-penalty loss except the naive supervised loss (MSE/MAE) is now replaced by the primal objective over $f(.)$. This method retains the drawbacks of penalty-based methods having to decide on the value of the penalty coefficient $\lambda$. 

Recently, \citet{park2022self} proposed a self-supervised primal-dual learning method that jointly optimizes for the primal variables ($\hat{\mathbf{y}}$) and the dual variables ($\lambda$) and can generate instance specific values for the Lagrange multipliers. This technique leverages two distinct networks to emulate the functionality of the Augmented Lagrangian Method (ALM): the first network specializes in learning the primal aspect, while the second network focuses on learning the dual facet. The training process uses a sequential approach, where one network is trained while the other remains frozen to furnish the requisite values for the loss computation. 

The methods proposed by \citet{donti2021dc3}, \citet{park2022self} can be easily adapted to solve the CMPE task. A drawback of both methods is that the global minima of the Lagrangian relaxation may not equal the global minima of the CMPE task (because of duality gap in multilinear problems). We address this limitation by using first principles to derive a new loss function.

\subsection{Methods for \mincmpe Objective Minimization}
\label{subsec:method_cmpe_obj}

\sva{Let the optimal solution for $ Q $ in Equation \ref{eqn:cmpe_min} be denoted as $ p^* $. Next we will look at various algorithms for solving constrained optimization problems, specifically the \mincmpe.}

\sva{\textbf{Lagrangian Relaxation:}
The optimal solution for $ Q $ can be obtained through Lagrangian relaxation \citep{geoffrion1974lagrangean}, which seeks to minimize the function $ f_\mathbf{x}(\hat{\mathbf{y}}) + \lambda_\mathbf{x}^{\mathcal{L}} g_\mathbf{x}(\hat{\mathbf{y}}) $, subject to $ \lambda_\mathbf{x}^{\mathcal{L}} \geq 0 $. The resultant optimal solution from this optimization process is denoted as $ a^* $. The parameter $ \lambda_\mathbf{x}^{\mathcal{L}} $ serves as the Lagrange multiplier and is employed to penalize constraint violations. It should be emphasized that this method yields a lower bound for $ Q $. The selection of appropriate values for the Lagrange multipliers is crucial for the quality of the generated lower bound. Sub-gradient optimization is a suitable technique for determining the values of $ \lambda $. The technique is iterative, beginning with an initial set of multipliers. Subsequently, subgradients are computed, and gradient updates are performed, enabling a reduction in the gap between the lower bound and the actual solution. 
}

\sva{\textbf{Penalty Method:} We can also solve the constrained optimization problem $Q$ through penalty methods. A term, referred to as the penalty function, is incorporated into the objective function. This term consists of a penalty parameter multiplied by a function that quantifies constraint violations. Various penalty functions ($g_\mathcal{P}(\hat{\mathbf{y}})$) can be utilized, such as $ \max\{0, g_\mathbf{x}(\hat{\mathbf{y}})\} $ and $ \max\{0, g_\mathbf{x}(\hat{\mathbf{y}})\}^2 $. The objective function to minimize is $ f_\mathbf{x}(\hat{\mathbf{y}}) + \lambda_\mathbf{x}^{\mathcal{P}} g_\mathcal{P}(\hat{\mathbf{y}}) $. Determining the optimal $ \lambda_\mathbf{x}^{\mathcal{P}} $ is crucial. This is achieved by progressively increasing $ \lambda_\mathbf{x}^{\mathcal{P}} $ either until the constraints are met or a predefined maximum $ \lambda_{\text{max}} $ is attained. The value of $ \lambda_\mathbf{x}^{\mathcal{P}} $ can be updated iteratively using straightforward strategies, such as multiplication by a fixed factor. Let the solution obtained from minimizing the penalty-based loss be denoted as $ b^* $. It should be noted that the ideal value for $ \lambda_\mathbf{x}^{\mathcal{P}} $ would be $ \infty $, enabling $ b^* = p^* $.}

\sva{\textbf{Augmented Lagrangian:} Augmented Lagrangian methods provide a unique approach to constrained optimization by formulating an unconstrained objective based on the Lagrangian of the original problem, further augmented with an additional penalty term. The objective function to minimize is $ f_\mathbf{x}(\hat{\mathbf{y}}) + \lambda_\mathbf{x}^{\hat{\mathcal{L}}}g_{\mathbf{x}}(\hat{\mathbf{y}}) + \lambda_\mathbf{x}^{\mathcal{P}} g_\mathcal{P}(\hat{\mathbf{y}})$. $ \lambda_\mathbf{x}^{\hat{\mathcal{L}}} $ serves as an estimate of the Lagrange multipliers, with its accuracy improving in each iteration. Meanwhile, $ \lambda_\mathbf{x}^{\hat{\mathcal{P}}} $ represents the penalty parameter. In this methodology, both $ \lambda_\mathbf{x}^{\hat{\mathcal{L}}} $ and $ \lambda_\mathbf{x}^{\mathcal{P}} $ are iteratively updated. The adjustment of $ \lambda_\mathbf{x}^{\mathcal{P}} $ follows the pattern described in the penalty methods. For updating $ \lambda_\mathbf{x}^{\hat{\mathcal{L}}} $, the subsequent rule is employed:
$ \lambda_\mathbf{x}^{\hat{\mathcal{L}}} \longleftarrow\lambda_\mathbf{x}^{\hat{\mathcal{L}}} +  max\{ 0, \lambda_\mathbf{x}^{\mathcal{P}}g_{\mathbf{x}}(\hat{\mathbf{y}_k})\}$, where $\hat{\mathbf{y}_k}$ is the solution to the unconstrained problem at the current step. 
}

\eat{Notably, PDL incorporates instance-specific Lagrangian multipliers sourced from the dual network, thereby enabling the primal network to yield the desired outputs while the dual network supplies the Lagrangian multiplier for each specific example. In addition to the aforementioned components, the PDL method also encompasses the dynamic updating of the penalty coefficient, denoted as $\rho$. This coefficient serves a dual purpose: first, it determines the step size for updating the dual variables, and second, it restricts the severity of constraint violations.} 

\eat{
Supervised methods necessitate the availability of labels for each training example. In the context of our specific task, these labels represent the optimal solutions for the given problem, and their generation can be a laborious and resource-intensive endeavor. 

\textbf{Self-Supervised Penalty Approach} - 
The self-supervised penalty approach employs a similar loss function as the supervised penalty approach, with the key distinction lying in the utilization of the objective function in lieu of the conventional naïve loss. This modification effectively transforms the method from a supervised paradigm to a self-supervised paradigm, as it no longer necessitates the availability of optimal solutions for training purposes.}

}

%% file: content/new-section3.tex
 \section{Solving CMPE using Methods from the Learning to Optimize Literature}
 \label{sec:ML-CMPE}
 In this section, we show how techniques developed in the learning to optimize literature \cite{donti2021dc3,fioretto2020predicting,park2022self,zamzam2019learning} which seeks to develop machine learning approaches for solving constrained optimization problems can be leveraged to solve the CMPE task. The main idea is to train a deep neural network $\mathcal{F}_\Theta:\mathbf{X}\rightarrow \mathbf{Y}$ parameterized by the set $\Theta\in\mathbb{R}^M$ such that at test time given evidence $\mathbf{x}$, the network is able to predict an (near) optimal solution $\hat{\mathbf{y}}$ to the CMPE problem. Note that as far as we are aware, no prior work exists on solving CMPE using deep neural networks.
 
 
 \eat{
 Let $\mathcal{F}_\Theta:\mathbf{X}\rightarrow \mathbf{Y}$ be a deep neural network parameterized by the set $\Theta \in \mathbb{R}^M$ which takes as input an assignment $\mathbf{x}$ over $\mathbf{X}$ and outputs an assignment $\hat{\mathbf{y}}$ over $\mathbf{Y}$ that is a good approximation to the optimal solution of the CMPE problem in \ref{eqn:cmpe}.} 
\subsection{Supervised Methods} In order to train the parameters of $\mathcal{F}_\Theta$ in a supervised manner, we need to acquire labeled data in the form $\mathcal{D}=\{\langle \mathbf{x}_i, \mathbf{y}_i\rangle \}_{i=1}^{N}$ where each label $\mathbf{y}_i$ is an optimal solution to the problem given in Eq. \eqref{eqn:cmpe_min} given $\mathbf{x}_i$. In practice, we can generate the assignments $\{\mathbf{x}_i\}_{i=1}^{N}$ by sampling them from the graphical model corresponding to $f$ and the labels $\{\mathbf{y}_i\}_{i=1}^{N}$ by solving the minimization problem given in Eq. \eqref{eqn:cmpe_min} using off-the-shelf solvers such as Gurobi and SCIP. 

Let $\hat{\mathbf{y}}_i = \mathcal{F}_\Theta (\mathbf{x}_i)$ denote the labels predicted by the neural network for  $\mathbf{x}_i$. Following \citet{zamzam2019learning}, we propose to train $\mathcal{F}_\Theta$ using the following two loss functions
\begin{align}
    \text{Mean-Squared Error (MSE)}:  \frac{1}{N}\sum_i \left (\mathbf{y}_i-\mathbf{\hat{y}}_i \right )^2
\end{align}
\begin{align}
    \text{Mean-Absolute-Error (MAE):}  \frac{1}{N}\sum_i \left |\mathbf{y}_i-\mathbf{\hat{y}}_i \right |
\end{align}


Experimentally (see the supplementary material), we found that neural networks trained using the MAE and MSE loss functions often output infeasible assignments. To address this issue,  following prior work \citep{nellikkath2021physicsinformed}, we propose to add $\lambda_{\mathbf{x}} \max\{0,g_{\mathbf{x}}(\hat{\mathbf{y}})\}$ to the loss function where $\lambda_{\mathbf{x}}$ is a penalty coefficient. 


In prior work \cite{fioretto2020predicting}, it was observed that the quality of the solutions greatly depends on the value chosen for $\lambda_{\mathbf{x}}$. Moreover, it is not straightforward to choose it optimally because it varies for each $\mathbf{x}$. To circumvent this issue, we propose to update $\lambda_{\mathbf{x}}$ via a \textit{Lagrangian dual} method \cite{nocedal}. More specifically, we propose to use the following subgradient method to optimize the value of $\lambda_{\mathbf{x}}$.
While training a neural network, let $\lambda^k_{\mathbf{x}_i}$ and $\hat{\mathbf{y}}_i^k$ denote the values of the penalty co-efficient and the predicted assignment respectively at the $k$-th epoch and for the $i$-th example in $\mathcal{D}$ (if the $i$-th example is part of the current mini-batch), then, we update $\lambda^{k+1}_{\mathbf{x}_i}$ using
\begin{align}
\label{eq:sub-1}
    \lambda^{k+1}_{\mathbf{x}_i}=\lambda^k_{\mathbf{x}_i} + \rho \max \{0 , g_{\mathbf{x}_i}(\hat{\mathbf{y}}_i^k)\}
\end{align}
where $\rho$ is the Lagrangian step size.
In our experiments, we evaluated both the naive and the penalty based supervised loss approaches (for CMPE) and found that the penalty method with MSE loss yields the best results. Therefore, in our experiments, we use it as a strong supervised baseline.
\subsection{Self-Supervised Methods} 
Supervised methods require pre-computed solutions for numerous NP-hard/multilinear problem instances, which are computationally expensive to derive. Therefore, we propose to train the neural network in a self-supervised manner that does not depend on the pre-computed results. Utilizing findings from \citet{kotary2021learning} and \citet{park2022self}, we introduce two self-supervised approaches: one is grounded in the \textit{penalty method}, and the other builds upon the \textit{augmented Lagrangian method}.

\textbf{Penalty Method} \cite{donti2021dc3,kotary2021learning,fioretto2020predicting}. In the penalty method, we solve the constrained minimization problem by iteratively transforming it into a sequence of unconstrained problems. Each unconstrained problem at iteration $k$ is constructed by adding a term, which consists of a penalty parameter $\lambda^k_\textbf{x}$ multiplied by a function $\max\{0, g_\mathbf{x}(\mathbf{y})\}^2$ that quantifies the constraint violations, to the objective function. Formally, the optimization problem at the $k$-th step is given by:
\begin{align}
\label{eq:penalty-1}
    \min_\mathbf{y} f_\mathbf{x}(\mathbf{y}) + \frac{\lambda^k_\textbf{x}}{2} \max \left \{0, g_\mathbf{x}(\mathbf{y})\right \}^2
\end{align}
Here, $\lambda_\mathbf{x}^{k} $ is progressively increased either until the constraint is satisfied or a predefined maximum $ \lambda_{\text{max}} $ is reached. $ \lambda_\mathbf{x}^{k} $ can be updated after a few epochs using simple strategies such as multiplication by a fixed factor (e.g., 2, 10, etc.). 

The penalty method can be adapted to learn a neural network in a self-supervised manner as follows. At each epoch $k$, we sample an assignment $\mathbf{x}$ (or multiple samples for a mini-batch) from the graphical model corresponding to $f$, predict $\hat{\mathbf{y}}$ using the neural network and then use the following loss function to update its parameters:
\begin{align}
\label{eq:penalty-loss}
    \mathcal{L}^{pen}_\mathbf{x}(\hat{\mathbf{y}}) = f_\mathbf{x}(\hat{\mathbf{y}}) + \frac{\lambda^k_\textbf{x}}{2} \max \left \{0, g_\mathbf{x}(\hat{\mathbf{y}})\right \}^2
\end{align}


Determining the optimal $ \lambda_\mathbf{x}^{k} $ is crucial. In prior work, \citet{kotary2021learning} and \citet{fioretto2020predicting} proposed to update it via a subgradient method, similar to the update rule given by Eq. \eqref{eq:sub-1}. More formally, we can update $\lambda_\mathbf{x}^{k}$ using:
\begin{align}
\label{eq:sub-2}
    \lambda_\mathbf{x}^{k+1}=\lambda_\mathbf{x}^{k} + \rho  \max \left \{0, g_\mathbf{x}(\hat{\mathbf{y}}^k)\right \}
\end{align}
where $\rho$ is the Lagrangian step size.



\textbf{Augmented Lagrangian Method (ALM)}. In this method, we augment the objective used in the penalty method with a Lagrangian term. More formally, the optimization problem at the $k$-th step is given by (compare with Eq. \eqref{eq:penalty-1}):
\begin{align}
    \label{eq:alm}
    \min_\mathbf{y} f_\mathbf{x}(\mathbf{y})  +\frac{\lambda^k_\textbf{x}}{2} \max \left \{0, g_\mathbf{x}(\mathbf{y})\right \}^2 + \mu^k_\textbf{x} g_\mathbf{x}(\mathbf{y})
\end{align}
Here, $\lambda^k_\textbf{x}$ may be progressively increased similar to the penalty method while $\mu^k_\textbf{x}$ is updated using
\begin{align}
    \label{eq:alm-1}
    \mu^{k+1}_\textbf{x} = \max \left \{ 0, \mu^k_\textbf{x} + \lambda^k_\textbf{x} g_\mathbf{x}(\mathbf{y}^k) \right \}
\end{align}
Recently, \citet{park2022self} proposed a self-supervised primal-dual learning method that leverages two distinct networks to emulate the functionality of ALM: the first (primal) network takes as input $\mathbf{x}$ and outputs $\mathbf{y}$  while the second network focuses on learning the dual aspects; specifically it takes $\mathbf{x}$ as input and outputs $\mu^k_\textbf{x}$. The training process uses a sequential approach, where one network is trained while the other remains frozen to furnish the requisite values for the loss computation. 

The primal network uses the following loss function:
\begin{align*}
    \mathcal{L}^{A,p}_\mathbf{x}(\hat{\mathbf{y}}|\mu,\lambda) = f_\mathbf{x}(\hat{\mathbf{y}}) + \frac{\lambda}{2} \max \left \{0, g_\mathbf{x}(\hat{\mathbf{y}})\right \}^2+ \mu g_\mathbf{x}(\mathbf{y})
\end{align*}
While the dual network uses the following loss function
\begin{align*}
    \mathcal{L}^{A,d}_\mathbf{x}(\hat{\mu}|\mathbf{y},\lambda, \mu^k) =  || \hat{\mu} - \max \left \{ 0, \mu^k + \lambda g_\mathbf{x}(\mathbf{y}) \right \}||
\end{align*}
where $\hat{\mu}$ is the predicted value of the Lagrangian multiplier.

\subsubsection{Drawbacks of the Penalty and ALM Methods}
A limitation of the penalty-based self-supervised method is that it does not guarantee a global minimum unless specific conditions are met. In particular, the optimal solution w.r.t. the loss function (see Eq. \eqref{eq:penalty-loss}) may be far away from the optimal solution $\mathbf{y}^*$ of the problem given in Eq. \eqref{eqn:cmpe_min}, unless the penalty co-efficient $\lambda^k_\mathbf{x} \rightarrow \infty$. Moreover, when $\lambda^k_\mathbf{x}$ is large for all $\mathbf{x}$, the gradients will be uninformative. In the case of ALM method (cf. \cite{nocedal}), for global minimization, we require that either $\lambda^k_\mathbf{x} \rightarrow \infty$ or $\forall \mathbf{x}$ with $g_\mathbf{x}(\mathbf{y})>0$, $\mu_\mathbf{x}^k$ should be such that $\min_{\mathbf{y}} f_\mathbf{x}(\mathbf{y})+\mu_\mathbf{x}^k g_\mathbf{x}(\mathbf{y}) > p_\mathbf{x}^*$. Additionally, ALM introduces a dual network, increasing the computational complexity and potentially leading to negative information transfer when the dual network's outputs are inaccurate. These outputs are subsequently utilized in the loss to train the primal network for the following iteration, thereby exerting a negative effect. To address these limitations, next, we introduce a self-supervised method that achieves global minimization without the need for a dual network or infinite penalty coefficients.

%% file: content/3-preliminary.tex
We denote sets of random variables by bold upper-case letters (e.g., $\mathbf{X}$, $\mathbf{Y}$, $\mathbf{Z}$, etc.) and assignments to them by bold lower-case letters (e.g., $\mathbf{x}$, $\mathbf{y}$, $\mathbf{z}$, etc.). Let $\mathbf{z}_A$ denote the projection of the complete assignment $\mathbf{z}$ on to the set $A$ of variables. Random variables can either be discrete or continuous. A log-linear model $\mathcal{M}_f = \langle \mathbf{Z}, f \rangle$ represents the following joint probability distribution over a set $\mathbf{Z}$ of random variables
\begin{equation}\label{eqn:log-linear}
    P(\mathbf{z}) \propto \exp \left (\sum_i f_i(\mathbf{z}_{S(f_i)}) \right ) 
\end{equation}
Where each $f_i \in f$ is a log-potential function that takes as input an assignment to a subset $S(f_i)$ of variables and maps it to $\mathbb{R}$. $S(f_i)$ is called the scope of function $f_i$. 
We slightly abuse notation and write $f(\mathbf{z})$ to denote $\sum_i f_i(\mathbf{z}_{S(f_i)})$. Let $\mathbf{Z}=\mathbf{X} \cup \mathbf{Y}$ such that $\mathbf{X} \cap \mathbf{Y} =\emptyset$. Given a log-linear model defined over the random variables $\mathbf{Z}$, the most probable explanation (MPE) task is to find an assignment $\mathbf{y}^*$ to all variables in $\mathbf{Y}$ given an observation (evidence) $\mathbf{x}$ over $\mathbf{X}$ such that $P(\mathbf{x}, \mathbf{y}^*)$ is maximized. MPE is NP-hard in general, but can be solved in polynomial time on certain classes of  probabilistic circuits (typically having no latent variables) such as cutset networks [\cite{}], AND/OR graphs having small context size [\cite{}] and low tree-width Bayesian and Markov networks [\cite{}].  


\subsection{Constrained Most Probable Explanation}
Constrained most probable explanation or CMPE is an MPE task with constraints which restrict the values the decision variables $\mathbf{Y}$ can take. Given two log-linear models having log-potentials $f$ and $g$ respectively, such that both of them are defined over the same sets of variables $\mathbf{X}$ and $\mathbf{Y}$, an observation $\mathbf{x}$ over $\mathbf{X}$, and a constant $q \in \mathbb{R}$, the CMPE task is to find the most likely assignment $\mathbf{y}^*$ to all the variables $\mathbf{Y}$ w.r.t. $f$ such that the constraint $g(\mathbf{x,y}^*)\leq q$ is satisfied. Formally, CMPE is the following maximization problem.
\begin{equation}\label{eqn:cmpe}
    \underset{\mathbf{y}}{\text{maximize}} \:\: f(\mathbf{x,y}) \:\: s.t. \:\: g(\mathbf{x, y}) \leq q
\end{equation}
For brevity, we will denote $f(\mathbf{x, y})$ as $f_\mathbf{x}(\mathbf{y})$ and $g(\mathbf{x, y})$ as $g_\mathbf{x}(\mathbf{y})$. Constraints over the decision variables in $\mathbf{Y}$ can be specified in a variety of ways. In our experiments, we have used both log-linear models and neural network classifiers as $g$. 

\subsubsection{Upper Bounding Methods for CMPE}\label{sec:upper_bounds}

\cite{rahman2021novel} proposed two approaches for computing upper bounds to the CMPE problem in \ref{eqn:cmpe}. \eat{Both approaches use method of multipliers to relax the original constrained problem.} The first approach is based on the Lagrangian dual method that introduces an auxiliary variable $\lambda\geq 0$ to transform the constrained maximization of \ref{eqn:cmpe} to an unconstrained maximization of $f_\mathbf{x}(\mathbf{y}) + \lambda g_\mathbf{x}(\mathbf{y})$. The optimal value for $\lambda$ is then obtained by solving the dual problem using iterative sub-gradient methods. The second approach uses the Lagrangian Decomposition to transform the problem into a multi-choice knapsack problem and then utilize off-the-shelf MCKP solvers. Both of these approaches yield upper bounds to the optimal solution of CMPE that can be used in subsequent search methods like branch-and-bound over the decision variables in $\mathbf{Y}$.  

\eat{
\begin{itemize}
    \item formally define the task here.
    \item describe the Lagrangian formulation and its guarantees.  
    \item Mention the linear approximation to the multi-linear integer program and that existing solvers are slow and inaccurate. 
\end{itemize}

We are interested in solving a more generalized version of the CMPE task, namely, given observation or evidence to some of the variables, find the CMPE assignment to the rest of the variables. We partition the set of given random variables into disjoint subsets $\mathbf{X}$ and $\mathbf{Y}$ such that $|\mathbf{X}|=n$ and $|\mathbf{Y}|=m$. Let $\mathbf{X}$ be the set of evidence (input) variables and $\mathbf{Y}$ be the set of query (output) variables. Given an assignment $\mathbf{X=x}$, a conditional CMPE task is the following maximization problem. 

\begin{align}\label{eqn:ccmpe}
    \underset{\mathbf{y}}{\text{maximize}} \:\:\:f(\mathbf{x, y}) \:\:s.t.\:\: g(\mathbf{x,y}) \leq 0
\end{align}}
 \subsubsection{Machine Learning Methods for CMPE}
 A learning approach to solve the CMPE task means to train a deep learning model with data such that at test time given an observation or evidence $\mathbf{X=x}$, the model is able to predict a near optimal solution $\mathbf{y}^*$ to the CMPE problem. Although, to the best of our knowledge, no prior work exists on solving this specific task using a deep neural network, we discuss here possible supervised and self-supervised machine learning techniques that can used to learn for CMPE.

 Let $\mathcal{F}_\Theta:\mathbf{X}\rightarrow \mathbf{Y}$ be a deep neural network parameterized by the set $\Theta \in \mathbb{R}^M$ which takes as input an assignment $\mathbf{x}$ over $\mathbf{X}$ and outputs an assignment $\hat{\mathbf{y}}$ over $\mathbf{Y}$ that is a solution to the CMPE problem in \ref{eqn:cmpe}. Learning the parameters of $\mathcal{F}_\Theta$ in a supervised manner will require labeled data in the form $\mathcal{D}=\{\langle \mathbf{x}_i, \mathbf{y}_i\rangle \}_{1}^{N}$ where each label $\mathbf{y}_i$ is the true optimal solution to the problem in \ref{eqn:cmpe} given the observation $\mathbf{x}_i$. In such setting one can choose to minimize a differential loss function like the squared loss $\sum_i(\mathbf{y}_i-\mathbf{\hat{y}}_i)^2$ where $\hat{\mathbf{y}_i} = \mathcal{F}_\Theta (\mathbf{x}_i)$. Unfortunately, this method requires pre-computed solutions (labels) to several NP-hard problem instances each of which is computationally expensive to obtain. Instead, we propose a self-supervised approach that learns to optimize for the problem in \ref{eqn:ccmpe} given data over $\mathbf{X}$ only i.e. $\mathcal{D}=\{\mathbf{x}_i\}_{i=1}^{i=N}$.

Since neural networks are designed to minimize a loss function, we can easily construct an equivalent loss to minimize given the maximization problem in \ref{eqn:cmpe} as follows. Take each instantiated potential function value of the log-linear model in $f$ and negate it. Since this transformation can be efficiently done inside a program, we omit using any explicit notations for the transformed potentials and continue to use $f_\mathbf{x}(\mathbf{\hat{y}})$ to denote $-f_\mathbf{x}(\mathbf{\hat{y}})$.

\begin{align}\label{eqn:cmpe_min}
     \underset{\hat{\mathbf{y}}}{\text{minimize}} \:\: f_{\mathbf{x}}(\hat{\mathbf{y}}) \:\:s.t.\:\: g_{\mathbf{x}}(\hat{\mathbf{y}}) \leq 0
\end{align}

\textbf{Supervised Methods} - 
Supervised methods typically leverage pre-solved optimal solutions generated by an optimization solver. \\
\textbf{Supervised Loss Approach} - The supervised loss approach, exemplified by \citet{zamzam2019learning} adopts a naïve strategy of minimizing the loss between the model's outputs and the corresponding optimal solutions. Mean Squared Error (MSE) and Mean Absolute Error (MAE) are employed as metrics for comparing these two quantities. \\
\textbf{Supervised Penalty Approach} -  The supervised penalty approach extends the supervised loss approach by incorporating penalty terms for each constraint violation. This method adds a penalty to each constraint, with a penalty coefficient denoted as $\rho$.  This loss formulation, as demonstrated in \citet{nellikkath2021physicsinformed}, introduces penalties to account for violations of constraints within the optimization process.

\textbf{Self-Supervised Methods} - Supervised methods necessitate the availability of labels for each training example. In the context of our specific task, these labels represent the optimal solutions for the given problem, and their generation can be a laborious and resource-intensive endeavor. An additional challenge pertaining to the acquisition of optimal solutions in supervised methods stems from the potential variability in outputs produced by contemporary solvers, as elucidated by \citet{kotary2021learning}. Due to the neural networks' capacity to learn the mapping from input to output, the presence of divergent outputs for a given input within the dataset can lead to confusion, as it imparts disparate gradient information to the network for identical input instances.

\textbf{Self-Supervised Penalty Approach} - 
The self-supervised penalty approach employs a similar loss function as the supervised penalty approach, with the key distinction lying in the utilization of the objective function in lieu of the conventional naïve loss. This modification effectively transforms the method from a supervised paradigm to a self-supervised paradigm, as it no longer necessitates the availability of optimal solutions for training purposes.

\textbf{PDL} - 
The Primal-Dual Learning (PDL) \citep{parkSelfSupervisedPrimalDualLearningConstrainedOptimization2022} is a self-supervised primal-dual learning method.  This technique leverages two distinct networks to emulate the functionality of the Augmented Lagrangian Method (ALM): the first network specializes in learning the primal aspect, while the second network focuses on learning the dual facet. The training process entails a sequential approach, where one network is trained while the other remains frozen to furnish the requisite values for the loss computation. Notably, PDL incorporates instance-specific Lagrangian multipliers sourced from the dual network, thereby enabling the primal network to yield the desired outputs while the dual network supplies the Lagrangian multiplier for each specific example. In addition to the aforementioned components, the PDL method also encompasses the dynamic updating of the penalty coefficient, denoted as $\rho$. This coefficient serves a dual purpose: first, it determines the step size for updating the dual variables, and second, it restricts the severity of constraint violations. 

\tr{NN takes as input either assignment or the instantiated features of the log-linear model and each output is associated with a random variable}

\tr{without eveidence nn solves the unconditional CMPE task}

\tr{Discuss the primal-dual learning and its drawbacks}

%% file: content/4-method.tex
\eat{The methods proposed by \citet{donti2021dc3}, \citet{park2022self} can be easily adapted to solve the CMPE task. A drawback of both methods is that the global minima of the Lagrangian relaxation may not equal the global minima of the CMPE task (because of duality gap in multilinear problems). We address this limitation by using first principles to derive a new loss function.
}
	
\eat{ 
Problem \ref{eqn:nn_cmpe} is equivalent to problem \ref{eqn:ccmpe} and are equivalent to each other and can be transformed to another by simply taking negating the potentials in the log-linear model $f$.

let $\mathcal{C}_x$ denote the CMPE problem defined over the variables $\mathbf{Y}$ as
Given data, training $\mathcal{F}_\Theta$ means to learn the parameters $\Theta$ such that $\hat{\mathbf{y}}$ is the \textit{optimal} solution to the problem $\mathcal{C}_x$. We propose a self-supervised learning approach to train $\mathcal{F}_\Theta$ such that the given data constitutes observations only over the variables in $\mathbf{X}$. }

\eat{\subsection{Determining Value of $\lambda$ using Bounding Methods}
\sva{How can we get bounds for $\lambda$ - how are we doing something different}

\sva{It is important to note that all the methods mentioned in \ref{subsec:method_cmpe_obj} necessitate iterative updates to the values of $ \lambda $. Such an iterative process can be computationally expensive. Additionally, the gap between the optimal solution to problem $ Q $ and the solutions obtained through these methods, as well as the number of constraint violations, is contingent on the value of $ \lambda $. Consequently, if a suitable value for $ \lambda $ is not identified, the performance of these methods is likely to be sub-optimal.
It is crucial to emphasize that neural network training is inherently iterative. Incorporating additional iterative methods into a process already prone to incorrect gradient signals amplifies the likelihood of sub-optimal training, leading to a potential failure in achieving a solution with a narrow optimality gap and minimal constraint violations.}

\sva{In cases where optimizing $ \lambda $ as described in Section \ref{subsec:method_cmpe_obj} is not desired, an alternative is to establish bounds on $ \lambda $ in a one-shot manner. Consider, for example, the penalty method previously discussed, specifically the variant employing a linear penalty term $ f_\mathbf{x}(\hat{\mathbf{y}}) + \lambda_{\mathcal{P}} \max\{0, g_\mathbf{x}(\hat{\mathbf{y}})\}$. We can define the loss as a piece-wise function as follows 
\begin{align}\label{eqn:penalty_}
    \mathcal{L}_\mathbf{x}(\hat{\mathbf{y}}) & = \begin{cases}
    & f_{\mathbf{x}}(\hat{\mathbf{y}}) \:\:\text{if}\:\: g_\mathbf{x}(\hat{\mathbf{y}}) \leq 0\\
    & f_{\mathbf{x}}(\hat{\mathbf{y}}) + \lambda g_{\mathbf{x}}(\hat{\mathbf{y}}) \:\:\text{if}\:\: g_\mathbf{x}(\hat{\mathbf{y}}) > 0\\
\end{cases}
\end{align}}

\sva{It is important to note that the optimal value of $ \lambda $ needs to be identified only for cases where $ g_\mathbf{x}(\hat{\mathbf{y}}) > 0 $. Specifically, for the infeasible region, the objective is to ensure that the loss associated with all infeasible assignments surpasses the optimal value within the feasible region. Mathematically, $ f_{\mathbf{x}}(\hat{\mathbf{y}}) + \lambda g_\mathbf{x}(\hat{\mathbf{y}}) > p_{\mathbf{x}}^* $, where $ p_{\mathbf{x}}^* $ represents the optimal value of $ \min_{\hat{\mathbf{y}}}  f_\mathbf{x}(\hat{\mathbf{y}}) \:\: \text{s.t.} \:\: g_\mathbf{x}(\hat{\mathbf{y}}) \leq 0 $. This leads us to the following optimization problem:
\begin{equation}
\begin{aligned}
    \min_{\mathbf{y}} \quad & f_{\mathbf{x}}(\hat{\mathbf{y}}) + \lambda g_{\mathbf{x}}(\hat{\mathbf{y}}) \\
    \text{s.t.} \quad & g_\mathbf{x}(\hat{\mathbf{y}}) > 0, \\
                      & f_{\mathbf{x}}(\hat{\mathbf{y}}) + \lambda g_{\mathbf{x}}(\hat{\mathbf{y}}) > p_{\mathbf{x}}^*
\end{aligned}
\label{eq:lambda_objective}
\end{equation}
We can use methods discussed in \ref{subsec:method_cmpe_obj} to minimize the objective in \ref{eq:lambda_objective}. Let us rewrite this equation as
\begin{equation}
\begin{aligned}
    \max_{\beta_1, \beta_2} \min_{x, \lambda} & f_{\mathbf{x}}(\hat{\mathbf{y}}) + \lambda g_{\mathbf{x}}(\hat{\mathbf{y}}) \\
    & - \beta_1 g_{\mathbf{x}}(\hat{\mathbf{y}}) \\
    & + \beta_2 \left( p_{\mathbf{x}}^* - f_{\mathbf{x}}(\hat{\mathbf{y}}) - \lambda g_{\mathbf{x}}(\hat{\mathbf{y}}) \right)
\end{aligned}
\label{eq:lambda_objective_unconstrained}
\end{equation}
It is essential to note that optimizing the objective in Eq. \eqref{eq:lambda_objective_unconstrained} requires prior knowledge of $ p_{\mathbf{x}}^* $, the very objective function the optimization aims to minimize in the feasible region. As a result, an approximation for $ p_{\mathbf{x}}^* $ must be employed. This approximation, however, introduces an error that could propagate through subsequent calculations, potentially undermining the optimization process and precluding the identification of optimal $ \hat{\mathbf{y}} $ values.
}
}
\eat{\subsection{A New Loss Function Based on First Principles}}
	
An appropriately designed loss function should have the following characteristics.  For feasible solutions, namely when $g_{\mathbf{x}}(\mathbf{y}) \leq 0$, the loss function should be proportional to $f_{\mathbf{x}}(\mathbf{y})$. While for infeasible assignments, it should equal infinity. This loss function will ensure that once a feasible solution is found, the neural network will only explore the space of feasible solutions. Unfortunately, infinity does not provide any gradient information, and the neural network will get stuck in the infeasible region if the neural network generates an infeasible assignment during training. 
	
	An alternative approach is to use $g$ as a loss function when the constraint is not satisfied (i.e., $g_{\mathbf{x}}(\mathbf{y})>0$) in order to push the infeasible solutions towards feasible ones \cite{Liu2023jan}. Unfortunately,  this approach will often yield feasible solutions that lie at the boundary $g_{\mathbf{x}}(\mathbf{y})=0$.  For instance, for a boundary assignment $\mathbf{y}_b$ where  $g_{\mathbf{x}}(\mathbf{y}_b)=0$ but $f_{\mathbf{x}}(\mathbf{y}_b)>0$ (or decreasing), the sub-gradient will be zero, and the neural network will treat the boundary assignment as an optimal one. 
 

	To circumvent this issue, we propose a loss function which has the following two properties: (1) It is proportional to $g$ in the infeasible region with $f$ acting as a control in the boundary region (when $g$ is zero); and (2) It is proportional to $f$ in the feasible region. Formally, 
		\begin{align}\label{eqn:loss_1}
		\mathcal{L}_\mathbf{x}(\hat{\mathbf{y}}) & = \begin{cases}
			& f_{\mathbf{x}}(\hat{\mathbf{y}}) \:\:\text{if}\:\: g_\mathbf{x}(\hat{\mathbf{y}}) \leq 0\\
			& \alpha_{\mathbf{x}}(f_{\mathbf{x}}(\hat{\mathbf{y}}) + g_{\mathbf{x}}(\hat{\mathbf{y}})) \:\:\text{if}\:\: g_\mathbf{x}(\hat{\mathbf{y}}) > 0\\
		\end{cases}
	\end{align}
	where $\alpha_\mathbf{x}$  is  a function of  the evidence $\mathbf{x}$. Our goal is to find a bound for $\alpha_\mathbf{x}$ such that the following desirable property is satisfied and the bound can be computed in polynomial time for each $\mathbf{x}$ by leveraging bounding methods for CMPE.
 


 \noindent\textbf{Property (Consistent Loss)}: The loss for all infeasible assignments is higher than the optimal value $p_\mathbf{x}^*$.\\
	To satisfy this property,  we have to ensure that:
  \begin{align*}
\forall \hat{\mathbf{y}} \:\: \text{s.t.}\:\: g_\mathbf{x}(\hat{\mathbf{y}}) > 0, \:\:\:\: \alpha_\mathbf{x} \left( f_\mathbf{x}(\hat{\mathbf{y}}) + g_\mathbf{x}(\hat{\mathbf{y}}) \right) & > p^*_\mathbf{x}
\end{align*}
	 which implies that the following condition holds. 
	 \begin{align*}
	 	& \alpha_\mathbf{x} \left (\min_{\mathbf{\hat{y}}}\:\:   f_\mathbf{x}(\hat{\mathbf{y}}) + g_\mathbf{x}(\hat{\mathbf{y}}) \:\:s.t.\:\: g_\mathbf{x}(\hat{\mathbf{y}}) > 0 \right )  > p^*_\mathbf{x} 
    \end{align*}

Let $q_\mathbf{x}^*$ denote the optimal value of $\min_{\hat{\mathbf{y}}}f_\mathbf{x}(\hat{\mathbf{y}}) + g_\mathbf{x}(\hat{\mathbf{y}}) \:\:s.t.\:\: g_\mathbf{x}(\hat{\mathbf{y}}) > 0$. Then, $\alpha_\mathbf{x}  > \frac{p^*_\mathbf{x}}{q^*_\mathbf{x}}$.
\eat{
    \begin{align}
        		& \eat{\therefore\:\:}  \alpha_\mathbf{x}  > \frac{p^*_\mathbf{x}}{q^*_\mathbf{x}}
    \end{align}
}
\eat{We consider three different properties below and derive bounds for $\alpha_{\mathbf{x}}$, thus yielding a hyper-parameter free loss function. To facilitate our presentation, we introduce new notation. Let $\text{p}\text{min}_{\mathbf{x}}^*$,  $\text{pmax}_{\mathbf{x}}^*$, $\text{pbar}_{\mathbf{x}}^*$ and $\text{qmin}_{\mathbf{x}}^*$ denote the optimal values of   $\min_{\mathbf{\hat{y}}}  f_\mathbf{x}(\mathbf{\hat{y}}) \:\:s.t.\:\: g_\mathbf{x}(\mathbf{\hat{y}}) < 0$, $\max_{\mathbf{\hat{y}}}  f_\mathbf{x}(\mathbf{\hat{y}}) \:\:s.t.\:\: g_\mathbf{x}(\mathbf{\hat{y}}) < 0$, $\max_{\mathbf{\hat{y}}}  f_\mathbf{x}(\mathbf{\hat{y}}) \:\:s.t.\:\: -\epsilon <g_\mathbf{x}(\mathbf{\hat{y}}) < 0$ (where $\epsilon$ is a small constant) and $\min_{\mathbf{\hat{y}}}  f_\mathbf{x}(\mathbf{\hat{y}}) + g_\mathbf{x}(\mathbf{\hat{y}}) \:\:s.t.\:\: g_\mathbf{x}(\mathbf{\hat{y}}) > 0$ respectively.
	\noindent\textbf{Property 1 (Conservative Loss)}: Ensure that the loss for all infeasible assignments is higher than the optimal value over the feasible region.\\
	To satisfy property 1,  we have to ensure that:
	 \begin{align*}
	 	\forall \hat{\mathbf{y}} \:\: s.t.\:\:g_\mathbf{x}(\mathbf{\hat{y}}) > 0\:\:\alpha_\mathbf{x} (f_\mathbf{x}(\mathbf{\hat{y}}) + g_\mathbf{x}(\mathbf{\hat{y}})) & > \text{pmin}^*_\mathbf{x}\nonumber\\
	 \end{align*}
	 which implies that the following condition holds. 
	 \begin{align}\label{eqn:alpha}
	 	& \min_{\mathbf{\hat{y}}} \alpha_\mathbf{x} \biggl( f_\mathbf{x}(\mathbf{\hat{y}}) + g_\mathbf{x}(\mathbf{\hat{y}}) \:\:s.t.\:\: g_\mathbf{x}(\mathbf{\hat{y}}) > 0\biggr)  > \text{pmin}^*_\mathbf{x} \nonumber\\
	 		& \therefore  \alpha_\mathbf{x}  > \frac{\text{pmin}^*_\mathbf{x}}{\text{qmin}^*_\mathbf{x}}
	 	\end{align}
	
	\noindent\textbf{Property 2 (Contrastive Loss)}: Ensure that the loss for all infeasible assignments is higher than the maximum value over the feasible region.\\
	Using the same analysis as above, it is easy to see that in order to satisfy property 2,  we have to ensure that:
	\begin{align}
		 \alpha_\mathbf{x}  > \frac{\text{pmax}^*_\mathbf{x}}{\text{qmin}^*_\mathbf{x}}
	\end{align}
	
		\noindent\textbf{Property 3 (Barrier Loss)}: Ensure that the loss for all infeasible assignments is higher than the maximum value over feasible assignments that lie on the boundary of size $\epsilon$ between the feasible and infeasible regions.\\
	To satisfy property 3,  we have to ensure that:
	\begin{align}
		\alpha_\mathbf{x}  > \frac{\text{pbar}^*_\mathbf{x}}{\text{qmin}^*_\mathbf{x}}
	\end{align}
 
For all the three losses (bounds on $\alpha_{\mathbf{x}}$) described above, the problems in the numerator and  denominator are separate instances of the CMPE task (subject to replacing max by min and/or converting the greater than equal to constraint with a less than equal to constraint). Since solving them exactly is impractical (and moreover, we are interested in self-supervised methods where we do not assume access to such a solver), we propose to upper bound the numerator using a feasible solution (therefore, a simple yet efficient approach is to keep track of a feasible solution during batch-style gradient descent) and lower bound the denominator using Lagrange relaxation and decomposition methods described in \cite{rahman-nips21}.  Note that lower bounds can also be computed via LP relaxations. Also, not that when a feasible solution is hard to compute, we can leverage fast variational approximations such as dual decomposition methods \cite{globerson}.
}
\begin{proposition}
If $\mathcal{L}_\mathbf{x}(\hat{\mathbf{y}})$ is consistent, i.e., $\alpha_\mathbf{x}  > \frac{p^*_\mathbf{x}}{q^*_\mathbf{x}} $ then $\min_{\hat{\mathbf{y}}}\mathcal{L}_\mathbf{x}(\hat{\mathbf{y}})=p_\mathbf{x}^*$, namely $\mathcal{L}_\mathbf{x}(\hat{\mathbf{y}})$ is an optimal loss function.
\end{proposition}
\begin{proof}
From equation \eqref{eqn:loss_1}, we have
\begin{align}
\nonumber
\min_{\mathbf{\hat{y}}}\mathcal{L}_\mathbf{x}(\hat{\mathbf{y}})&= \min  \biggl\{ \min_{\mathbf{\hat{y}}} f_\mathbf{x}(\hat{\mathbf{y}}) \:\:s.t.\:\: g_\mathbf{x}(\hat{\mathbf{y}}) \leq 0  , \\
\nonumber
&\alpha_\mathbf{x} \left (\min_{\mathbf{\hat{y}}}\:\:   f_\mathbf{x}(\hat{\mathbf{y}}) + g_\mathbf{x}(\hat{\mathbf{y}}) \:\:s.t.\:\: g_\mathbf{x}(\hat{\mathbf{y}}) > 0 \right ) \biggr \}\\
\label{eq:pf1}
& = \min \left \{p^*_{\mathbf{x}},\alpha_\mathbf{x} q^*_{\mathbf{x}}\right \}
\end{align}
Because $\mathcal{L}_\mathbf{x}(\hat{\mathbf{y}})$ is consistent, namely, $\alpha_\mathbf{x} > \frac{p^*_{\mathbf{x}}}{q^*_{\mathbf{x}}}$, we have 
\begin{align}
\label{eq:pf2}
    \min \left \{p^*_{\mathbf{x}},\alpha_\mathbf{x} q^*_{\mathbf{x}}\right \} =   p^*_{\mathbf{x}}
\end{align}
From equations \eqref{eq:pf1} and \eqref{eq:pf2}, the proof follows.
\end{proof}

\eat{
\svar{
\textbf{Proof of Proposition 4.1}: 
\begin{align}
\nonumber
    \alpha_\mathbf{x} > \frac{p^*_{\mathbf{x}}}{q^*_{\mathbf{x}}}\\
    \nonumber
    \therefore \:\alpha_\mathbf{x} q^*_{\mathbf{x}} > p^*_{\mathbf{x}}\\
    \nonumber
    \therefore\: \min \left \{p^*_{\mathbf{x}},\alpha_\mathbf{x} q^*_{\mathbf{x}}\right \} =   p^*_{\mathbf{x}}
\end{align}
We assume that $f_{\mathbf{x}}(\mathbf{y})$ and $g_{\mathbf{x}}(\mathbf{y})$ are bounded functions, namely for any assignment $(\mathbf{x},\mathbf{y})$, $l_f\leq f_\mathbf{x}(\mathbf{y}) \leq u_f$ and $l_g\leq g_\mathbf{x}(\mathbf{y}) \leq u_g$ where $-\infty < s < \infty$ and $s \in \{l_f,u_f,l_g,u_g\}$. Also, for simplicity, we assume that $f_{\mathbf{x}}(\mathbf{y})$ is a strictly positive function, namely $l_f >0$.

Thus, based on the assumptions given above, we have
\[\frac{p^*_{\mathbf{x}}}{q^*_{\mathbf{x}}} \leq \frac{u_f}{l_f} \]
and 
\[0 <\alpha_{\mathbf{x}}\leq \frac{u_f}{l_f} \]

The above assumptions will ensure that the gradients are bounded, because  $\alpha_\mathbf{x}$, $f$ and $g$ are bounded.

}

}

We assume that $f_{\mathbf{x}}(\mathbf{y})$ and $g_{\mathbf{x}}(\mathbf{y})$ are bounded functions, namely for any assignment $(\mathbf{x},\mathbf{y})$, $l_f\leq f_\mathbf{x}(\mathbf{y}) \leq u_f$ and $l_g\leq g_\mathbf{x}(\mathbf{y}) \leq u_g$ where $-\infty < s < \infty$ and $s \in \{l_f,u_f,l_g,u_g\}$. Also, for simplicity, we assume that $f_{\mathbf{x}}(\mathbf{y})$ is a strictly positive function, namely $l_f >0$.

Thus, based on the assumptions given above, we have
\[\frac{p^*_{\mathbf{x}}}{q^*_{\mathbf{x}}} \leq \frac{u_f}{l_f} 
\:\: \text{and}\:\:
0 <\alpha_{\mathbf{x}}\leq \frac{u_f}{l_f} \]

The above assumptions will ensure that the gradients are bounded, because  $\alpha_\mathbf{x}$, $f$ and $g$ are bounded, and both $p_\mathbf{x}^*$ and $q_\mathbf{x}^*$ are greater than zero. 

Next, we show how to compute an upper bound on $\alpha_\mathbf{x}$ using $\alpha_\mathbf{x}  > \frac{p^*_\mathbf{x}}{q^*_\mathbf{x}}$, thus ensuring that we have an optimal loss function. The terms in the numerator ($p_\mathbf{x}^*$) and denominator ($q_\mathbf{x}^*$) require solving two instances of the CMPE task. Since solving CMPE exactly is impractical and moreover, since we are interested in self-supervised methods where we do not assume access to such a solver, we propose to lower bound $q_\mathbf{x}^*$ and upper bound $p_\mathbf{x}^*$. 

For a given instance $\mathbf{x}$, a lower bound on $q_\mathbf{x}^*$  can be obtained using the Lagrangian relaxation method described in section \ref{sec:lower_bound} (see Eq. \eqref{eq:lagrange-cmpe_min}) for discrete variables and the Reformulation-Linearization method described in \citet{sherali1992global} for continuous variables. On the other hand, any feasible solution can serve as an upper bound for $p_\mathbf{x}^*$. A simple yet efficient approach is to begin with a loose upper bound by upper bounding the MPE task\eat{(CMPE without constraints)}: $ \max_\mathbf{y}\:f_\mathbf{x}(\mathbf{y})$, using fast algorithms such as mini-bucket elimination \cite{dechter2003mini} or fast linear programming based approximations \cite{ihler2012join,globerson2007fixing} and then keep track of feasible solutions during batch-style gradient descent.

In summary, we proposed a new loss function which uses the quantity $\alpha_\mathbf{x}$. When the neural network predicts a feasible $\hat{\mathbf{y}}$, the loss equals $f$, whereas when it predicts an infeasible $\hat{\mathbf{y}}$, the loss is such that the infeasible solution can quickly be pushed towards a feasible solution (because it uses gradients from $g$). A key advantage of our proposed loss function is that $\alpha_\mathbf{x}$ is not treated as an optimization variable, and a bound on it can be pre-computed for each example $\mathbf{x}$.

		\eat{This contrastive effect of $\alpha_\mathbf{x}$ can be amplified by using weaker bounds on either or both $p^*_\mathbf{x}$ and $q^*_\mathbf{x}$. the worst feasible solution as the upper bound on $p_\mathbf{x}^*$ and the worst infeasible solution as the lower bound. The assignment to the variables in $\mathbf{Y}$ with the lowest weight or an upper bound on it can serve as the worst lower bound on $f_\mathbf{x}(q^*_\mathbf{x})$.}
		
\subsection{Making The Loss Function Smooth and Continuous}

		The loss function defined in Eq. \eqref{eqn:loss_1} is continuous and differential everywhere except at $g_\mathbf{x}(\hat{\mathbf{y}}) = 0$. There is a \textit{jump discontinuity} at $g_\mathbf{x}(\hat{\mathbf{y}}) = 0$ since $\lim_{g_\mathbf{x}(\hat{\mathbf{y}})\rightarrow 0^-} f_\mathbf{x}(\hat{\mathbf{y}}) \neq \lim_{g_\mathbf{x}(\hat{\mathbf{y}})\rightarrow 0^+}\alpha_\mathbf{x} (f_\mathbf{x}(\hat{\mathbf{y}}) + g_\mathbf{x}(\hat{\mathbf{y}}))$. To address this issue, we propose the following continuous approximation 
        \begin{align}\label{eqn:loss_cont}
        \widetilde{\mathcal{L}}_\mathbf{x}(\hat{\mathbf{y}}) = & \Bigl((1 - \sigma(\beta g_\mathbf{x}(\hat{\mathbf{y}}))) \cdot [f_\mathbf{x}(\hat{\mathbf{y}})]\Bigr) + \\
        & \Bigl(\sigma (\beta g_\mathbf{x}(\hat{\mathbf{y}})) \nonumber  \cdot [ \alpha_\mathbf{x}(f_\mathbf{x}(\hat{\mathbf{y}}) \nonumber + \max\{0,g_\mathbf{x}(\hat{\mathbf{y}})\})]\Bigr) \nonumber
        \end{align}
		where $\sigma(.)$ is the sigmoid function and $\beta \geq 0$ is a hyper-parameter that controls the steepness of the sigmoid. At a high level, the above continuous approximation uses a sigmoid function to approximate a Heaviside step function. 
  
		
		\eat{
		\subsection{Regularized Loss for Discrete Output}
		Each predicted output $y\in\mathbf{y}$ from the neural network is a continuous quantity in $[0,1]$. One technique to obtain discrete outputs is to round each $y$ to the nearest 0 or 1. Such rounding can lead to poor discrete solutions that may be suboptimal and even infeasible. Instead, we propose to incorporate a regularizer into the loss function that will encourage the neural network model to produce discrete quantities in $\{0, 1\}$.
		The proposed regularizer is the  sum of the entropies of each individual output variable $Y \in \mathbf{Y}$ weighted by the hyperparameter $\gamma \geq 0$.  
		\begin{equation*}\label{eqn:loss_disc}
			\widehat{\mathcal{L}}_\mathbf{x}(\hat{\mathbf{y}}) = \Tilde{\mathcal{L}_{\mathbf{x}}}(\hat{\mathbf{y}}) - \gamma\sum_{\hat{y} \in \hat{\mathbf{y}}} \hat{y} \log(\hat{y}) + (1-\hat{y}) \log(1-\hat{y}) 
		\end{equation*}
  \eat{
	\input{content/algo}
\subsection{The Algorithm}	
		Algorithm 1 outlines the steps to learn a self-supervised CMPE solver called \sscmpe. \sscmpe\xspace takes as input a randomly initialized neural network $\mathcal{F}_\Theta$, dataset $\mathcal{D}$ having assignments to the evidence variables and two log-linear models $f$ and $g$ whose parameters have been appropriately transformed for the minimization task (see Preliminaries). The algorithm starts by computing the constants, $\alpha_{\mathbf{x}_i}$ for each $\mathbf{x}_i\in \mathcal{D}$ using methods described in \cite{rahman2021novel}. In each iteration, we compute the predictions $\hat{\mathbf{y}}_i$ for each $\mathbf{x}_i$ made by the deep neural network through a single forward pass. If the prediction is a feasible solution ($g(\mathbf{x}_i, \hat{\mathbf{y}}_i)$), then the loss added for $\mathbf{x}_i$ is $f(\mathbf{x}_i, \hat{\mathbf{y}}_i)$ otherwise it is $\alpha_{\mathbf{x}_i} \times(f(\mathbf{x}_i, \hat{\mathbf{y}}_i)+g(\mathbf{x}_i,\hat{\mathbf{y}}_i))$. Once the loss has been computed for all the samples, it is backpropagated to update the $\Theta$'s. In order to use a continuous loss, we can replace the lines 6-10 with the equation in \ref{eqn:loss_cont} and for the discrete loss with equation \ref{eqn:loss_disc} while providing the appropriate hyper-parameters $\beta$ and $\gamma$ as input. Instead of using the full dataset at each iteration, we used mini-batches in our experiments. At test time, given x and the value q, we adjust the potential function values as described in section 2. \tr{shift the 'q' value to the left} 
}
  \paragraph{Extension to Multiple Constraints and Adding Penalty.}
  Our proposed method can be easily extended to handle multiple constraints, and combined with popular approaches such as the penalty method. In particular, when a penalty such as $\max\{0,g_\mathbf{x}(\hat{\mathbf{y}})\}^2$ is added to the expression $ \alpha_{\mathbf{x}}(f_{\mathbf{x}}(\hat{\mathbf{y}})+ g_{\mathbf{x}}(\hat{\mathbf{y}})) \:\:\text{if}\:\: g_\mathbf{x}(\hat{\mathbf{y}}) > 0$ (see Eq. \eqref{eqn:loss_1}), it does not change the set of global optima of the loss function. Details of these extensions are provided in the supplement. \tr{Is this a correct claim @Vibhav Gogate?}
  
  }
  
  \eat{. Formally, let $g^1,\ldots,g^m$ be the multilinear functions involved in the constraints of the form $g^i_\mathbf{x}(\mathbf{y}) \leq 0$ for $i=1,\ldots,m$. Then, we can use the following loss function instead of the loss function given in Eq. \eqref{eqn:loss_1}. \begin{align}\label{eqn:loss_multiple}
		\mathcal{L}_\mathbf{x}(\hat{\mathbf{y}}) & = \begin{cases}
			& f_{\mathbf{x}}(\hat{\mathbf{y}}) \:\:\text{if}\:\: \sum_{i=1}^{m}\max\left (0,g^i_\mathbf{x}(\hat{\mathbf{y}}) \right ) = 0\\
			& \alpha_{\mathbf{x}}\left (f_{\mathbf{x}}(\hat{\mathbf{y}}) +  \sum_{i=1}^{m}\max\left (0,g^i_\mathbf{x}(\hat{\mathbf{y}}) \right )\right )\:\:\text{if}\:\: \sum_{i=1}^{m}\max\left (0,g^i_\mathbf{x}(\hat{\mathbf{y}}) \right ) > 0\\
		\end{cases}
	\end{align}
 }

%% file: content/algo.tex
\begin{algorithm}\label{algo:a1}
\caption{Self-Supervised CMPE Solver (\texttt{SS-CMPE})} 
\begin{algorithmic}[1]
\Require{$\mathcal{F}_\theta$ := A Neural Network with randomly initialized parameters $\theta$, $\mathcal{D}$ := $\{\mathbf{x}_i\}_1^N$, $f,g$ := Log-linear models} 
\Ensure{$\mathcal{F}_\theta$ trained on $\mathcal{D}$ to minimize loss in \ref{eqn:loss_1}}
\State Compute $\mathbf{\alpha} = \{\alpha_{\mathbf{x}_i}\}$ for each $\mathbf{x}_i \in \mathcal{D}$ using techniques proposed in \cite{rahman2021novel}
\For{$t = 1$ to $T$}
\Comment{T:=Maximum number of iterations}
\State $\{\hat{\mathbf{y}_i}\}_1^N =  \mathcal{F}_\Theta(\mathcal{D})$
\State $\mathcal{L} = 0$ 
\Comment{Loss function}
\For{$i = 1$ to $N$}
\If {$g(\mathbf{x}_i, \hat{\mathbf{y}}_i) < 0$}
\State $\mathcal{L} = \mathcal{L} + f(\mathbf{x}_i, \hat{\mathbf{y}}_i)$ \Comment{Constraint is satisfied} 
\Else
\State $\mathcal{L} = \mathcal{L} + \alpha_{\mathbf{x}_i} \times (f(\mathbf{x}_i, \hat{\mathbf{y}}_i) + g(\mathbf{x}_i, \hat{\mathbf{y}}_i))$ \Comment{Constraint is violated}
\EndIf
\EndFor
\State Back-propagate the $\mathcal{L}$

\EndFor
\end{algorithmic}
\end{algorithm}

%% file: content/5-experiment.tex
\input{content/opt_gap_plots}
\input{content/uai_new}
\input{content/tractable_new}
In this section, we thoroughly evaluate the effectiveness of our proposed neural networks based solvers for CMPE. We evaluate the competing methods on several test problems using three criteria: optimality gap (relative difference between the optimal solution and the one found by the method), constraint violations (percentage of time the method outputs an infeasible solution), and training and inference times. 

\eat{In this study, we conducted a comprehensive comparison of the proposed method with various state-of-the-art approaches that employ neural networks for optimization tasks. The evaluation encompassed three types of \cmpe problems, each representing different levels of complexity. The first type involved tractable probabilistic circuits, specifically cutset networks learned on widely recognized benchmarks utilized by the tractable models community.  The second type comprised intractable models with high treewidth, sourced from the UAI (Uncertainty in Artificial Intelligence) competitions. These models pose significant challenges due to their complex and intractable nature. Lastly, we considered models designed for performing adversarial attacks on classifiers. By conducting experiments across these three types of CMPE problems, we obtained a comprehensive understanding of the proposed method's performance in various contexts, highlighting its efficacy and versatility in tackling optimization tasks across different levels of complexity.}

\eat{In our comparative analysis, we evaluated the proposed \sscmpe method alongside seven baseline methods. Here, we present the results for four of these baselines. The first baseline is \ilp, where the solutions are obtained from SCIP \cite{achterberg2009scip} and Gurobi \cite{gurobi2021gurobi}, and these solutions serve as the optimal ground truth for the supervised methods. The second baseline is \slpen, which extends the supervised approaches by incorporating penalty coefficients. Next, we assess the performance of the self-supervised methods. We begin with \ssl (self-supervised penalty), followed by \pdl (primal-dual self-supervised). Additionally, we compare all of these methods with the method described in Section \ref{our-method}. For more detailed results of the other supervised methods, we have included them in the supplementary material accompanying this paper.}

\subsection{The Loss Functions: Competing Methods}
We trained several neural networks to minimize both supervised and self-supervised loss functions. We evaluated both MSE and MAE supervised losses with and without penalty coefficients (see section 3). In the main paper, we show results on the best performing supervised loss, which is MSE with penalty, denoted by \slpen (results for other supervised loss functions are provided in the supplement). 

For self-supervised loss, we experimented with the following three approaches: (1) penalty-based method, (2) ALM, which uses a primal-dual loss (PDL), and the approach described in section \ref{sec:our-method}. We will refer to these three schemes as \sslpen, \pdl, and \sscmpe, respectively. We used the experimental setup described by \citet{park2022self} for tuning the hyperparameters of \pdl and \sslpen.  For the \sscmpe method, we employed a grid search approach to determine the optimal values for $\beta$. The range of values considered for $\beta$ was $\{0.1, 1.0, 2.0, 5.0, 10.0, 20.0\}$.

Note that all methods used the same neural network architecture (described in the supplement) except \pdl, which uses two neural networks.
We obtained the ground-truth for the supervised training by solving the original \ilp problem using SCIP \cite{achterberg2009scip} and Gurobi \cite{gurobi2021gurobi}. We report the objective values of the ILP solutions in each table (see Tables \ref{tab:uai}, \ref{tab:tractable}, and \ref{tab:adv-table}). 

\subsection{Datasets and Benchmarks}
We evaluate the competing algorithms (\slpen, \sslpen, \pdl, and \sscmpe) on a number of log-linear Markov networks ranging from simple models (low treewidth) to high treewidth models. The simple models comprise of learned tractable probabilistic circuits \cite{choi2020probabilistic} without latent variables, specifically, cutset networks \cite{rahman2014cutset} from benchmark datasets used in the literature. The complex, high treewidth models are sourced from past UAI inference competitions \cite{elidan_2010_2010}. Finally, we evaluated all methods on the task of generating adversarial examples for neural network classifiers. 


\eat{\textbf{Results on Networks from the UAI Competition and MPE Tractable Models}}
\subsection{High Tree-Width Markov Networks and Tractable Probabilistic Circuits}
Our initial series of experiments focus on high treewidth Grids and Image Segmentation Markov networks from the UAI inference competitions \cite{uai14competition,uai16competition}. In this investigation, we generated CMPE problems by utilizing \eat{utilized} the model employed in the UAI competitions, denoted as $M_1$. Subsequently, $M_2$ was created by adjusting the parameters of $M_1$ while incorporating a noise parameter $\epsilon$ drawn from a normal distribution with mean 0 and variance $\sigma^2 = 0.1$. To select $q$, we randomly generated 100 samples, sorted them based on their weight, and then selected the weight of the 10th, 30th, 60th, 80th, and 90th sample as a value for $q$. We assess the impact of changing the value of $q$ in the supplement. Experiments in the main body of the paper use $q$ equal to the weight of the 80th random sample. For each network, we randomly chose 60\% of variables as evidence ($\mathbf{X}$) and the remaining as query variables ($\mathbf{Y})$. For both the UAI models and tractable probabilistic circuits, we generated 10K samples from $M_1$, and used 9K for training and 1K for testing. We used 5-fold cross validation for selecting the hyperparameters.


The results for the UAI datasets are shown in Table \ref{tab:uai}. We see that for the majority of the datasets, our method produces solutions with superior gap values compared to the other methods. Even in situations where our methods do not achieve better gap values, they exhibit fewer violations. This demonstrates the effectiveness and robustness of our methods in generating solutions that strike a balance between optimizing the objective function and maintaining constraint adherence. For certain datasets, our methods exhibit significantly lower constraint violations, even up to 10 times less than supervised methods. 


\eat{\begin{figure}[htbp]
    \centering
    \begin{subfigure}[b]{0.31\textwidth}
        \centering
        \includegraphics[width=\textwidth]{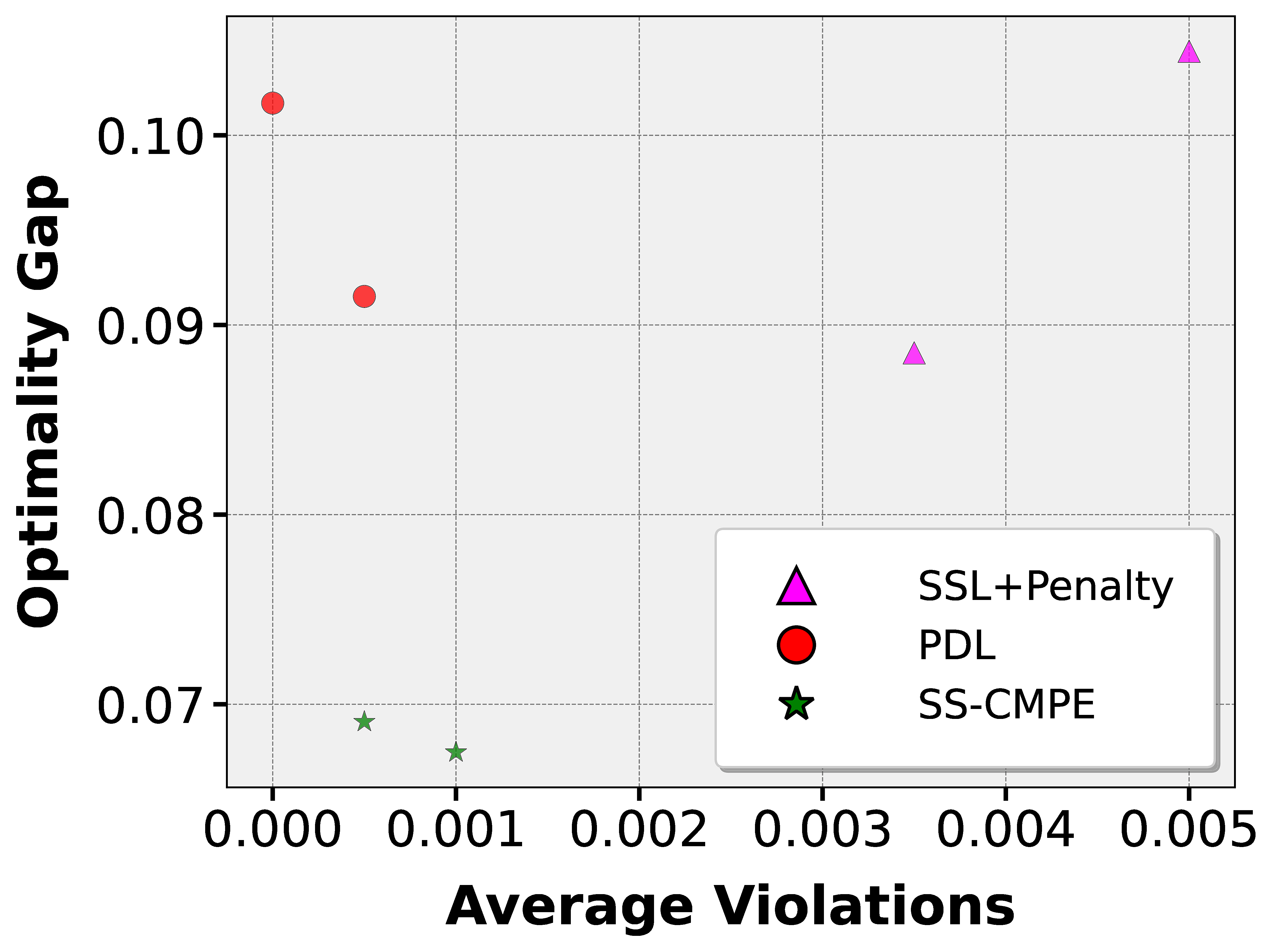}
        \caption{Optimality Gap (avg \%) and  Average Violations for Grids UAI networks}
    \end{subfigure}
        \hfill
    \begin{subfigure}[b]{0.31\textwidth}
        \centering
        \includegraphics[width=\textwidth]{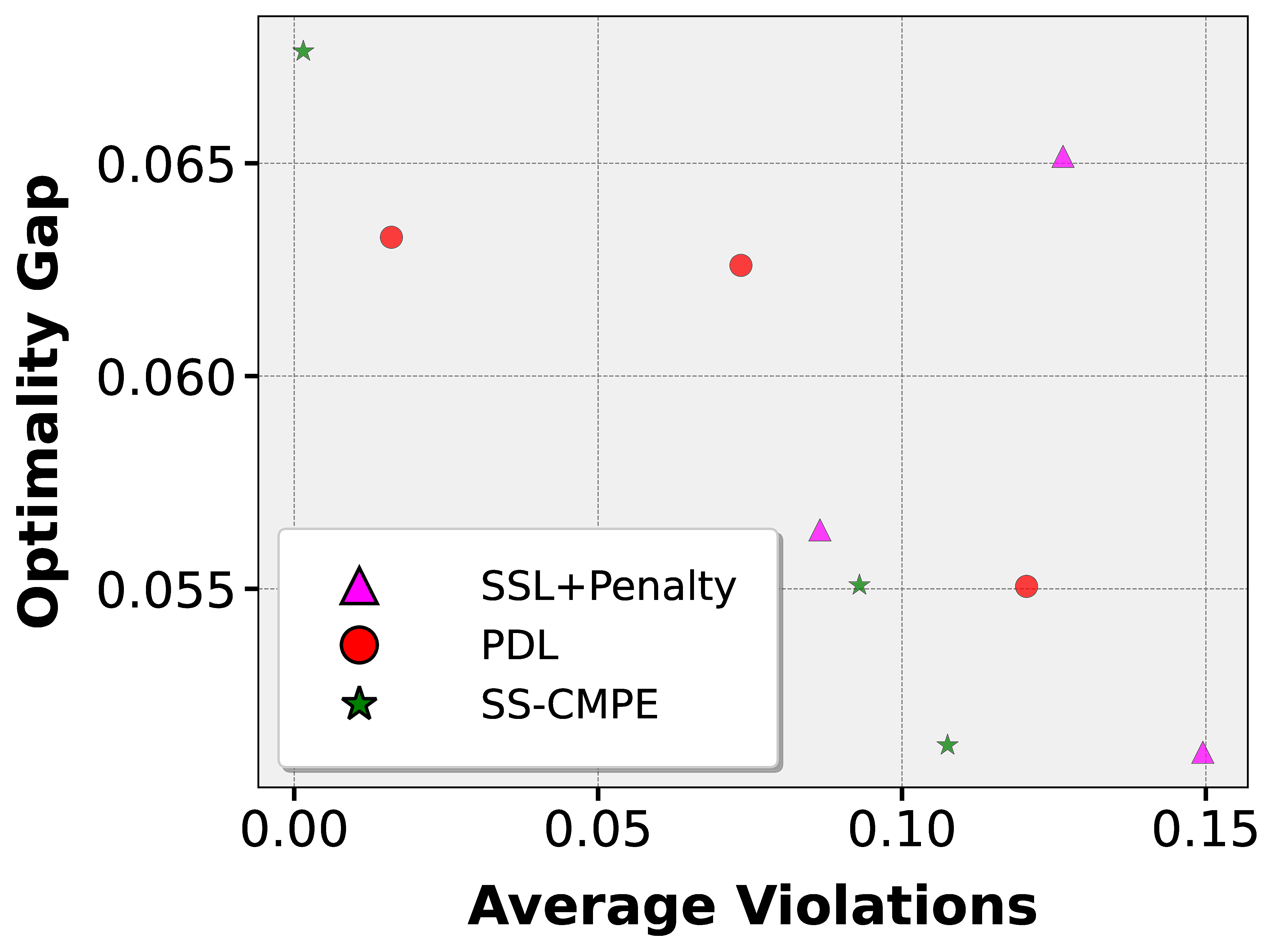}
        \caption{Optimality Gap (avg \%) and  Average Violations for Segmentation UAI networks}
    \end{subfigure}
    \hfill
    \begin{subfigure}[b]{0.31\textwidth}
        \centering
        \includegraphics[width=\textwidth]{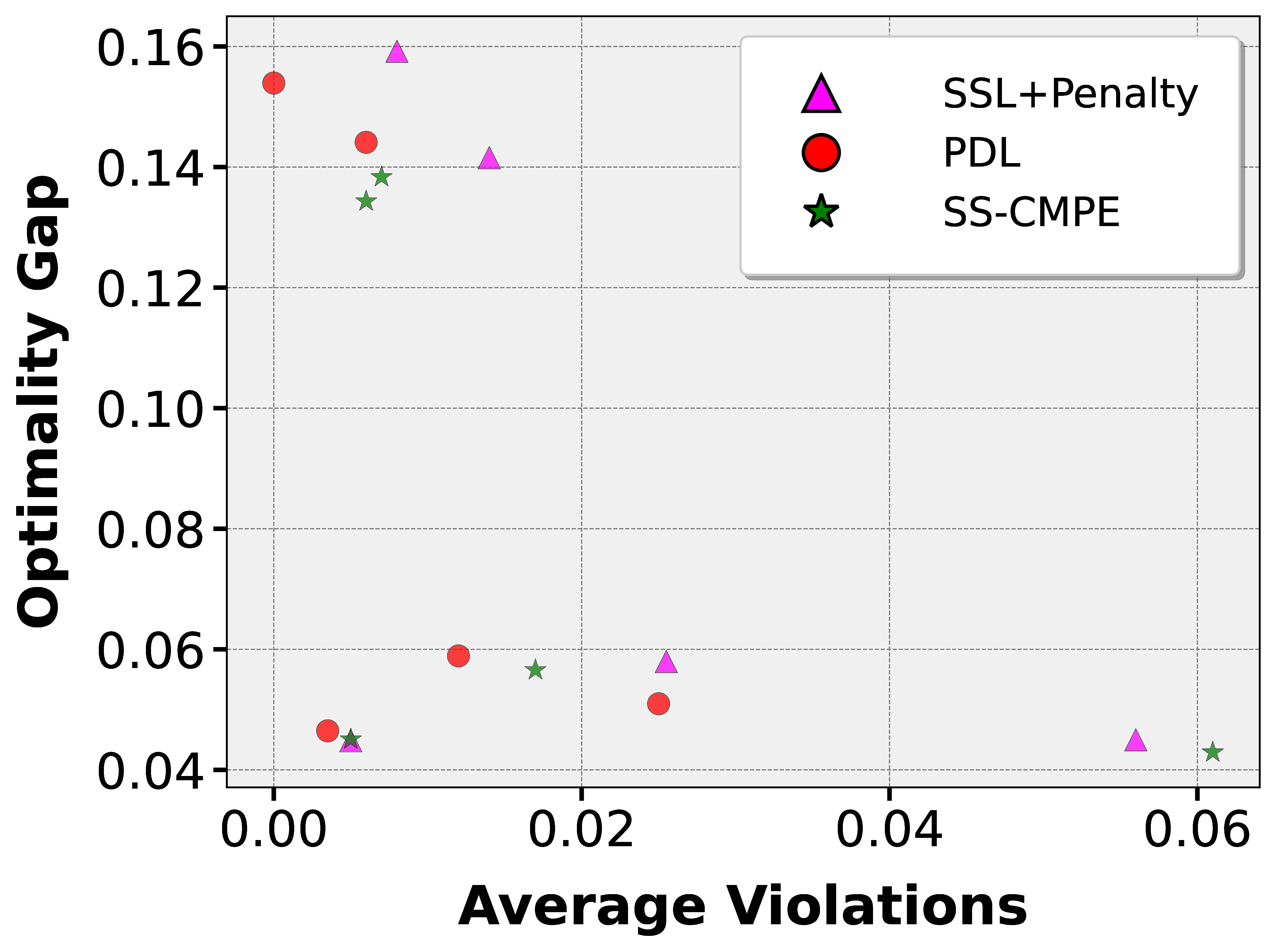}
        \caption{Optimality Gap (avg \%) and  Average Violations for Tractable Models}
    \end{subfigure}
\end{figure}}

\input{supp_tables/mnist_all}

In the next phase of our study, we employed MPE (Most Probable Explanation) tractable models, which were learned on five high-dimensional datasets (see \citet{Lowd2010Dec_1} for details in the datasets): DNA, NewsGroup (c20ng), WebKB1 (cwebkb), AD, and BBC. These learned models served as $\mathcal{M}_1$. We then applied Gaussian noise as described earlier to generate $\mathcal{M}_2$ based on $\mathcal{M}_1$. A similar trend can be observed for tractable probabilistic models in Table \ref{tab:tractable}, where our method consistently outperforms the other self-supervised methods across all datasets. Not only does our approach exhibit superior performance in terms of gap values, but it also demonstrates \eat{either }comparable \eat{ or even fewer} constraint violations.  When comparing with the supervised method, our proposed algorithm exhibits significantly fewer constraint violations while maintaining a better or comparable gap value. This emphasizes the strength of our method in effectively balancing the optimization objectives and constraint adherence, thereby offering improved overall performance compared to both the self-supervised and supervised approaches in the context of tractable probabilistic models. In Figure \ref{fig:combined_gap_vs_violation}, we present the average optimality gap and average violations for different dataset groups. It is important to note that results closer to the origin indicate better performance.

\subsection{Adversarial Modification on the MNIST Dataset}
\input{content/adversarial_new}
We also evaluated our approach on the task of adversarial example generation for discriminative classifiers, specifically neural networks.  Adversarial examples play a crucial role in assessing the robustness of models and facilitating the training of more resilient models.  The task of Adversarial Example Generation involves producing new images by making minimal modifications to input images that are mis-classified by the model. 
\eat{
\citet{rahman2021novel} showed that this problem could be reduced to \cmpe. Formally, let $\mathcal{G}$ be a differentiable, continuous function defined over a set of inputs $\mathcal{X}$. Given an assignment $\mathcal{X} = x$, we introduce the decision variable $\mathcal{D}$, which takes the value $d$ when $\mathcal{G} > 0$ and $\Bar{d}$ otherwise. In the context of adversarial attacks, our objective is to make minimal alterations to the pixel values of $x$ to generate a new image $x'$ such that the decision is flipped. We represent the distance between
the two images $x$ and $x'$ using a log-linear model

that causes a change in the decision outcome. We represent the distance between
the two images $x$ and $x'$ using a log-linear model $\mathcal{F} = \{f_1, \dots, f_n\}$. Each weight in the log-linear model is associated with a feature over $x$ and $x'$ such that $x \in \mathcal{X}$. Notably, each pair of images has four associated features: $x \cap x'$, $x \cap \Bar{x'}$, $\Bar{x} \cap x'$, and $\Bar{x} \cap \Bar{x'}$. Given an assignment $\mathcal{X} = x$ satisfying $\mathcal{G} > 0$, as well as the sets of functions $\mathcal{F}$ and $\mathcal{G}$, the task of adversarial example generation can be formulated as the following CMPE problem: maximize $\sum_{f \in \mathcal{F}} f\left(x^{\prime}\right)$ s.t. $\mathcal{G}\left(x^{\prime}\right) \leq 0$. 
}
\citet{rahman2021novel} showed that this problem can be reduced to \cmpe. Formally, let $\mathcal{G}$ be a differentiable, continuous function defined over a set of inputs $\mathcal{X}$. Given an assignment $\mathcal{X} = x$, we introduce the decision variable $\mathcal{D}$, which takes the value $d$ when $\mathcal{G} > 0$ and $\Bar{d}$ otherwise. In the context of adversarial attacks, given an image $x$, our objective is to generate a new image $x'$ such that the distance between $x$ and $x'$ is minimized and the decision is flipped (namely $\mathcal{G}<0$). We used a log-linear model $\mathcal{F}$ to represent the sum absolute distance between the pixels. Then the task of adversarial example generation can be formulated as the following CMPE problem: $\text{maximize} \sum_{f \in \mathcal{F}} f\left(x^{\prime}|x\right)$ s.t. $\mathcal{G}\left(x^{\prime}|x\right) \leq 0$.

We evaluated the algorithms using the MNIST handwritten digit dataset \citep{lecun-mnisthandwrittendigit-2010}. We trained a multi-layered neural network having >95\% test accuracy and used it as our $\mathcal{G}$ function. To generate adversarial examples corresponding to a given test example, we trained an autoencoder $A: X \longrightarrow X'$ using four loss functions corresponding to \slpen, \sslpen, \pdl and \sscmpe. We used the default train-test split (10K examples for testing and 60K for training).

Table \ref{tab:adv-table} shows quantitative results comparing our proposed \sscmpe method with other competing methods. We can clearly see that \sscmpe is superior to competing self-supervised (\sslpen and \pdl) and supervised methods (\slpen) in terms of both constraint violations and optimality gap. The second best method in terms of optimality gap is \slpen. However, its constraint violations are much higher, and its training time is significantly larger because it needs access to labeled data, which in turn requires using computationally expensive ILP solvers. The training time of \sscmpe is much smaller than \pdl (because the latter uses two networks) and is only slightly larger than \sslpen. 

Figure \ref{tab:mnist_all} shows qualitative results on adversarial modification to the MNIST digits $\{1,2,6,7,8\}$ by all the eight methods. The CMPE task minimally changes an input image such that the corresponding class is flipped according to a discriminative classifier. \mse and our proposed method \sscmpe are very competitive and were able to generate visually indistinguishable, high-quality modifications whereas the other methods struggled to do so.

\textbf{Summary:} Our experiments show that \sscmpe consistently outperforms competing self-supervised methods, \pdl and \sslpen, in terms of optimality gap and is comparable to \pdl in terms of constraint violations. The training time of \sscmpe is smaller than \pdl (by half as much) and is slightly larger than \sslpen. However, it is considerably better than \sslpen in terms of constraint violations. \sscmpe also employs fewer hyperparameters as compared to \pdl.

%% file: content/opt_gap_plots.tex
\begin{figure*}[htbp]  
    \centering
    \begin{subfigure}[b]{0.32\textwidth}
        \centering
        \includegraphics[width=\textwidth]{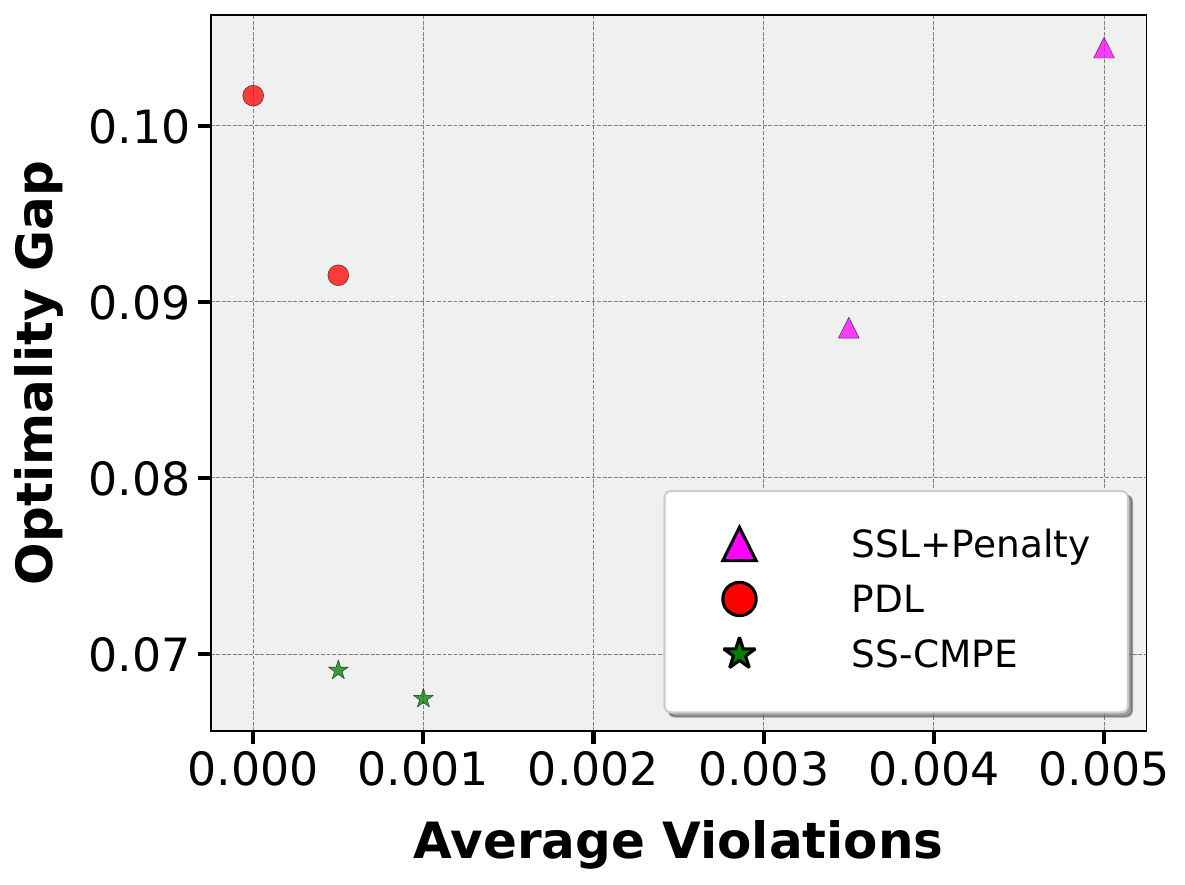}
        \caption{Grids}
    \end{subfigure}
    \hfill
    \begin{subfigure}[b]{0.32\textwidth}
        \centering
        \includegraphics[width=\textwidth]{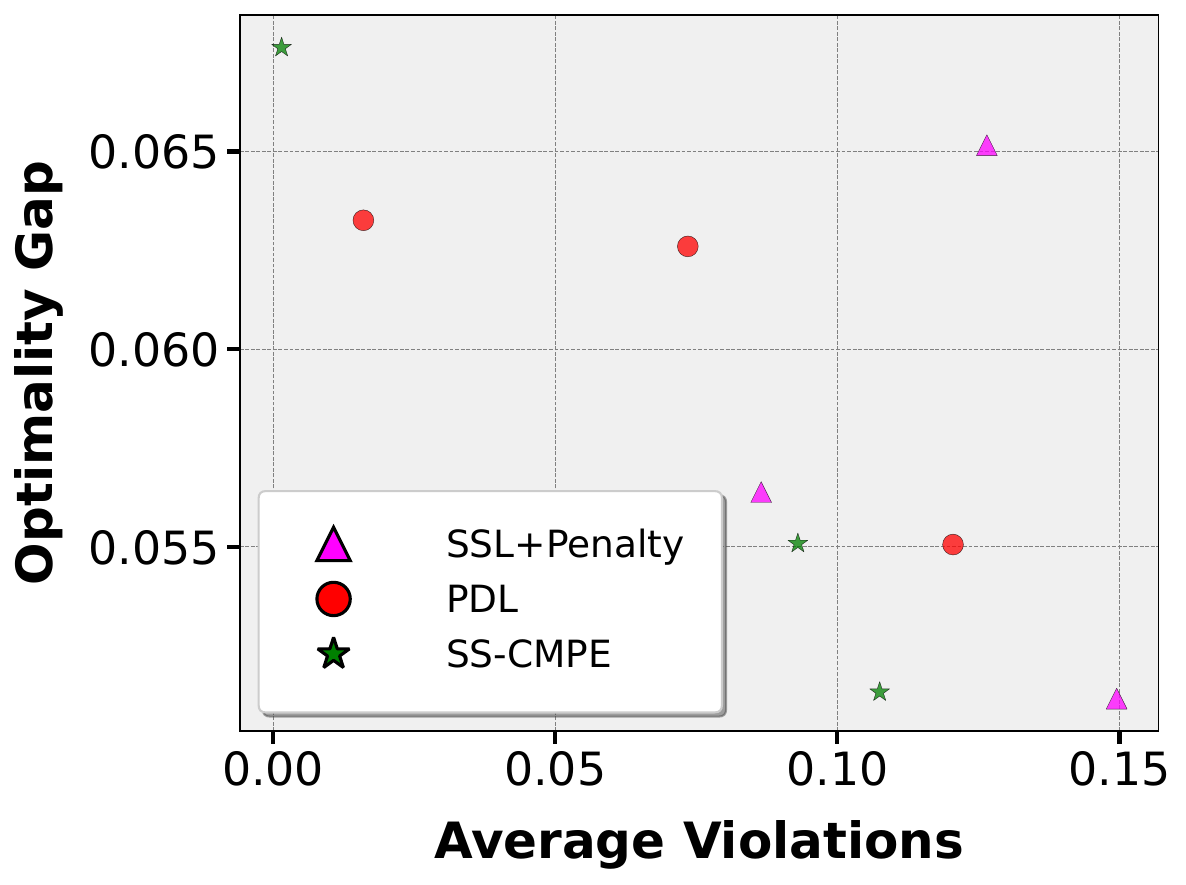}
        \caption{Segmentation}
    \end{subfigure}
    \hfill
    \begin{subfigure}[b]{0.32\textwidth}
        \centering
        \includegraphics[width=\textwidth]{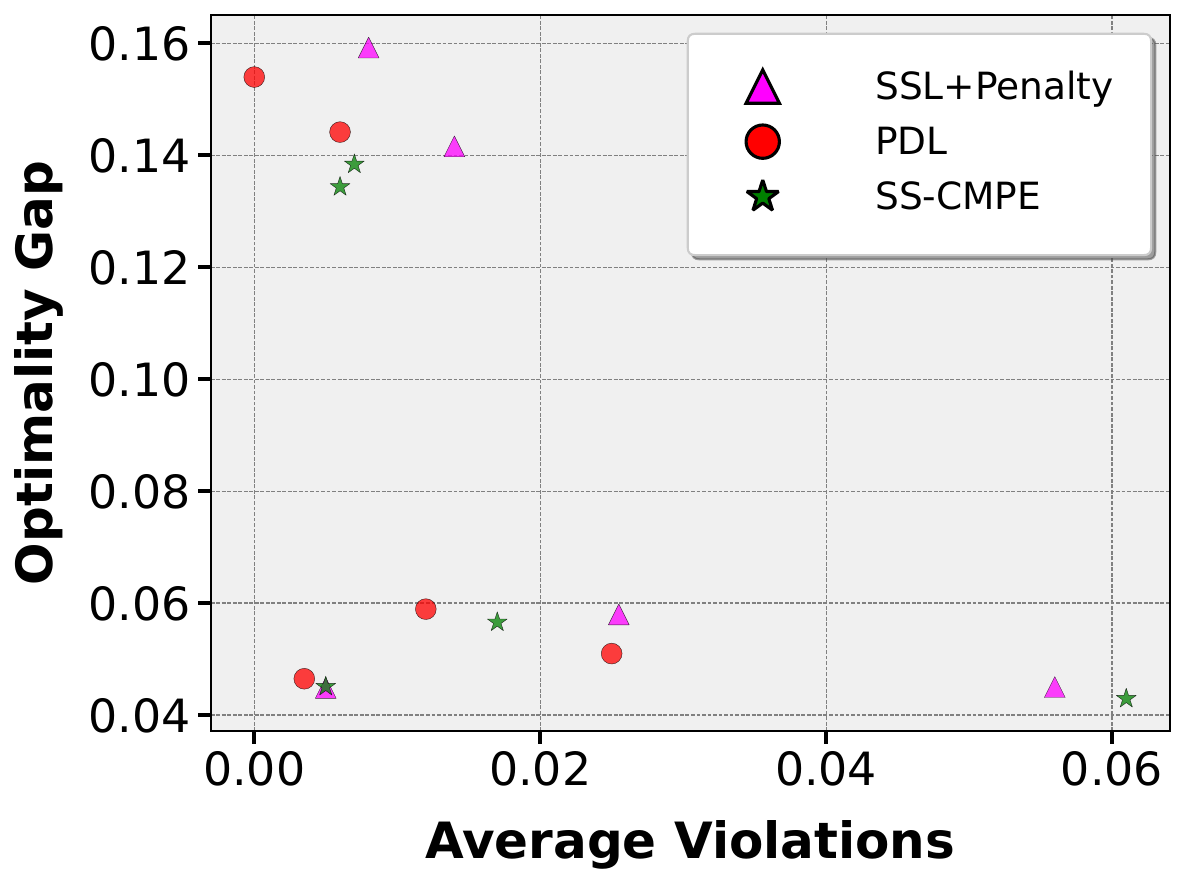}
        \caption{Tractable Models}
    \end{subfigure}
    \caption{\label{fig:combined_gap_vs_violation} Optimality Gap (avg \%) and Average Violations for Self-Supervised methods}
\end{figure*}  

%% file: content/uai_new.tex
\begin{table*}[t]
\begin{center}
\caption{Average gap and constraint violations over test samples for models from the UAI competition. $\pm$ denotes standard deviation. \textbf{Bold} values indicate the methods with the highest performance. {\ul Underlined} values denote significant violations, particularly those exceeding a threshold of 0.15. For these methods, the gap values are not considered in our analysis.}
\label{tab:uai}
\begin{tabular}{|cc|c|c|c|c|c|}
\hline
\multicolumn{2}{|c|}{Methods}                                                      & Segment12              & Segment14               & Segment15              & Grids17                & Grids18                \\ \hline
\multicolumn{2}{|c|}{ILP Obj.}                                                     & 463.454                & 471.205                 & 514.287                & 2879.469               & 4160.196               \\ \hline
\multicolumn{1}{|c|}{}                                                & Gap        & 0.053 ± 0.043          & 0.053 ± 0.043           & 0.053 ± 0.041          & 0.092 ± 0.070          & 0.082 ± 0.065          \\
\multicolumn{1}{|c|}{\multirow{-2}{*}{\slpen}} & Violations & {\ul 0.238 ± 0.426}          & {\ul 0.248 ± 0.432}           & {\ul 0.153 ± 0.361}          & 0.054 ± 0.226          & 0.053 ± 0.224          \\ \hline
\multicolumn{1}{|c|}{}                                                & Gap        & \textbf{0.051 ± 0.042}          & 0.065 ± 0.048           & \textbf{0.056 ± 0.044} & 0.089 ± 0.055          & 0.104 ± 0.062          \\
\multicolumn{1}{|c|}{\multirow{-2}{*}{\sslpen}}                       & Violations & 0.149 ± 0.357          & 0.127 ± 0.332           & 0.086 ± 0.281          & 0.004 ± 0.059          & 0.005 ± 0.071          \\ \hline
\multicolumn{1}{|c|}{}                                                & Gap        & 0.063 ± 0.050          & 0.055 ± 0.042           & 0.063 ± 0.049          & 0.102 ± 0.059          & 0.092 ± 0.061          \\
\multicolumn{1}{|c|}{\multirow{-2}{*}{\pdl}}                          & Violations & \textbf{0.073 ± 0.261} & 0.120 ± 0.326           & 0.016 ± 0.126          & \textbf{0.000 ± 0.000} & \textbf{0.001 ± 0.022}          \\ \hline
\multicolumn{1}{|c|}{}                         & Gap        & 0.055 ± 0.045          & \textbf{0.051 ± 0.040}  & 0.068 ± 0.051          & \textbf{0.067 ± 0.049} & \textbf{0.069 ± 0.051} \\
\multicolumn{1}{|c|}{\multirow{-2}{*}{\ours}}  & Violations & 0.093 ± 0.291          & \textbf{0.107  ± 0.415} & \textbf{0.002 ± 0.039} & 0.001 ± 0.032          & \textbf{0.001 ± 0.022}         \\ \hline
\end{tabular}
\end{center}
\end{table*}

%% file: content/tractable_new.tex
\begin{table*}[t]
\begin{center}
\caption{Average gap and constraint violations over test samples from tractable probabilistic models. $\pm$ denotes standard deviation. \textbf{Bold} values indicate the methods with the highest performance. {\ul Underlined} values denote significant violations, particularly those exceeding a threshold of 0.15. For these methods, the gap values are not considered in our analysis.}
\label{tab:tractable}
\begin{tabular}{|cc|c|c|c|c|c|}
\hline
\multicolumn{2}{|c|}{Methods}                               & AD                     & BBC                    & DNA                    & 20 NewsGroup                  & WebKB1                 \\ \hline
\multicolumn{2}{|c|}{ILP Obj.}                              & 2519.128               & 871.567                & 221.119                & 921.702                & 824.493                \\ \hline
\multicolumn{1}{|c|}{\multirow{2}{*}{\slpen}}  & Gap        & 0.156 ± 0.057          & 0.036 ± 0.027 & 0.143 ± 0.113          & \textbf{0.041 ± 0.031} & \textbf{0.044 ± 0.035} \\
\multicolumn{1}{|c|}{}                         & Violations & 0.135 ± 0.341          & {\ul 0.237 ± 0.425}         & {\ul 0.151 ± 0.358}          & 0.084 ± 0.277          & 0.070 ± 0.254          \\ \hline
\multicolumn{1}{|c|}{\multirow{2}{*}{\sslpen}} & Gap        & 0.159 ± 0.055          & 0.045 ± 0.033          & 0.142 ± 0.116          & 0.045 ± 0.036          & 0.058 ± 0.043          \\
\multicolumn{1}{|c|}{}                         & Violations & 0.008 ± 0.089          & 0.056 ± 0.230          & 0.014 ± 0.118          & 0.005 ± 0.071          & 0.025 ± 0.158          \\ \hline
\multicolumn{1}{|c|}{\multirow{2}{*}{\pdl}}    & Gap        & 0.154 ± 0.055          & 0.051 ± 0.036          & 0.144 ± 0.117          & 0.046 ± 0.035          & 0.059 ± 0.043          \\
\multicolumn{1}{|c|}{}                         & Violations & \textbf{0.000 ± 0.000} & \textbf{0.025 ± 0.156} & \textbf{0.006 ± 0.077} & \textbf{0.004 ± 0.059} & \textbf{0.012 ± 0.109} \\ \hline
\multicolumn{1}{|c|}{\multirow{2}{*}{\ours}}   & Gap        & \textbf{0.134 ± 0.054} & \textbf{0.043 ± 0.033}          & \textbf{0.138 ± 0.112} & 0.045 ± 0.035          & 0.057 ± 0.043          \\
\multicolumn{1}{|c|}{}                         & Violations & 0.006 ± 0.077          & 0.056 ± 0.230          & 0.007 ± 0.083          & 0.005 ± 0.071          & 0.016 ± 0.126          \\ \hline
\end{tabular}
\end{center}

\end{table*}

%% file: supp_tables/mnist_all.tex
\begin{figure*}[!t]
\setlength{\tabcolsep}{1.25pt}
\centering

\begin{tabular}{llllllll}
\multicolumn{1}{c}{\ilp} & \multicolumn{1}{c}{MSE } & \multicolumn{1}{c}{MSE$_{pen}$ } & \multicolumn{1}{c}{MAE} & \multicolumn{1}{c}{MAE$_{pen}$} & \multicolumn{1}{c}{\sslpen} & \multicolumn{1}{c}{\pdl} & \multicolumn{1}{c}{\sscmpe } \\

\includegraphics[scale=0.80]{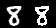} & \includegraphics[scale=0.80]{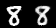} & \includegraphics[scale=0.80]{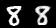} & \includegraphics[scale=0.80]{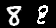} & \includegraphics[scale=0.80]{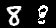}   & \includegraphics[scale=0.80]{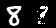}   & \includegraphics[scale=0.80]{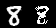}   & \includegraphics[scale=0.80]{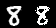}   \\
\includegraphics[scale=0.80]{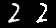} & \includegraphics[scale=0.80]{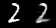} &  \includegraphics[scale=0.80]{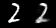} & \includegraphics[scale=0.80]{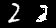} &  \includegraphics[scale=0.80]{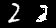} & \includegraphics[scale=0.80]{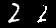} & \includegraphics[scale=0.80]{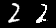} & \includegraphics[scale=0.80]{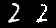} \\
\includegraphics[scale=0.80]{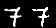} & \includegraphics[scale=0.80]{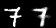}  & \includegraphics[scale=0.80]{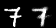} & \includegraphics[scale=0.80]{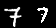}  & \includegraphics[scale=0.80]{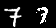} & \includegraphics[scale=0.80]{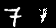} & \includegraphics[scale=0.80]{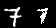} & \includegraphics[scale=0.80]{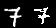} \\
\includegraphics[scale=0.80]{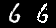} & \includegraphics[scale=0.80]{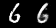} & \includegraphics[scale=0.80]{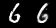} & \includegraphics[scale=0.80]{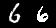} & \includegraphics[scale=0.80]{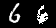} & \includegraphics[scale=0.80]{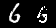} & \includegraphics[scale=0.80]{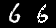} & \includegraphics[scale=0.80]{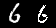} \\
\includegraphics[scale=0.80]{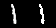} & \includegraphics[scale=0.80]{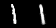} & \includegraphics[scale=0.80]{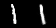} & \includegraphics[scale=0.80]{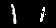} & \includegraphics[scale=0.80]{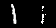} & \includegraphics[scale=0.80]{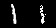} & \includegraphics[scale=0.80]{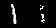} & \includegraphics[scale=0.80]{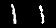}
\end{tabular}
\caption{\label{tab:mnist_all}Qualitative results on the adversarially generated MNIST digits. Each row represents an original image followed by a corresponding image generated adversarially by 8 different methods: \ilp, \mse, \msepen, \mae, \maepen, \sslpen, \pdl, and \ours.}
\end{figure*}

%% file: content/adversarial_new.tex
\begin{table}[!hbt]
\centering
\small
\caption{Performance comparison of supervised and self-supervised methods. The table presents the average objective value, gap, and constraint violations over the test examples, along with the training and inference time required for each method for adversarial example generation. \textbf{Bold} values signify the methods that achieved the best scores.}
\resizebox{0.6\linewidth}{!}{%
\begin{tabular}{|l|l|l|l|ll|}
\hline
\multirow{2}{*}{Methods} & \multirow{2}{*}{Obj. Value} & \multirow{2}{*}{Gap} & \multirow{2}{*}{Violation} & \multicolumn{2}{c|}{Time in seconds} \\ \cline{5-6} 
 &  &  &  & \multicolumn{1}{l|}{Train} & Inf. \\ \hline
\ilp & 30.794 & 0.000 & 0.000 & \multicolumn{1}{l|}{NA} & 5.730 \\ \hline
\slpen & 63.670 & 1.069 & 0.071 & \multicolumn{1}{l|}{57534.4} & 0.003 \\ \hline
\sslpen & 76.316 & 1.480 & 0.052 & \multicolumn{1}{l|}{\textbf{469.540}} & 0.003 \\ \hline
\pdl & 66.400 & 1.158 & 0.055 & \multicolumn{1}{l|}{839.025} & 0.003 \\ \hline
\ours & \textbf{62.400} & \textbf{1.028} & \textbf{0.021} & \multicolumn{1}{l|}{520.149} & 0.003 \\ \hline
\end{tabular}%
}
\label{tab:adv-table}
\end{table}

%% file: content/2-related-work.tex
\tr{
\begin{enumerate}
    \item our previous CMPE papers 
    \item Self-supervised paper and citations therein. 
    \item Neural set paper and citations therein.
\end{enumerate}}

%% file: content/6-conclusion.tex
In this paper, we proposed a new self-supervised learning algorithm for solving the constrained most probable explanation task which at a high level is the task of optimizing a multilinear polynomial subject to a multilinear constraint. Our main contribution is a new loss function for self-supervised learning which is derived from first principles, has the same set of global optima as the CMPE task, and operates exclusively on the primal variables. It also uses only one hyperparameter in the continuous case and two hyperparameters in the discrete case.  Experimentally, we evaluated our new self-supervised method with penalty-based and Lagrangian duality-based methods proposed in literature and found that our method is often superior in terms of optimality gap and training time (also requires less hyperparameter tuning) to the Lagrangian duality-based methods and also superior in terms of optimality gap and the number of constraint violations to the penalty-based methods.  

Our proposed method has several limitations and we will address them in future work. First, it requires a bound for $\alpha_{\mathbf{x}}$. This bound is easy to obtain for graphical models/multilinear objectives but may not be straightforward to obtain for arbitrary non-convex functions. Second, the ideal objective in the infeasible region should be proportional to $g_\mathbf{x}(\mathbf{y})$ but our method uses $\alpha_{\mathbf{x}} (f_\mathbf{x}(\mathbf{y})+ g_\mathbf{x}(\mathbf{y}))$. 

%% file: supplement_content.tex
\renewcommand\thesection{\Alph{section}}
\renewcommand\thesubsection{\thesection.\Alph{subsection}}
\setcounter{section}{0}

\renewcommand{\algorithmicrequire}{\textbf{Input:}}
\renewcommand{\algorithmicensure}{\textbf{Output:}}
\renewcommand{\thefigure}{SF\arabic{figure}} 
\renewcommand{\thetable}{ST\arabic{table}} 
\input{supplement_theory}
\section{EXPERIMENTAL SETUP AND DETAILS}
\subsection{Dataset and Model Description}
\eat{We conducted our experimentation using a selection of trained tractable probabilistic circuits trained on DNA \citep{VanHaaren2012_2, Ucla-Starai2023May}, NewsGroup (c20ng) \citep{Lowd2010Dec_1, Ucla-Starai2023May}, WebKB1 (cwebkb) \citep{Lowd2010Dec_1, Ucla-Starai2023May}, AD \citep{VanHaaren2012_2, Ucla-Starai2023May}, and BBC \citep{VanHaaren2012_2, Ucla-Starai2023May}.} 

Table \ref{tab:supp-dataset-descrip} provides a comprehensive overview of the characteristics of each binary dataset, including the number of variables and functions present in each dataset. These datasets were specifically chosen to provide diverse and representative examples for evaluating the performance and scalability of our algorithms.

We used the following two classes of Markov networks from the UAI competitions \cite{elidan_2010_2010,uai14competition,uai16competition}: Ising models (Grids) and Image Segmentation networks. Specifically, we used the Grids\_17 and Grids\_18 networks and Segmentation\_12, Segmentation\_14 and  Segmentation\_15 networks. 

We learned MPE tractable cutset networks without latent variables using the scheme of \citet{rahman2014cutset} on five high-dimensional datasets:
DNA \citep{VanHaaren2012_2, Ucla-Starai2023May}, NewsGroup (c20ng) \citep{Lowd2010Dec_1, Ucla-Starai2023May}, WebKB1 (cwebkb) \citep{Lowd2010Dec_1, Ucla-Starai2023May}, AD \citep{VanHaaren2012_2, Ucla-Starai2023May}, and BBC \citep{VanHaaren2012_2, Ucla-Starai2023May}. These datasets are widely used in the probabilistic circuits literature \cite{Lowd2010Dec_1}. Note that CMPE is intractable on these models even though MPE is tractable.
\input{supp_tables/datasets}

\subsection{Data Generation}

Recall that the CMPE problem uses two Markov networks $M_1$ and $M_2$, and a value of $q$. We used the original Markov networks (chosen from the UAI competitions or learned from data) as $M_1$ and generated $M_2$ by adding a value $v$, which was randomly sampled from a normal distribution with mean $0$ and variance $0.1$, to each entry in each potential in $M_1$. We used the following strategy to generate $q$. We generated 100 random samples from $M_2$, sorted them according to their weight w.r.t. $M_2$, and then chose the 10th, 30th, 60th, and 90th sample as a value for $q$. At a high level, as we go from the 10th sample to the 90th sample, namely as $q$ increases, the constraint (weight w.r.t. $M_2$ is less than or equal to $q$) becomes less restrictive. In other words, as we increase $q$, the set of feasible solutions increases (or stays the same). \eat{

For each value of $q$, we use either $30$ or $60$\% of the variables as evidence variables $\mathbf{X}$ and the remaining as $\mathbf{Y}$.}
For each value of $q$, we use $60$\% of the variables as evidence variables $\mathbf{X}$ and the remaining as $\mathbf{Y}$.
\eat{
As the number of evidence variables increases, the number of feasible solutions is likely to decrease (because it is more constrained); however, the problem may become easier to optimize because we need to predict values of fewer variables (since the number of query variables is reduced). Thus, in our setup, each UAI network yields $10$ CMPE problems ($5$ values for $q$ and $2$ values for number of evidence variables).}

For each CMPE problem, we generated 10000 samples, used the first 9000 samples for training and the remaining 1000 samples for testing. For the supervised methods, we generated the optimal assignment to $\mathbf{Y}$ using an integer linear programming solver called SCIP \cite{achterberg2009scip}. 

For our proposed scheme, which we call \sscmpe, we used approach described in Section \ref{sec:supp_upper_lower_bound} to find the upper bound of $p^*_\mathbf{x}$ and the lower bound of $q^*_\mathbf{x}$

\eat{
For our proposed scheme, which we call \sscmpe, we used the following approach to generate an upper bound on $\frac{p^*_\mathbf{x}}{q^*_\mathbf{x}}$ (recall that the condition in our algorithm is $\alpha_\mathbf{x} > \frac{p^*_\mathbf{x}}{q^*_\mathbf{x}}$). For upper bounding $p^*_\mathbf{x}$, we first find an upper bound on $\max_\mathbf{y} f_\mathbf{x}(\mathbf{y})\:\:\text{s.t.}\:\: g_\mathbf{x}(\mathbf{y})\geq 0$ using the method discussed in \citet{rahman2021novel}, and later refine this bound while training as the neural network starts generating feasible solutions (note that the weight of any feasible solution is an upper bound on $p^*_\mathbf{x}$). For lower bounding $q^*_\mathbf{x}$, we again use the methods described in \citet{rahman2021novel}. More specifically, \citet{rahman2021novel} describe two approaches for upper bounding CMPE (or lower bounding CMPE if we express it as a minimization problem). In order to improve the quality of upper and lower bounds, when we compute upper bounds, we use the minimum over the upper bounds output by the two approaches, and when we compute lower bounds, we use the maximum.
}

Note that CMPE is a much harder task than MPE. Our scheme can be easily adapted to MPE, all we have to do is use $f$ to yield a supervised scheme.

\eat{
\textbf{Tractable Probabilistic Circuits}\\
 We used these networks as $M_1$. To construct $M_2$, we modified the
parameters of $M_1$ using a noise parameter sampled from a  Gaussian distribution with $0$ mean and a variance of $0.1$. 

Then, we used the same approach that we used for Grids and Image segmentation benchmarks (described above) to generate \sva{5} \eat{10} CMPE problems per network as well as samples, exact values and upper bound on $\alpha_\mathbf{x}$.
}

\subsection{Architecture Design and Training Procedure}
In our experimental evaluations, we employed a Multi-Layer Perceptron (MLP) with a Rectified Linear Unit (ReLU) activation function for all hidden layers. The final layer of the MLP utilized a sigmoid activation function, as it was necessary to obtain outputs within the range of [0, 1] for all our experiments. Each fully connected neural network in our study consisted of three hidden layers with respective sizes of [128, 256, and 512]. We maintained this consistent architecture across all our supervised \cite{zamzam2019learning, nellikkath2021physicsinformed} and self-supervised \cite{park2022self, donti2021dc3} methods. It is important to highlight that in the adversarial modification experiments, the neural network possessed an equal number of inputs and outputs, specifically set to $28 \times 28$ (size of an image in MNIST). However, in the remaining two experiments concerning probabilistic graphical models, the input size was the number of evidence variables ($|\mathbf{X}|$), while the output size was $|\mathbf{Y}|$. 

For PDL \cite{park2022self}, the dual network had one hidden layer with 128 nodes. The number of outputs of the dual network corresponds to the number of constraints in the optimization problem.  It is worth emphasizing that our method is not constrained to the usage of Multi-Layer Perceptrons (MLPs) exclusively, and we have the flexibility to explore various neural network architectures. This flexibility allows us to consider and utilize alternative architectures that may better suit the requirements and objectives of other optimization tasks.

Regarding the training process, all methods underwent 300 epochs using the Adam optimizer \citep{Kingma2014Dec} with a learning rate of $1e^{-3}$. 
We employed a Learning Rate Scheduler to dynamically adapt the learning rate as the loss reaches a plateau. The training and testing processes for all models were conducted on a single NVIDIA A40 GPU.

\subsection{Hyper-parameters}
The number of instances in the minibatch was set to 128 for all the experiments. We decay the learning rate in all the experiments by $0.9$ when the loss becomes a plateau. Given the empirical observations that learning rate decay often leads to early convergence in most cases and does not yield beneficial results for the supervised baselines, we have made the decision not to apply learning rate decay to these methods. This choice is based on the understanding that the baselines perform optimally without this particular form of learning rate adjustment. 
For detailed information regarding the hyper-parameters utilized in the benchmarking methods, we refer readers to the corresponding papers associated with each method. As stated in the main text, the optimal hyperparameters were determined using a grid search approach. For the \sscmpe method, for each dataset, the hyperparameters were selected from the following available options -
\begin{itemize}
    \item $\beta$ - $\{0.1, 1.0, 2.0, 5.0, 10.0, 20.0\}$
    \item $\rho$ - $\{0.01, 0.1, 1, 10, 100\}$
\end{itemize}
In the optimization problem of \sscmpe, the parameter $\rho$ is employed to penalize the violation of constraints. The methodology for this approach, denoted as \sscmpepen, is explained in detail in Section \ref{sec:extensions}. The corresponding experiments are presented in tables \ref{tab:sup-seg12-q} through \ref{dna-q}.

\subsection{The Loss Function: Competing Methods}

We evaluated eight different loss functions, including our method to train a deep neural network to solve our CMPE tasks. The loss functions used for supervised training are 1) Mean-Squared-Error (\mse), and 2) Mean-Absolute-Error (\mae). Both of these losses were then extended to incorporate penalty terms as suggested by \cite{nellikkath2021physicsinformed}. We denote them as \msepen and \maepen. We evaluated the self-supervised loss proposed by \cite{donti2021dc3} (\sslpen) and by \cite{park2022self} (\pdl). Finally, we extend our self-supervised CMPE loss function to incorporate the penalty term $\max\{0,g_\mathbf{x}(\hat{\mathbf{y}})\}^2$ (see section \ref{sec:extensions}). We denote it as \sscmpepen.



\section{EXAMINING THE INFLUENCE OF \textit{q}: EVALUATING THE PERFORMANCE OF OUR PROPOSED METHOD FOR CHALLENGING PROBLEM INSTANCES}
To determine the value of $q$, a total of 100 random samples were generated, and their weights were calculated. Subsequently, the samples were sorted based on their weight in ascending order. The weight values corresponding to the 10th, 30th, 60th, and 90th samples were then chosen as the values for $q$. For each value of $q$, we compare the average gap and violations obtained by our method (\sscmpe and \sscmpepen) against six other supervised and self-supervised methods. \eat{The subsequent tables present an examination of the impact of different values of $q$ on the performance of both supervised and self-supervised methods.}
Tables \ref{tab:sup-seg12-q} through \ref{dna-q} show the scores obtained by each of the eight methods along with their standard deviations on the generated test problems. This study investigates the performance of each method in finding near-optimal feasible solutions for difficult problems which is directly controlled by the percentile rank of $q$; problems with a $q$ value in the 10th and 30th percentile are considered harder problems to solve as the size of the feasible region is considerably smaller than the size of the infeasible region. As a result all methods have higher violations on these problems than the problems with a $q$ value in the 60th and 90th percentile.

Table \ref{tab:summary_cmpe_vs_sup} presents a summary of the performances of \sscmpe and other \textit{supervised} methods based on their average gap and the average number of violations on the test data for different values of q. We compute the minimum gap and violations achieved among the four supervised methods \mse, \msepen, \mae, and \maepen and label them as the \texttt{best supervised} method. We choose the minimum gap and violations among \sscmpe and \sscmpepen and label them under the unified term \texttt{best \sscmpe}. We observe that \texttt{best \sscmpe}  consistently has significantly lower violations than the best supervised method in all the problem instances, and its gap is often comparable to the gap achieved by the best supervised method, winning when compared to the average gap. 

Table \ref{tab:summary_cmpe_vs_ssl} presents a similar summary of the performances of \sscmpe and other \textit{self-supervised} methods. We compute the minimum gap and violations achieved among the 2 self-supervised methods \sslpen and \pdl and label them as the \texttt{best SSL} method. As before, we choose the minimum gap and violations among \sscmpe and \sscmpepen and label them under the unified term \texttt{best \sscmpe}. Although self-supervised methods have larger gaps compared to supervised methods but lesser violations, we observe that \sscmpe continues to consistently achieve significantly lower violations than the best performing self-supervised method in all the problem instances, and its gap is often comparable to the gap achieved by the best supervised method, winning when compared to the average gap.\\

Finally, in table \ref{tab:summary}, we present a quick summary of the performances of our best \sscmpe\xspace method vs all \texttt{other} methods. We choose the minimum gap and violations achieved among the 6 \texttt{other} supervised and self-supervised methods and the minimum gap and violations among \sscmpe and \sscmpepen. In all problems and $q$ values, the best \sscmpe has significantly lower violations compared to other methods while having a very competitive gap.  

\eat{

\section{Investigating the Impact of Incorporating Penalty Terms into our Loss Function}
Furthermore, we will investigate the behavior of \sscmpe when penalties are incorporated into the loss function. As previously stated in the main paper, the inclusion of a penalty term, such as $\max(0,g_\mathbf{x}(\hat{\mathbf{y}}))^2$, to the expression $ \alpha_{\mathbf{x}}(f_{\mathbf{x}}(\hat{\mathbf{y}})+ g_{\mathbf{x}}(\hat{\mathbf{y}})) ::\text{if}:: g_\mathbf{x}(\hat{\mathbf{y}}) > 0$, has no impact on the set of global optima for the loss function. By analyzing the performance and behavior of \sscmpe with the addition of penalty terms, we aim to gain insights into how the incorporation of penalties affects its optimization properties and overall performance.}

\subsection{\eat{Comparative Analysis of Optimality Gaps on Feasible Solutions: Evaluating the Performance of Self-Supervised Methods}
The Feasible-Only Optimality Gaps: Comparing Self-Supervised Approaches}

From the results presented in tables \ref{tab:summary_cmpe_vs_sup} through \ref{dna-q}, we observe that self-supervised approaches produce more feasible solutions compared to supervised approaches. In this section, we present the results of a controlled study that shows how each of the self-supervised approaches perform in terms of finding optimal solutions in the feasible region. \\

We selected a subset of problems from the test set on which all self-supervised methods, namely, \sslpen \cite{donti2021dc3}, \pdl \cite{park2022self} and our method \sscmpe and \sscmpepen, obtained feasible solutions and this was done for each possible value of $q$. We then computed their gaps and compare them via figure \ref{fig:supp_only_feasible}. Among the three methods analyzed, \sscmpe and \sscmpepen consistently exhibits superior performance across the majority of cases. Its optimality gaps are significantly smaller compared to the other two methods. This finding suggests that \sscmpe is more effective in minimizing the objective value and achieving solutions closer to optimality for the given examples.

\eat{
Figure \ref{fig:supp_only_feasible}, we examine and compare the optimality gaps for a specific set of examples on which all the self-supervised methods have attain feasible solutions. \eat{These examples were carefully selected to ensure that only feasible solutions were generated by all the methods under consideration.} By focusing on this subset, we aim to provide a fair and unbiased evaluation of the self-supervised methods.
}

\eat{
The selection of examples with feasible solutions for all methods allows for a meaningful and equitable comparison. Since all the problems have optimal solutions, any differences in the optimality gaps can be attributed to the inherent capabilities of the self-supervised methods themselves. This controlled evaluation provides valuable insights into the relative strengths and weaknesses of each method in optimizing the objective value.

Overall, the results depicted in Figure \ref{fig:supp_only_feasible} highlight the advantage of \sscmpe in terms of minimizing the objective value and support the claim that it outperforms the other two self-supervised methods across a wide range of scenarios.
}

\input{supp_tables/ours_supervised}
\input{supp_tables/our_self_supervised}
\input{supp_tables/summary}

\include{supp_tables/SSL-feasible-only-figures}

\input{supp_tables/seg12-q}

\input{supp_tables/seg14-q}

\input{supp_tables/seg15-q}

\input{supp_tables/grids17-q}

\input{supp_tables/grids18-q}
\input{supp_tables/ad}
\input{supp_tables/bbc}
\input{supp_tables/c20ng}
\input{supp_tables/cwebkb}
\input{supp_tables/dna}
\clearpage

\subsection{Optimality Gap And Violations in Self-Supervised Methods for Different q Values}

\begin{figure*}[htb]
    \centering
    \begin{subfigure}[b]{0.31\textwidth}
        \centering
        \includegraphics[width=\textwidth]{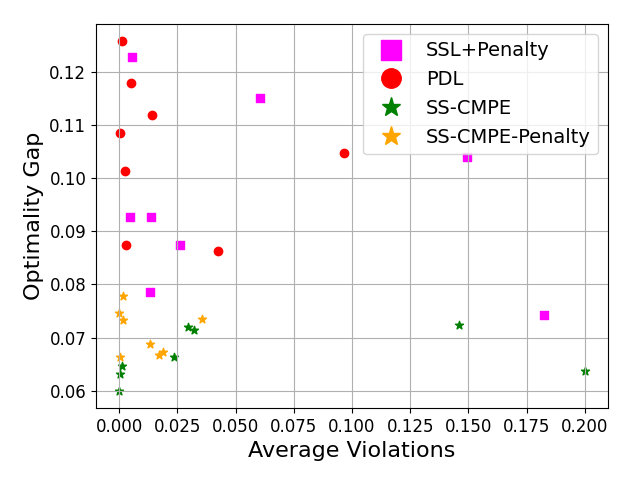}
        \caption{Optimality Gap (avg \%) and  Average Violations for Grids UAI networks}
    \end{subfigure}
        \hfill
    \begin{subfigure}[b]{0.31\textwidth}
        \centering
        \includegraphics[width=\textwidth]{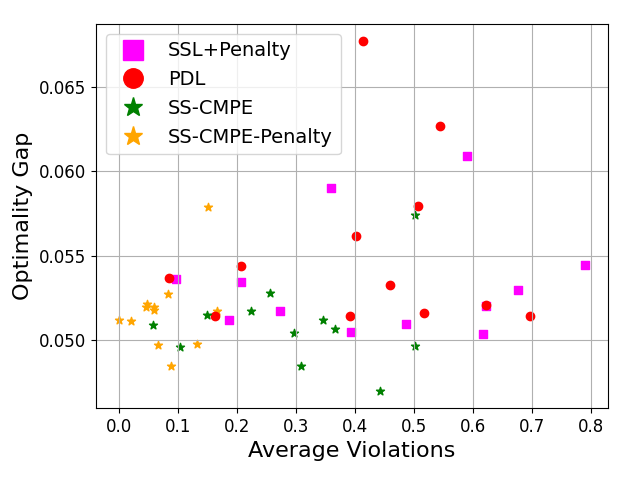}
        \caption{Opt. Gap (avg \%) and  Avg. Violations for Segmentation UAI networks}
    \end{subfigure}
    \hfill
    \begin{subfigure}[b]{0.31\textwidth}
        \centering
        \includegraphics[width=\textwidth]{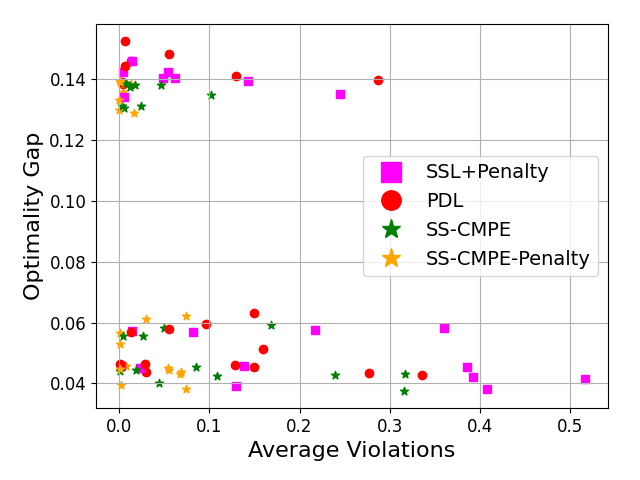}
        \caption{Optimality Gap (avg \%) and  Average Violations for Tractable Models}
    \end{subfigure}
    \caption{Visualization of Optimality Gap (average \%) and Average Violations for Self-Supervised Methods across different q values. Points closer to the origin indicate better performance.}
    \label{fig:supp_gap_vs_violations}
\end{figure*}

In the scatter plots depicted in Figure \ref{fig:supp_gap_vs_violations}, three distinct evaluations of the optimality gap against the average violations for various self-supervised methods across different q values are visualized. Points positioned closer to the origin indicate better performance, with reduced optimality gaps and fewer violations. In Figure \ref{fig:supp_gap_vs_violations}(a), focused on Grids UAI networks, the \sscmpepen method generally occupies a position near the origin, indicating its commendable performance in this setting. The \sscmpe method exhibits a comparable performance to the \sscmpepen method, with occasional high levels of violations observed in two instances.

In Figure \ref{fig:supp_gap_vs_violations}(b), showcasing the Segmentation UAI networks, the \sscmpe and \sscmpepen methods again demonstrate superiority, particularly evident by their prevalence near the origin. Finally, in Figure \ref{fig:supp_gap_vs_violations}(c) related to tractable models, the \sscmpepen method often achieves optimal placement close to the origin, reflecting a balanced performance. These evaluations provide critical insights into the effectiveness and robustness of the proposed self-supervised methods across different problems.

\eat{
\section{ADVERSARIAL GENERATION OF MNIST DIGITS: A QUALITATIVE STUDY}
Figure \ref{tab:mnist_all} shows empirical results on adversarial modification to the MNIST digits $\{1,2,6,7,8\}$ by all the eight methods. The CMPE task minimally changes an input image such that the corresponding class is flipped according to a discriminative classifier. \mse and our proposed method \sscmpe are very competitive and were able to generate visually indistinguishable, high-quality modifications whereas the other methods struggled to do so. 
\input{supp_tables/mnist_all}
}




%% file: supplement_theory.tex
\eat{
\section{ASSUMPTIONS AND PROOF OF PROPOSITION 1}
\paragraph{Assumptions.}

We make the following assumptions in order to ensure that our proposed algorithm is \textit{well-behaved}.

We assume that $f$ and $g$ are bounded functions, namely for any assignment $(\mathbf{x},\mathbf{y})$, $l_f\leq f(\mathbf{x},\mathbf{y}) \leq u_f$ and $l_g\leq g(\mathbf{x},\mathbf{y}) \leq u_g$ where $-\infty < s < \infty$ and $s \in \{l_f,u_f,l_g,u_g\}$. Also, for simplicity, we assume that $f$ is a strictly positive function, namely $l_f >0$. Note that we can always ensure that $f$ satisfies this criteria, because we can always add a constant to each log-potential of a Markov network and its distribution remains unchanged. 

Thus, based on the assumptions given above, we have
\[\frac{p^*_{\mathbf{x}}}{q^*_{\mathbf{x}}} \leq \frac{u_f}{l_f} \]
and 
\[0 <\alpha_{\mathbf{x}}\leq \frac{u_f}{l_f} \]

The above assumptions will ensure that the gradients are bounded, because  $\alpha_\mathbf{x}$, $f$ and $g$ are bounded.

\textbf{Proof of Proposition 4.1}: 
\begin{align}
\nonumber
    \alpha_\mathbf{x} > \frac{p^*_{\mathbf{x}}}{q^*_{\mathbf{x}}}\\
    \nonumber
    \therefore \:\alpha_\mathbf{x} q^*_{\mathbf{x}} > p^*_{\mathbf{x}}\\
    \nonumber
    \therefore\: \min \left \{p^*_{\mathbf{x}},\alpha_\mathbf{x} q^*_{\mathbf{x}}\right \} =   p^*_{\mathbf{x}}
\end{align}
}
\section{DERIVING UPPER BOUNDS FOR $p^*_{\mathbf{x}}$ AND LOWER BOUNDS FOR $q^*_{\mathbf{x}}$}
\label{sec:supp_upper_lower_bound}
To compute the value of $ \alpha_{\mathbf{x}} $, we utilize the following equation:
\begin{equation}
       \alpha_{\mathbf{x}} > \frac{p^*_{\mathbf{x}}}{q^*_{\mathbf{x}}}  
\end{equation}

We choose to set a lower bound for the denominator $ q_{\mathbf{x}}^* $ and an upper bound for the numerator $ p_{\mathbf{x}}^* $ due to the computational challenges of finding exact solutions for the CMPE task.

To estimate an upper bound for the optimal value ($ p_{\mathbf{x}}^* $) of the constrained optimization problem

\begin{equation}    
\underset{\hat{\mathbf{y}}}{\text{min}} \: f_{\mathbf{x}}(\hat{\mathbf{y}}) \: \text{s.t.} \: g_{\mathbf{x}}(\hat{\mathbf{y}}) \leq 0,
\end{equation}

we begin by seeking a loose upper bound through solving the unconstrained task $ \max_\mathbf{y}\:f_\mathbf{x}(\hat{\mathbf{y}}) $ by utilizing  mini-bucket elimination \cite{dechter2003mini}. Subsequently, feasible solutions are tracked during batch-style gradient descent to refine the initial upper bound (note that the weight of any feasible solution is an upper bound on $p^*_\mathbf{x}$). For each iteration, the feasible solution with the optimal objective value for each example is stored and subsequently utilized.

To derive a lower bound for $ q_{\mathbf{x}}^* $, which represents the optimal solution for the following constrained optimization problem,

\begin{equation}
\min_{\hat{\mathbf{y}}} f_\mathbf{x}(\hat{\mathbf{y}}) + g_\mathbf{x}(\hat{\mathbf{y}}) \quad \text{s.t.} \quad g_\mathbf{x}(\hat{\mathbf{y}}) > 0,    
\end{equation}

we can employ the methodologies delineated in \citet{rahman2021novel}. These techniques provide a mechanism for either upper bounding or lower bounding the CMPE task, contingent on whether it is formulated as a maximization or minimization problem, respectively.

To establish a lower bound for $ q_{\mathbf{x}}^* $, the constrained optimization task is initially transformed into an unconstrained formulation via Lagrange Relaxation. This results in the following optimization problem:

\begin{align}
\label{eq:sup_lagrange-cmpe_min}
    \max_{\mu \geq 0} \min_{\hat{\mathbf{y}}} f_\mathbf{x}(\hat{\mathbf{y}}) + (1-\mu) \times g_\mathbf{x}(\hat{\mathbf{y}})
\end{align}

Here, $ \mu $ denotes the Lagrangian multiplier. By addressing this dual optimization problem, we enhance the precision of the lower bound for $ q_{\mathbf{x}}^* $. For the inner minimization task, the mini-bucket elimination method is employed. The outer maximization is solved through the utilization of sub-gradient descent.

\section{EXTENSIONS}
\label{sec:extensions}
\paragraph{Adding a Penalty for Constraint Violations.}
A penalty of the form $\max\{0,g_\mathbf{x}(\hat{\mathbf{y}})\}^2$ can be easily added to the loss function as described by the following equation
\begin{align*}\label{eqn:loss_1}
		\mathcal{L}_\mathbf{x}(\hat{\mathbf{y}}) & = \begin{cases}
			& f_{\mathbf{x}}(\hat{\mathbf{y}}) \:\:\text{if}\:\: g_\mathbf{x}(\hat{\mathbf{y}}) \leq 0\\
			& \alpha_{\mathbf{x}}(f_{\mathbf{x}}(\hat{\mathbf{y}}) + g_{\mathbf{x}}(\hat{\mathbf{y}})) + \rho  \max\{0,g_\mathbf{x}(\hat{\mathbf{y}})\}^2\:\:\text{if}\:\: g_\mathbf{x}(\hat{\mathbf{y}}) > 0\\
		\end{cases}
	\end{align*}
where $\rho \geq 0 $ is a hyperparameter.
\eat{
\paragraph{Multiple Constraints.}
Let $g^1,\cdots,g^m$ be the multilinear functions involved in the constraints of the form $g^i_\mathbf{x}(\mathbf{y}) \leq 0$ for $i=1,\ldots,m$. Then, we can use the following loss function instead of the original loss function given in Eq. (4) (see the main paper). \begin{align}
\nonumber
\mathcal{L}_\mathbf{x}(\hat{\mathbf{y}}) & = \begin{cases}
			& f_{\mathbf{x}}(\hat{\mathbf{y}}) \:\:\text{if}\:\: \sum_{i=1}^{m}\max\left \{0,g^i_\mathbf{x}(\hat{\mathbf{y}}) \right \} = 0\\
			& \alpha_{\mathbf{x}}\left (f_{\mathbf{x}}(\hat{\mathbf{y}}) +  \sum_{i=1}^{m}\max\left \{0,g^i_\mathbf{x}(\hat{\mathbf{y}}) \right \}\right )\:\:\text{if}\:\: \sum_{i=1}^{m}\max\left \{0,g^i_\mathbf{x}(\hat{\mathbf{y}}) \right \} > 0\\
		\end{cases}
	\end{align}
It is also possible to consider a separate hyperparameter $\alpha_\mathbf{x}^1$ for each constraint
}

%% file: supp_tables/datasets.tex
\begin{table*}[thb]
\centering
\caption{Dataset and Model Descriptions}
\label{tab:supp-dataset-descrip}
\begin{tabular}{|ccc|}
\hline
\multicolumn{1}{|c|}{Dataset}        & \multicolumn{1}{c|}{\textbf{Number of Variables}} & \textbf{Number of Functions} \\ \hline
\multicolumn{3}{|c|}{Tractable Probabilistic Circuits}                                                                  \\ \hline
\multicolumn{1}{|c|}{AD}             & \multicolumn{1}{c|}{1556}                         & 1556                         \\ \hline
\multicolumn{1}{|c|}{BBC}            & \multicolumn{1}{c|}{1058}                         & 1058                         \\ \hline
\multicolumn{1}{|c|}{20NewsGroup}    & \multicolumn{1}{c|}{910}                          & 910                          \\ \hline
\multicolumn{1}{|c|}{WebKB}          & \multicolumn{1}{c|}{839}                          & 839                          \\ \hline
\multicolumn{1}{|c|}{DNA}            & \multicolumn{1}{c|}{180}                          & 180                          \\ \hline
\multicolumn{3}{|c|}{High Tree-Width Markov Networks}                                                                   \\ \hline
\multicolumn{1}{|c|}{Grids17}        & \multicolumn{1}{c|}{400}                          & 1160                         \\ \hline
\multicolumn{1}{|c|}{Grids18}        & \multicolumn{1}{c|}{400}                          & 1160                         \\ \hline
\multicolumn{1}{|c|}{Segmentation12} & \multicolumn{1}{c|}{229}                          & 851                          \\ \hline
\multicolumn{1}{|c|}{Segmentation14} & \multicolumn{1}{c|}{226}                          & 845                          \\ \hline
\multicolumn{1}{|c|}{Segmentation15} & \multicolumn{1}{c|}{232}                          & 863                          \\ \hline
\end{tabular}
\end{table*}

%% file: supp_tables/ours_supervised.tex
\begin{table*}[ht!]
\centering
\setlength{\tabcolsep}{1pt}
\caption{\label{tab:summary_cmpe_vs_sup} Summary: best \sscmpe vs \texttt{other supervised} methods including \mse, \msepen, \mae, and \maepen. Bold represents the minimum gap, while underlined means the least violations}

\begin{tabular}{|cc|cc|cc|cc|cc|}
\hline
\multicolumn{2}{|c|}{q}                                                                                                                      & \multicolumn{2}{c|}{10}                                                                         & \multicolumn{2}{c|}{30}                                                                         & \multicolumn{2}{c|}{60}                                                                         & \multicolumn{2}{c|}{90}                                                                         \\ \hline
\multicolumn{1}{|c|}{\begin{tabular}[c]{@{}c@{}}Models\\ /Dataset\end{tabular}} & \begin{tabular}[c]{@{}c@{}}Gap /\\ Violations\end{tabular} & \multicolumn{1}{c|}{\begin{tabular}[c]{@{}c@{}}best\\ \small{\sscmpe}\end{tabular}}    & \begin{tabular}[c]{@{}c@{}}\small{best}\\ \small{supervised}\end{tabular} & \multicolumn{1}{c|}{\begin{tabular}[c]{@{}c@{}}best\\ \small{\sscmpe}\end{tabular}} & \begin{tabular}[c]{@{}c@{}}\small{best}\\ \small{supervised}\end{tabular} & \multicolumn{1}{c|}{\begin{tabular}[c]{@{}c@{}}best\\ \small{\sscmpe}\end{tabular}} & \begin{tabular}[c]{@{}c@{}}\small{best}\\ \small{supervised}\end{tabular} & \multicolumn{1}{c|}{\begin{tabular}[c]{@{}c@{}}best\\ \small{\sscmpe}\end{tabular}} & \begin{tabular}[c]{@{}c@{}}\small{best}\\ \small{supervised}\end{tabular} \\ \hline
\multicolumn{1}{|c|}{\multirow{2}{*}{Segment-12}}                               & gap                                                        & \multicolumn{1}{c|}{0.057}          & \textbf{0.054}                                            & \multicolumn{1}{c|}{\textbf{0.051}} & 0.052                                                     & \multicolumn{1}{c|}{0.053}          & \textbf{0.051}                                            & \multicolumn{1}{c|}{\textbf{0.050}} & 0.051                                                     \\
\multicolumn{1}{|c|}{}                                                          & violations                                                 & \multicolumn{1}{c|}{{\ul 0.152}}    & 0.511                                                     & \multicolumn{1}{c|}{{\ul 0.166}}    & 0.500                                                     & \multicolumn{1}{c|}{{\ul 0.084}}    & 0.348                                                     & \multicolumn{1}{c|}{{\ul 0.021}}    & 0.131                                                     \\ \hline
\multicolumn{1}{|c|}{\multirow{2}{*}{Segment-14}}                               & gap                                                        & \multicolumn{1}{c|}{0.050}          & \textbf{0.049}                                            & \multicolumn{1}{c|}{\textbf{0.047}} & 0.049                                                     & \multicolumn{1}{c|}{0.048}          & 0.048                                                     & \multicolumn{1}{c|}{0.051}          & 0.051                                                     \\
\multicolumn{1}{|c|}{}                                                          & violations                                                 & \multicolumn{1}{c|}{{\ul 0.134}}    & 0.691                                                     & \multicolumn{1}{c|}{{\ul 0.088}}    & 0.616                                                     & \multicolumn{1}{c|}{{\ul 0.066}}    & 0.410                                                     & \multicolumn{1}{c|}{{\ul 0.046}}    & 0.207                                                     \\ \hline
\multicolumn{1}{|c|}{\multirow{2}{*}{Segment-15}}                               & gap                                                        & \multicolumn{1}{c|}{0.051}          & 0.051                                                     & \multicolumn{1}{c|}{\textbf{0.050}} & 0.051                                                     & \multicolumn{1}{c|}{0.052}          & \textbf{0.051}                                            & \multicolumn{1}{c|}{\textbf{0.051}} & 0.052                                                     \\
\multicolumn{1}{|c|}{}                                                          & violations                                                 & \multicolumn{1}{c|}{{\ul 0.060}}    & 0.570                                                     & \multicolumn{1}{c|}{{\ul 0.060}}    & 0.417                                                     & \multicolumn{1}{c|}{{\ul 0.049}}    & 0.248                                                     & \multicolumn{1}{c|}{{\ul 0.001}}    & 0.061                                                     \\ \hline
\multicolumn{1}{|c|}{\multirow{2}{*}{Grids-17}}                                 & gap                                                        & \multicolumn{1}{c|}{0.072}          & \textbf{0.054}                                            & \multicolumn{1}{c|}{0.069}          & \textbf{0.057}                                            & \multicolumn{1}{c|}{\textbf{0.066}} & 0.067                                                     & \multicolumn{1}{c|}{0.063}          & \textbf{0.058}                                            \\
\multicolumn{1}{|c|}{}                                                          & violations                                                 & \multicolumn{1}{c|}{{\ul 0.035}}    & 0.304                                                     & \multicolumn{1}{c|}{{\ul 0.013}}    & 0.125                                                     & \multicolumn{1}{c|}{{\ul 0.002}}    & 0.044                                                     & \multicolumn{1}{c|}{{\ul 0.001}}    & 0.002                                                     \\ \hline
\multicolumn{1}{|c|}{\multirow{2}{*}{Grids-18}}                                 & gap                                                        & \multicolumn{1}{c|}{0.064}          & \textbf{0.056}                                            & \multicolumn{1}{c|}{0.067}          & \textbf{0.060}                                            & \multicolumn{1}{c|}{\textbf{0.060}} & 0.065                                                     & \multicolumn{1}{c|}{0.065}          & \textbf{0.064}                                            \\
\multicolumn{1}{|c|}{}                                                          & violations                                                 & \multicolumn{1}{c|}{{\ul 0.017}}    & 0.210                                                     & \multicolumn{1}{c|}{{\ul 0.019}}    & 0.087                                                     & \multicolumn{1}{c|}{{\ul 0.000}}    & 0.025                                                     & \multicolumn{1}{c|}{{\ul 0.000}}    & 0.015                                                     \\ \hline
\multicolumn{1}{|c|}{\multirow{2}{*}{DNA}}                                      & gap                                                        & \multicolumn{1}{c|}{0.138}          & \textbf{0.135}                                            & \multicolumn{1}{c|}{0.138}          & \textbf{0.136}                                            & \multicolumn{1}{c|}{0.137}          & \textbf{0.136}                                            & \multicolumn{1}{c|}{0.139}          & \textbf{0.137}                                            \\
\multicolumn{1}{|c|}{}                                                          & violations                                                 & \multicolumn{1}{c|}{{\ul 0.013}}    & 0.434                                                     & \multicolumn{1}{c|}{{\ul 0.002}}    & 0.448                                                     & \multicolumn{1}{c|}{{\ul 0.001}}    & 0.281                                                     & \multicolumn{1}{c|}{{\ul 0.001}}    & 0.089                                                     \\ \hline
\multicolumn{1}{|c|}{\multirow{2}{*}{20NewsGr.}}                                & gap                                                        & \multicolumn{1}{c|}{\textbf{0.043}} & 0.044                                                     & \multicolumn{1}{c|}{\textbf{0.045}} & 0.046                                                     & \multicolumn{1}{c|}{\textbf{0.044}} & 0.046                                                     & \multicolumn{1}{c|}{0.044}          & 0.044                                                     \\
\multicolumn{1}{|c|}{}                                                          & violations                                                 & \multicolumn{1}{c|}{{\ul 0.069}}    & 0.455                                                     & \multicolumn{1}{c|}{{\ul 0.054}}    & 0.176                                                     & \multicolumn{1}{c|}{{\ul 0.007}}    & 0.046                                                     & \multicolumn{1}{c|}{0.001}          & 0.001                                                     \\ \hline
\multicolumn{1}{|c|}{\multirow{2}{*}{WebKB}}                                    & gap                                                        & \multicolumn{1}{c|}{0.059}          & \textbf{0.054}                                            & \multicolumn{1}{c|}{0.058}          & \textbf{0.054}                                            & \multicolumn{1}{c|}{0.056}          & \textbf{0.054}                                            & \multicolumn{1}{c|}{\textbf{0.053}} & 0.054                                                     \\
\multicolumn{1}{|c|}{}                                                          & violations                                                 & \multicolumn{1}{c|}{{\ul 0.074}}    & 0.471                                                     & \multicolumn{1}{c|}{{\ul 0.029}}    & 0.378                                                     & \multicolumn{1}{c|}{{\ul 0.001}}    & 0.174                                                     & \multicolumn{1}{c|}{{\ul 0.001}}    & 0.018                                                     \\ \hline
\multicolumn{1}{|c|}{\multirow{2}{*}{BBC}}                                      & gap                                                        & \multicolumn{1}{c|}{0.038}          & \textbf{0.036}                                            & \multicolumn{1}{c|}{0.043}          & \textbf{0.036}                                            & \multicolumn{1}{c|}{0.042}          & \textbf{0.037}                                            & \multicolumn{1}{c|}{0.040}          & \textbf{0.037}                                            \\
\multicolumn{1}{|c|}{}                                                          & violations                                                 & \multicolumn{1}{c|}{{\ul 0.074}}    & 0.657                                                     & \multicolumn{1}{c|}{{\ul 0.067}}    & 0.557                                                     & \multicolumn{1}{c|}{{\ul 0.056}}    & 0.384                                                     & \multicolumn{1}{c|}{{\ul 0.002}}    & 0.151                                                     \\ \hline
\multicolumn{1}{|c|}{\multirow{2}{*}{Ad}}                                       & gap                                                        & \multicolumn{1}{c|}{\textbf{0.129}} & 0.204                                                     & \multicolumn{1}{c|}{\textbf{0.131}} & 0.201                                                     & \multicolumn{1}{c|}{\textbf{0.130}} & 0.204                                                     & \multicolumn{1}{c|}{\textbf{0.131}} & 0.213                                                     \\
\multicolumn{1}{|c|}{}                                                          & violations                                                 & \multicolumn{1}{c|}{{\ul 0.017}}    & 0.085                                                     & \multicolumn{1}{c|}{{\ul 0.004}}    & 0.041                                                     & \multicolumn{1}{c|}{{\ul 0.000}}    & 0.021                                                     & \multicolumn{1}{c|}{{\ul 0.000}}    & 0.005                                                     \\ \hline
\multicolumn{1}{|c|}{\multirow{2}{*}{Average}}                                  & gap                                                        & \multicolumn{1}{c|}{\textbf{0.070}} & 0.074                                                     & \multicolumn{1}{c|}{\textbf{0.070}} & 0.074                                                     & \multicolumn{1}{c|}{\textbf{0.069}} & 0.076                                                     & \multicolumn{1}{c|}{\textbf{0.069}} & 0.076                                                     \\
\multicolumn{1}{|c|}{}                                                          & violations                                                 & \multicolumn{1}{c|}{{\ul 0.065}}    & 0.439                                                     & \multicolumn{1}{c|}{{\ul 0.050}}    & 0.335                                                     & \multicolumn{1}{c|}{{\ul 0.027}}    & 0.198                                                     & \multicolumn{1}{c|}{{\ul 0.007}}    & 0.068                                                     \\ \hline
\end{tabular}
\end{table*}
\eat{
\begin{table}[tbh]
\centering
\setlength{\tabcolsep}{1pt}
\caption{\label{tab:summary_cmpe_vs_sup} Summary: \sscmpe vs \texttt{other supervised} methods including \mse, \msepen, \mae, and \maepen.}
\begin{tabular}{|cc|cc|cc|cc|cc|}
\hline
\multicolumn{2}{|c|}{q}                                                                                                                      & \multicolumn{2}{c|}{10}                              & \multicolumn{2}{c|}{30}                              & \multicolumn{2}{c|}{60}                              & \multicolumn{2}{c|}{90}                              \\ \hline
\multicolumn{1}{|c|}{\begin{tabular}[c]{@{}c@{}}Models\\ /Dataset\end{tabular}} & \begin{tabular}[c]{@{}c@{}}Gap /\\ Violations\end{tabular} & \multicolumn{1}{c|}{\sscmpe}        & others         & \multicolumn{1}{c|}{\sscmpe}        & others         & \multicolumn{1}{c|}{\sscmpe}        & others         & \multicolumn{1}{c|}{\sscmpe}        & others         \\ \hline
\multicolumn{1}{|c|}{\multirow{2}{*}{Segment-12}}                               & gap                                                        & \multicolumn{1}{c|}{0.057}          & \textbf{0.054} & \multicolumn{1}{c|}{\textbf{0.051}} & 0.052          & \multicolumn{1}{c|}{0.053}          & \textbf{0.051} & \multicolumn{1}{c|}{\textbf{0.050}} & 0.051          \\
\multicolumn{1}{|c|}{}                                                          & violations                                                 & \multicolumn{1}{c|}{{\ul 0.152}}    & 0.511          & \multicolumn{1}{c|}{{\ul 0.166}}    & 0.500          & \multicolumn{1}{c|}{{\ul 0.084}}    & 0.348          & \multicolumn{1}{c|}{{\ul 0.021}}    & 0.131          \\ \hline
\multicolumn{1}{|c|}{\multirow{2}{*}{Segment-14}}                               & gap                                                        & \multicolumn{1}{c|}{0.050}          & \textbf{0.049} & \multicolumn{1}{c|}{\textbf{0.047}} & 0.049          & \multicolumn{1}{c|}{0.048}          & 0.048          & \multicolumn{1}{c|}{0.051}          & 0.051          \\
\multicolumn{1}{|c|}{}                                                          & violations                                                 & \multicolumn{1}{c|}{{\ul 0.134}}    & 0.691          & \multicolumn{1}{c|}{{\ul 0.088}}    & 0.616          & \multicolumn{1}{c|}{{\ul 0.066}}    & 0.410          & \multicolumn{1}{c|}{{\ul 0.046}}    & 0.207          \\ \hline
\multicolumn{1}{|c|}{\multirow{2}{*}{Segment-15}}                               & gap                                                        & \multicolumn{1}{c|}{0.051}          & 0.051          & \multicolumn{1}{c|}{\textbf{0.050}} & 0.051          & \multicolumn{1}{c|}{0.052}          & \textbf{0.051} & \multicolumn{1}{c|}{\textbf{0.051}} & 0.052          \\
\multicolumn{1}{|c|}{}                                                          & violations                                                 & \multicolumn{1}{c|}{{\ul 0.060}}    & 0.570          & \multicolumn{1}{c|}{{\ul 0.060}}    & 0.417          & \multicolumn{1}{c|}{{\ul 0.049}}    & 0.248          & \multicolumn{1}{c|}{{\ul 0.001}}    & 0.061          \\ \hline
\multicolumn{1}{|c|}{\multirow{2}{*}{Grids-17}}                                 & gap                                                        & \multicolumn{1}{c|}{0.072}          & \textbf{0.054} & \multicolumn{1}{c|}{0.069}          & \textbf{0.057} & \multicolumn{1}{c|}{\textbf{0.066}} & 0.067          & \multicolumn{1}{c|}{0.063}          & \textbf{0.058} \\
\multicolumn{1}{|c|}{}                                                          & violations                                                 & \multicolumn{1}{c|}{{\ul 0.035}}    & 0.304          & \multicolumn{1}{c|}{{\ul 0.013}}    & 0.125          & \multicolumn{1}{c|}{{\ul 0.002}}    & 0.044          & \multicolumn{1}{c|}{{\ul 0.001}}    & 0.002          \\ \hline
\multicolumn{1}{|c|}{\multirow{2}{*}{Grids-18}}                                 & gap                                                        & \multicolumn{1}{c|}{0.064}          & \textbf{0.056} & \multicolumn{1}{c|}{0.067}          & \textbf{0.060} & \multicolumn{1}{c|}{\textbf{0.060}} & 0.065          & \multicolumn{1}{c|}{0.065}          & \textbf{0.064} \\
\multicolumn{1}{|c|}{}                                                          & violations                                                 & \multicolumn{1}{c|}{{\ul 0.017}}    & 0.210          & \multicolumn{1}{c|}{{\ul 0.019}}    & 0.087          & \multicolumn{1}{c|}{{\ul 0.000}}    & 0.025          & \multicolumn{1}{c|}{{\ul 0.000}}    & 0.015          \\ \hline
\multicolumn{1}{|c|}{\multirow{2}{*}{DNA}}                                      & gap                                                        & \multicolumn{1}{c|}{0.138}          & \textbf{0.135} & \multicolumn{1}{c|}{0.138}          & \textbf{0.136} & \multicolumn{1}{c|}{0.137}          & \textbf{0.136} & \multicolumn{1}{c|}{0.139}          & \textbf{0.137} \\
\multicolumn{1}{|c|}{}                                                          & violations                                                 & \multicolumn{1}{c|}{{\ul 0.013}}    & 0.434          & \multicolumn{1}{c|}{{\ul 0.002}}    & 0.448          & \multicolumn{1}{c|}{{\ul 0.001}}    & 0.281          & \multicolumn{1}{c|}{{\ul 0.001}}    & 0.089          \\ \hline
\multicolumn{1}{|c|}{\multirow{2}{*}{20NewsGr.}}                                & gap                                                        & \multicolumn{1}{c|}{\textbf{0.043}} & 0.044          & \multicolumn{1}{c|}{\textbf{0.045}} & 0.046          & \multicolumn{1}{c|}{\textbf{0.044}} & 0.046          & \multicolumn{1}{c|}{0.044}          & 0.044          \\
\multicolumn{1}{|c|}{}                                                          & violations                                                 & \multicolumn{1}{c|}{{\ul 0.069}}    & 0.455          & \multicolumn{1}{c|}{{\ul 0.054}}    & 0.176          & \multicolumn{1}{c|}{{\ul 0.007}}    & 0.046          & \multicolumn{1}{c|}{0.001}          & 0.001          \\ \hline
\multicolumn{1}{|c|}{\multirow{2}{*}{WebKB}}                                    & gap                                                        & \multicolumn{1}{c|}{0.059}          & \textbf{0.054} & \multicolumn{1}{c|}{0.058}          & \textbf{0.054} & \multicolumn{1}{c|}{0.056}          & \textbf{0.054} & \multicolumn{1}{c|}{\textbf{0.053}} & 0.054          \\
\multicolumn{1}{|c|}{}                                                          & violations                                                 & \multicolumn{1}{c|}{{\ul 0.074}}    & 0.471          & \multicolumn{1}{c|}{{\ul 0.029}}    & 0.378          & \multicolumn{1}{c|}{{\ul 0.001}}    & 0.174          & \multicolumn{1}{c|}{{\ul 0.001}}    & 0.018          \\ \hline
\multicolumn{1}{|c|}{\multirow{2}{*}{BBC}}                                      & gap                                                        & \multicolumn{1}{c|}{0.038}          & \textbf{0.036} & \multicolumn{1}{c|}{0.043}          & \textbf{0.036} & \multicolumn{1}{c|}{0.042}          & \textbf{0.037} & \multicolumn{1}{c|}{0.040}          & \textbf{0.037} \\
\multicolumn{1}{|c|}{}                                                          & violations                                                 & \multicolumn{1}{c|}{{\ul 0.074}}    & 0.657          & \multicolumn{1}{c|}{{\ul 0.067}}    & 0.557          & \multicolumn{1}{c|}{{\ul 0.056}}    & 0.384          & \multicolumn{1}{c|}{{\ul 0.002}}    & 0.151          \\ \hline
\multicolumn{1}{|c|}{\multirow{2}{*}{Ad}}                                       & gap                                                        & \multicolumn{1}{c|}{\textbf{0.129}} & 0.204          & \multicolumn{1}{c|}{\textbf{0.131}} & 0.201          & \multicolumn{1}{c|}{\textbf{0.130}} & 0.204          & \multicolumn{1}{c|}{\textbf{0.131}} & 0.213          \\
\multicolumn{1}{|c|}{}                                                          & violations                                                 & \multicolumn{1}{c|}{{\ul 0.017}}    & 0.085          & \multicolumn{1}{c|}{{\ul 0.004}}    & 0.041          & \multicolumn{1}{c|}{{\ul 0.000}}    & 0.021          & \multicolumn{1}{c|}{{\ul 0.000}}    & 0.005          \\ \hline
\multicolumn{1}{|c|}{\multirow{2}{*}{Average}}                                  & gap                                                        & \multicolumn{1}{c|}{\textbf{0.070}} & 0.074          & \multicolumn{1}{c|}{\textbf{0.070}} & 0.074          & \multicolumn{1}{c|}{\textbf{0.069}} & 0.076          & \multicolumn{1}{c|}{\textbf{0.069}} & 0.076          \\
\multicolumn{1}{|c|}{}                                                          & violations                                                 & \multicolumn{1}{c|}{{\ul 0.065}}    & 0.439          & \multicolumn{1}{c|}{{\ul 0.050}}    & 0.335          & \multicolumn{1}{c|}{{\ul 0.027}}    & 0.198          & \multicolumn{1}{c|}{{\ul 0.007}}    & 0.068          \\ \hline
\end{tabular}
\end{table}
}

%% file: supp_tables/our_self_supervised.tex
\begin{table*}[ht]
\centering
\setlength{\tabcolsep}{1pt}
\caption{\label{tab:summary_cmpe_vs_ssl} Summary: best \sscmpe vs \texttt{other self-supervised} methods including \slpen, and \pdl. Bold represents the minimum gap, while underlined means the least violations}

\begin{tabular}{|cc|cc|cc|cc|cc|}
\hline
\multicolumn{2}{|c|}{q}                                                                                                                     & \multicolumn{2}{c|}{10}                                                                  & \multicolumn{2}{c|}{30}                                                                  & \multicolumn{2}{c|}{60}                                                                  & \multicolumn{2}{c|}{90}                                                                  \\ \hline
\multicolumn{1}{|c|}{\begin{tabular}[c]{@{}c@{}}Models/\\ Dataset\end{tabular}} & \begin{tabular}[c]{@{}c@{}}Gap/\\ Violations\end{tabular} & \multicolumn{1}{c|}{\begin{tabular}[c]{@{}c@{}}best\\ \small{\sscmpe}\end{tabular}}        & \begin{tabular}[c]{@{}c@{}}best\\ SSL\end{tabular} & \multicolumn{1}{c|}{\begin{tabular}[c]{@{}c@{}}best\\ \small{\sscmpe}\end{tabular}}     & \begin{tabular}[c]{@{}c@{}}best\\ SSL\end{tabular} & \multicolumn{1}{c|}{\begin{tabular}[c]{@{}c@{}}best\\ \small{\sscmpe}\end{tabular}}   & \begin{tabular}[c]{@{}c@{}}best\\ SSL\end{tabular} & \multicolumn{1}{c|}{\begin{tabular}[c]{@{}c@{}}best\\ \small{\sscmpe}\end{tabular}}  & \begin{tabular}[c]{@{}c@{}}best\\ SSL\end{tabular} \\ \hline
\multicolumn{1}{|c|}{\multirow{2}{*}{Segment-12}}                               & gap                                                       & \multicolumn{1}{c|}{0.057}          & \textbf{0.054}                                     & \multicolumn{1}{c|}{\textbf{0.051}} & 0.052                                              & \multicolumn{1}{c|}{0.053}          & \textbf{0.051}                                     & \multicolumn{1}{c|}{\textbf{0.050}} & 0.051                                              \\
\multicolumn{1}{|c|}{}                                                          & violations                                                & \multicolumn{1}{c|}{{\ul 0.152}}    & 0.545                                              & \multicolumn{1}{c|}{{\ul 0.166}}    & 0.622                                              & \multicolumn{1}{c|}{{\ul 0.084}}    & 0.486                                              & \multicolumn{1}{c|}{{\ul 0.021}}    & 0.163                                              \\ \hline
\multicolumn{1}{|c|}{\multirow{2}{*}{Segment-14}}                               & gap                                                       & \multicolumn{1}{c|}{\textbf{0.050}} & 0.058                                              & \multicolumn{1}{c|}{\textbf{0.047}} & 0.050                                              & \multicolumn{1}{c|}{\textbf{0.048}} & 0.050                                              & \multicolumn{1}{c|}{\textbf{0.051}} & 0.053                                              \\
\multicolumn{1}{|c|}{}                                                          & violations                                                & \multicolumn{1}{c|}{{\ul 0.134}}    & 0.507                                              & \multicolumn{1}{c|}{{\ul 0.088}}    & 0.414                                              & \multicolumn{1}{c|}{{\ul 0.066}}    & 0.394                                              & \multicolumn{1}{c|}{{\ul 0.046}}    & 0.207                                              \\ \hline
\multicolumn{1}{|c|}{\multirow{2}{*}{Segment-15}}                               & gap                                                       & \multicolumn{1}{c|}{0.051}          & 0.051                                              & \multicolumn{1}{c|}{\textbf{0.050}} & 0.053                                              & \multicolumn{1}{c|}{0.052}          & \textbf{0.051}                                     & \multicolumn{1}{c|}{\textbf{0.051}} & 0.054                                              \\
\multicolumn{1}{|c|}{}                                                          & violations                                                & \multicolumn{1}{c|}{{\ul 0.060}}    & 0.676                                              & \multicolumn{1}{c|}{{\ul 0.060}}    & 0.360                                              & \multicolumn{1}{c|}{{\ul 0.049}}    & 0.274                                              & \multicolumn{1}{c|}{{\ul 0.001}}    & 0.086                                              \\ \hline
\multicolumn{1}{|c|}{\multirow{2}{*}{Grids-17}}                                 & gap                                                       & \multicolumn{1}{c|}{\textbf{0.072}} & 0.086                                              & \multicolumn{1}{c|}{\textbf{0.069}} & 0.079                                              & \multicolumn{1}{c|}{\textbf{0.066}} & 0.093                                              & \multicolumn{1}{c|}{\textbf{0.063}} & 0.087                                              \\
\multicolumn{1}{|c|}{}                                                          & violations                                                & \multicolumn{1}{c|}{{\ul 0.035}}    & 0.043                                              & \multicolumn{1}{c|}{{\ul 0.013}}    & 0.003                                              & \multicolumn{1}{c|}{0.002}          & {\ul 0.001}                                        & \multicolumn{1}{c|}{{\ul 0.001}}    & 0.014                                              \\ \hline
\multicolumn{1}{|c|}{\multirow{2}{*}{Grids-18}}                                 & gap                                                       & \multicolumn{1}{c|}{\textbf{0.064}} & 0.105                                              & \multicolumn{1}{c|}{\textbf{0.067}} & 0.074                                              & \multicolumn{1}{c|}{\textbf{0.060}} & 0.093                                              & \multicolumn{1}{c|}{\textbf{0.065}} & 0.118                                              \\
\multicolumn{1}{|c|}{}                                                          & violations                                                & \multicolumn{1}{c|}{{\ul 0.017}}    & 0.060                                              & \multicolumn{1}{c|}{0.019}          & {\ul 0.001}                                        & \multicolumn{1}{c|}{{\ul 0.000}}    & 0.003                                              & \multicolumn{1}{c|}{{\ul 0.000}}    & 0.005                                              \\ \hline
\multicolumn{1}{|c|}{\multirow{2}{*}{DNA}}                                      & gap                                                       & \multicolumn{1}{c|}{\textbf{0.138}} & 0.140                                              & \multicolumn{1}{c|}{\textbf{0.138}} & 0.141                                              & \multicolumn{1}{c|}{\textbf{0.137}} & 0.139                                              & \multicolumn{1}{c|}{\textbf{0.139}} & 0.143                                              \\
\multicolumn{1}{|c|}{}                                                          & violations                                                & \multicolumn{1}{c|}{{\ul 0.013}}    & 0.048                                              & \multicolumn{1}{c|}{{\ul 0.002}}    & 0.062                                              & \multicolumn{1}{c|}{{\ul 0.001}}    & 0.003                                              & \multicolumn{1}{c|}{{\ul 0.001}}    & 0.004                                              \\ \hline
\multicolumn{1}{|c|}{\multirow{2}{*}{20NewsGr}}                                 & gap                                                       & \multicolumn{1}{c|}{0.043}          & 0.043                                              & \multicolumn{1}{c|}{\textbf{0.045}} & 0.046                                              & \multicolumn{1}{c|}{\textbf{0.044}} & 0.045                                              & \multicolumn{1}{c|}{\textbf{0.044}} & 0.046                                              \\
\multicolumn{1}{|c|}{}                                                          & violations                                                & \multicolumn{1}{c|}{{\ul 0.069}}    & 0.278                                              & \multicolumn{1}{c|}{{\ul 0.054}}    & 0.129                                              & \multicolumn{1}{c|}{{\ul 0.007}}    & 0.024                                              & \multicolumn{1}{c|}{0.001}          & 0.001                                              \\ \hline
\multicolumn{1}{|c|}{\multirow{2}{*}{WebKB}}                                    & gap                                                       & \multicolumn{1}{c|}{0.059}          & \textbf{0.058}                                     & \multicolumn{1}{c|}{0.058}          & \textbf{0.057}                                     & \multicolumn{1}{c|}{\textbf{0.056}} & 0.057                                              & \multicolumn{1}{c|}{\textbf{0.053}} & 0.057                                              \\
\multicolumn{1}{|c|}{}                                                          & violations                                                & \multicolumn{1}{c|}{{\ul 0.074}}    & 0.149                                              & \multicolumn{1}{c|}{{\ul 0.029}}    & 0.096                                              & \multicolumn{1}{c|}{{\ul 0.001}}    & 0.056                                              & \multicolumn{1}{c|}{{\ul 0.001}}    & 0.013                                              \\ \hline
\multicolumn{1}{|c|}{\multirow{2}{*}{BBC}}                                      & gap                                                       & \multicolumn{1}{c|}{\textbf{0.038}} & 0.041                                              & \multicolumn{1}{c|}{\textbf{0.043}} & 0.042                                              & \multicolumn{1}{c|}{0.042}          & \textbf{0.038}                                     & \multicolumn{1}{c|}{0.040}          & \textbf{0.039}                                     \\
\multicolumn{1}{|c|}{}                                                          & violations                                                & \multicolumn{1}{c|}{{\ul 0.074}}    & 0.336                                              & \multicolumn{1}{c|}{{\ul 0.067}}    & 0.160                                              & \multicolumn{1}{c|}{{\ul 0.056}}    & 0.149                                              & \multicolumn{1}{c|}{{\ul 0.002}}    & 0.029                                              \\ \hline
\multicolumn{1}{|c|}{\multirow{2}{*}{Ad}}                                       & gap                                                       & \multicolumn{1}{c|}{\textbf{0.129}} & 0.135                                              & \multicolumn{1}{c|}{\textbf{0.131}} & 0.140                                              & \multicolumn{1}{c|}{\textbf{0.130}} & 0.142                                              & \multicolumn{1}{c|}{\textbf{0.131}} & 0.134                                              \\
\multicolumn{1}{|c|}{}                                                          & violations                                                & \multicolumn{1}{c|}{{\ul 0.017}}    & 0.055                                              & \multicolumn{1}{c|}{{\ul 0.004}}    & 0.006                                              & \multicolumn{1}{c|}{{\ul 0.000}}    & 0.013                                              & \multicolumn{1}{c|}{{\ul 0.000}}    & 0.004                                              \\ \hline
\multicolumn{1}{|c|}{\multirow{2}{*}{Average}}                                  & gap                                                       & \multicolumn{1}{c|}{\textbf{0.070}} & 0.077                                              & \multicolumn{1}{c|}{\textbf{0.070}} & 0.073                                              & \multicolumn{1}{c|}{\textbf{0.069}} & 0.076                                              & \multicolumn{1}{c|}{\textbf{0.069}} & 0.078                                              \\
\multicolumn{1}{|c|}{}                                                          & violations                                                & \multicolumn{1}{c|}{{\ul 0.065}}    & 0.270                                              & \multicolumn{1}{c|}{{\ul 0.050}}    & 0.185                                              & \multicolumn{1}{c|}{{\ul 0.027}}    & 0.140                                              & \multicolumn{1}{c|}{{\ul 0.007}}    & 0.053                                              \\ \hline
\end{tabular}
\end{table*}
\eat{
\begin{table}[]
\centering
\setlength{\tabcolsep}{1pt}
\caption{\label{tab:summary_cmpe_vs_ssl} Summary: \sscmpe vs \texttt{other self-supervised} methods including \slpen, and \pdl. Bold represents the minimum gap, while underlined means the least violations}
\vspace{0.35cm}
\begin{tabular}{|cc|cc|cc|cc|cc|}
\hline
\multicolumn{2}{|c|}{q}                                                                                                                     & \multicolumn{2}{c|}{10}                              & \multicolumn{2}{c|}{30}                              & \multicolumn{2}{c|}{60}                              & \multicolumn{2}{c|}{90}                              \\ \hline
\multicolumn{1}{|c|}{\begin{tabular}[c]{@{}c@{}}Models/\\ Dataset\end{tabular}} & \begin{tabular}[c]{@{}c@{}}Gap/\\ Violations\end{tabular} & \multicolumn{1}{c|}{\sscmpe}        & others         & \multicolumn{1}{c|}{\sscmpe}        & others         & \multicolumn{1}{c|}{\sscmpe}        & others         & \multicolumn{1}{c|}{\sscmpe}        & others         \\ \hline
\multicolumn{1}{|c|}{\multirow{2}{*}{Segment-12}}                               & gap                                                       & \multicolumn{1}{c|}{0.057}          & \textbf{0.054} & \multicolumn{1}{c|}{\textbf{0.051}} & 0.052          & \multicolumn{1}{c|}{0.053}          & \textbf{0.051} & \multicolumn{1}{c|}{\textbf{0.050}} & 0.051          \\
\multicolumn{1}{|c|}{}                                                          & violations                                                & \multicolumn{1}{c|}{{\ul 0.152}}    & 0.545          & \multicolumn{1}{c|}{{\ul 0.166}}    & 0.622          & \multicolumn{1}{c|}{{\ul 0.084}}    & 0.486          & \multicolumn{1}{c|}{{\ul 0.021}}    & 0.163          \\ \hline
\multicolumn{1}{|c|}{\multirow{2}{*}{Segment-14}}                               & gap                                                       & \multicolumn{1}{c|}{\textbf{0.050}} & 0.058          & \multicolumn{1}{c|}{\textbf{0.047}} & 0.050          & \multicolumn{1}{c|}{\textbf{0.048}} & 0.050          & \multicolumn{1}{c|}{\textbf{0.051}} & 0.053          \\
\multicolumn{1}{|c|}{}                                                          & violations                                                & \multicolumn{1}{c|}{{\ul 0.134}}    & 0.507          & \multicolumn{1}{c|}{{\ul 0.088}}    & 0.414          & \multicolumn{1}{c|}{{\ul 0.066}}    & 0.394          & \multicolumn{1}{c|}{{\ul 0.046}}    & 0.207          \\ \hline
\multicolumn{1}{|c|}{\multirow{2}{*}{Segment-15}}                               & gap                                                       & \multicolumn{1}{c|}{0.051}          & 0.051          & \multicolumn{1}{c|}{\textbf{0.050}} & 0.053          & \multicolumn{1}{c|}{0.052}          & \textbf{0.051} & \multicolumn{1}{c|}{\textbf{0.051}} & 0.054          \\
\multicolumn{1}{|c|}{}                                                          & violations                                                & \multicolumn{1}{c|}{{\ul 0.060}}    & 0.676          & \multicolumn{1}{c|}{{\ul 0.060}}    & 0.360          & \multicolumn{1}{c|}{{\ul 0.049}}    & 0.274          & \multicolumn{1}{c|}{{\ul 0.001}}    & 0.086          \\ \hline
\multicolumn{1}{|c|}{\multirow{2}{*}{Grids-17}}                                 & gap                                                       & \multicolumn{1}{c|}{\textbf{0.072}} & 0.086          & \multicolumn{1}{c|}{\textbf{0.069}} & 0.079          & \multicolumn{1}{c|}{\textbf{0.066}} & 0.093          & \multicolumn{1}{c|}{\textbf{0.063}} & 0.087          \\
\multicolumn{1}{|c|}{}                                                          & violations                                                & \multicolumn{1}{c|}{{\ul 0.035}}    & 0.043          & \multicolumn{1}{c|}{{\ul 0.013}}    & 0.003          & \multicolumn{1}{c|}{0.002}          & {\ul 0.001}    & \multicolumn{1}{c|}{{\ul 0.001}}    & 0.014          \\ \hline
\multicolumn{1}{|c|}{\multirow{2}{*}{Grids-18}}                                 & gap                                                       & \multicolumn{1}{c|}{\textbf{0.064}} & 0.105          & \multicolumn{1}{c|}{\textbf{0.067}} & 0.074          & \multicolumn{1}{c|}{\textbf{0.060}} & 0.093          & \multicolumn{1}{c|}{\textbf{0.065}} & 0.118          \\
\multicolumn{1}{|c|}{}                                                          & violations                                                & \multicolumn{1}{c|}{{\ul 0.017}}    & 0.060          & \multicolumn{1}{c|}{0.019}          & {\ul 0.001}    & \multicolumn{1}{c|}{{\ul 0.000}}    & 0.003          & \multicolumn{1}{c|}{{\ul 0.000}}    & 0.005          \\ \hline
\multicolumn{1}{|c|}{\multirow{2}{*}{DNA}}                                      & gap                                                       & \multicolumn{1}{c|}{\textbf{0.138}} & 0.140          & \multicolumn{1}{c|}{\textbf{0.138}} & 0.141          & \multicolumn{1}{c|}{\textbf{0.137}} & 0.139          & \multicolumn{1}{c|}{\textbf{0.139}} & 0.143          \\
\multicolumn{1}{|c|}{}                                                          & violations                                                & \multicolumn{1}{c|}{{\ul 0.013}}    & 0.048          & \multicolumn{1}{c|}{{\ul 0.002}}    & 0.062          & \multicolumn{1}{c|}{{\ul 0.001}}    & 0.003          & \multicolumn{1}{c|}{{\ul 0.001}}    & 0.004          \\ \hline
\multicolumn{1}{|c|}{\multirow{2}{*}{20NewsGr}}                                 & gap                                                       & \multicolumn{1}{c|}{0.043}          & 0.043          & \multicolumn{1}{c|}{\textbf{0.045}} & 0.046          & \multicolumn{1}{c|}{\textbf{0.044}} & 0.045          & \multicolumn{1}{c|}{\textbf{0.044}} & 0.046          \\
\multicolumn{1}{|c|}{}                                                          & violations                                                & \multicolumn{1}{c|}{{\ul 0.069}}    & 0.278          & \multicolumn{1}{c|}{{\ul 0.054}}    & 0.129          & \multicolumn{1}{c|}{{\ul 0.007}}    & 0.024          & \multicolumn{1}{c|}{0.001}          & 0.001          \\ \hline
\multicolumn{1}{|c|}{\multirow{2}{*}{WebKB}}                                    & gap                                                       & \multicolumn{1}{c|}{0.059}          & \textbf{0.058} & \multicolumn{1}{c|}{0.058}          & \textbf{0.057} & \multicolumn{1}{c|}{\textbf{0.056}} & 0.057          & \multicolumn{1}{c|}{\textbf{0.053}} & 0.057          \\
\multicolumn{1}{|c|}{}                                                          & violations                                                & \multicolumn{1}{c|}{{\ul 0.074}}    & 0.149          & \multicolumn{1}{c|}{{\ul 0.029}}    & 0.096          & \multicolumn{1}{c|}{{\ul 0.001}}    & 0.056          & \multicolumn{1}{c|}{{\ul 0.001}}    & 0.013          \\ \hline
\multicolumn{1}{|c|}{\multirow{2}{*}{BBC}}                                      & gap                                                       & \multicolumn{1}{c|}{\textbf{0.038}} & 0.041          & \multicolumn{1}{c|}{\textbf{0.043}} & 0.042          & \multicolumn{1}{c|}{0.042}          & \textbf{0.038} & \multicolumn{1}{c|}{0.040}          & \textbf{0.039} \\
\multicolumn{1}{|c|}{}                                                          & violations                                                & \multicolumn{1}{c|}{{\ul 0.074}}    & 0.336          & \multicolumn{1}{c|}{{\ul 0.067}}    & 0.160          & \multicolumn{1}{c|}{{\ul 0.056}}    & 0.149          & \multicolumn{1}{c|}{{\ul 0.002}}    & 0.029          \\ \hline
\multicolumn{1}{|c|}{\multirow{2}{*}{Ad}}                                       & gap                                                       & \multicolumn{1}{c|}{\textbf{0.129}} & 0.135          & \multicolumn{1}{c|}{\textbf{0.131}} & 0.140          & \multicolumn{1}{c|}{\textbf{0.130}} & 0.142          & \multicolumn{1}{c|}{\textbf{0.131}} & 0.134          \\
\multicolumn{1}{|c|}{}                                                          & violations                                                & \multicolumn{1}{c|}{{\ul 0.017}}    & 0.055          & \multicolumn{1}{c|}{{\ul 0.004}}    & 0.006          & \multicolumn{1}{c|}{{\ul 0.000}}    & 0.013          & \multicolumn{1}{c|}{{\ul 0.000}}    & 0.004          \\ \hline
\multicolumn{1}{|c|}{\multirow{2}{*}{Average}}                                  & gap                                                       & \multicolumn{1}{c|}{\textbf{0.070}} & 0.077          & \multicolumn{1}{c|}{\textbf{0.070}} & 0.073          & \multicolumn{1}{c|}{\textbf{0.069}} & 0.076          & \multicolumn{1}{c|}{\textbf{0.069}} & 0.078          \\
\multicolumn{1}{|c|}{}                                                          & violations                                                & \multicolumn{1}{c|}{{\ul 0.065}}    & 0.270          & \multicolumn{1}{c|}{{\ul 0.050}}    & 0.185          & \multicolumn{1}{c|}{{\ul 0.027}}    & 0.140          & \multicolumn{1}{c|}{{\ul 0.007}}    & 0.053          \\ \hline
\end{tabular}
\end{table}
}

%% file: supp_tables/summary.tex
\begin{table*}[tbh]
\centering
\setlength{\tabcolsep}{1.15pt}

\caption{\label{tab:summary} Summary: best \sscmpe has significantly lower violations compared to \texttt{other} methods on all the problems and over all the chosen $q$ values. It has comparable gap to the other methods.}
\vspace{0.35cm}

\begin{tabular}{|cc|cc|cc|cc|cc|}
\hline
\multicolumn{2}{|c|}{q}                                                                                                                      & \multicolumn{2}{c|}{10}                              & \multicolumn{2}{c|}{30}                              & \multicolumn{2}{c|}{60}                              & \multicolumn{2}{c|}{90}                              \\ \hline
\multicolumn{1}{|c|}{\begin{tabular}[c]{@{}c@{}}Models\\ /Datasets\end{tabular}} & \begin{tabular}[c]{@{}c@{}}Gap\\ /Violations\end{tabular} & \multicolumn{1}{c|}{\begin{tabular}[c]{@{}c@{}}best\\ \small{\sscmpe}\end{tabular}}   & others         & \multicolumn{1}{c|}{\begin{tabular}[c]{@{}c@{}}best\\ \small{\sscmpe}\end{tabular}}   & others         & \multicolumn{1}{c|}{\begin{tabular}[c]{@{}c@{}}best\\ \small{\sscmpe}\end{tabular}}  & others         & \multicolumn{1}{c|}{\begin{tabular}[c]{@{}c@{}}best\\ \small{\sscmpe}\end{tabular}}     & others         \\ \hline
\multicolumn{1}{|c|}{\multirow{2}{*}{Segment-12}}                                & Gap                                                       & \multicolumn{1}{c|}{0.057}          & \textbf{0.054} & \multicolumn{1}{c|}{\textbf{0.051}} & 0.052          & \multicolumn{1}{c|}{0.053}          & \textbf{0.051} & \multicolumn{1}{c|}{\textbf{0.050}} & 0.051          \\
\multicolumn{1}{|c|}{}                                                           & Violations                                                & \multicolumn{1}{c|}{\textbf{0.152}} & 0.511          & \multicolumn{1}{c|}{\textbf{0.166}} & 0.500          & \multicolumn{1}{c|}{\textbf{0.084}} & 0.348          & \multicolumn{1}{c|}{\textbf{0.021}} & 0.131          \\ \hline
\multicolumn{1}{|c|}{\multirow{2}{*}{Segment-14}}                                & Gap                                                       & \multicolumn{1}{c|}{0.050}          & \textbf{0.049} & \multicolumn{1}{c|}{\textbf{0.047}} & 0.049          & \multicolumn{1}{c|}{0.048}          & 0.048          & \multicolumn{1}{c|}{0.051}          & 0.051          \\
\multicolumn{1}{|c|}{}                                                           & Violations                                                & \multicolumn{1}{c|}{\textbf{0.134}} & 0.507          & \multicolumn{1}{c|}{\textbf{0.088}} & 0.414          & \multicolumn{1}{c|}{\textbf{0.066}} & 0.394          & \multicolumn{1}{c|}{\textbf{0.046}} & 0.207          \\ \hline
\multicolumn{1}{|c|}{\multirow{2}{*}{Segment-15}}                                & Gap                                                       & \multicolumn{1}{c|}{0.051}          & 0.051          & \multicolumn{1}{c|}{\textbf{0.050}} & 0.051          & \multicolumn{1}{c|}{0.052}          & \textbf{0.051} & \multicolumn{1}{c|}{\textbf{0.051}} & 0.052          \\
\multicolumn{1}{|c|}{}                                                           & Violations                                                & \multicolumn{1}{c|}{\textbf{0.060}} & 0.570          & \multicolumn{1}{c|}{\textbf{0.060}} & 0.360          & \multicolumn{1}{c|}{\textbf{0.049}} & 0.248          & \multicolumn{1}{c|}{\textbf{0.001}} & 0.061          \\ \hline
\multicolumn{1}{|c|}{\multirow{2}{*}{Grids-17}}                                  & Gap                                                       & \multicolumn{1}{c|}{0.072}          & \textbf{0.054} & \multicolumn{1}{c|}{0.069}          & \textbf{0.057} & \multicolumn{1}{c|}{\textbf{0.066}} & 0.067          & \multicolumn{1}{c|}{0.063}          & \textbf{0.058} \\
\multicolumn{1}{|c|}{}                                                           & Violations                                                & \multicolumn{1}{c|}{\textbf{0.035}} & 0.043          & \multicolumn{1}{c|}{\textbf{0.013}} & 0.003          & \multicolumn{1}{c|}{\textbf{0.002}} & 0.001          & \multicolumn{1}{c|}{\textbf{0.001}} & 0.002          \\ \hline
\multicolumn{1}{|c|}{\multirow{2}{*}{Grids-18}}                                  & Gap                                                       & \multicolumn{1}{c|}{0.064}          & \textbf{0.056} & \multicolumn{1}{c|}{0.067}          & \textbf{0.060} & \multicolumn{1}{c|}{\textbf{0.060}} & 0.065          & \multicolumn{1}{c|}{0.065}          & \textbf{0.064} \\
\multicolumn{1}{|c|}{}                                                           & Violations                                                & \multicolumn{1}{c|}{\textbf{0.017}} & 0.060          & \multicolumn{1}{c|}{\textbf{0.019}} & 0.001          & \multicolumn{1}{c|}{\textbf{0.000}} & 0.003          & \multicolumn{1}{c|}{\textbf{0.000}} & 0.005          \\ \hline
\multicolumn{1}{|c|}{\multirow{2}{*}{DNA}}                                       & Gap                                                       & \multicolumn{1}{c|}{0.138}          & \textbf{0.135} & \multicolumn{1}{c|}{0.138}          & \textbf{0.136} & \multicolumn{1}{c|}{0.137}          & \textbf{0.136} & \multicolumn{1}{c|}{0.139}          & \textbf{0.137} \\
\multicolumn{1}{|c|}{}                                                           & Violations                                                & \multicolumn{1}{c|}{\textbf{0.013}} & 0.048          & \multicolumn{1}{c|}{\textbf{0.002}} & 0.062          & \multicolumn{1}{c|}{\textbf{0.001}} & 0.003          & \multicolumn{1}{c|}{\textbf{0.001}} & 0.004          \\ \hline
\multicolumn{1}{|c|}{\multirow{2}{*}{20NewsGr}}                                  & Gap                                                       & \multicolumn{1}{c|}{0.043}          & 0.043          & \multicolumn{1}{c|}{\textbf{0.045}} & 0.046          & \multicolumn{1}{c|}{\textbf{0.044}} & 0.045          & \multicolumn{1}{c|}{\textbf{0.044}} & 0.044          \\
\multicolumn{1}{|c|}{}                                                           & Violations                                                & \multicolumn{1}{c|}{\textbf{0.069}} & 0.278          & \multicolumn{1}{c|}{\textbf{0.054}} & 0.129          & \multicolumn{1}{c|}{\textbf{0.007}} & 0.024          & \multicolumn{1}{c|}{\textbf{0.001}} & 0.001          \\ \hline
\multicolumn{1}{|c|}{\multirow{2}{*}{WebKB}}                                     & Gap                                                       & \multicolumn{1}{c|}{0.059}          & \textbf{0.054} & \multicolumn{1}{c|}{0.058}          & \textbf{0.054} & \multicolumn{1}{c|}{0.056}          & \textbf{0.054} & \multicolumn{1}{c|}{\textbf{0.053}} & 0.054          \\
\multicolumn{1}{|c|}{}                                                           & Violations                                                & \multicolumn{1}{c|}{\textbf{0.074}} & 0.149          & \multicolumn{1}{c|}{\textbf{0.029}} & 0.096          & \multicolumn{1}{c|}{\textbf{0.001}} & 0.056          & \multicolumn{1}{c|}{\textbf{0.001}} & 0.013          \\ \hline
\multicolumn{1}{|c|}{\multirow{2}{*}{BBC}}                                       & Gap                                                       & \multicolumn{1}{c|}{0.038}          & \textbf{0.036} & \multicolumn{1}{c|}{0.043}          & \textbf{0.036} & \multicolumn{1}{c|}{0.042}          & \textbf{0.037} & \multicolumn{1}{c|}{0.040}          & \textbf{0.037} \\
\multicolumn{1}{|c|}{}                                                           & Violations                                                & \multicolumn{1}{c|}{\textbf{0.074}} & 0.336          & \multicolumn{1}{c|}{\textbf{0.067}} & 0.160          & \multicolumn{1}{c|}{\textbf{0.056}} & 0.149          & \multicolumn{1}{c|}{\textbf{0.002}} & 0.029          \\ \hline
\multicolumn{1}{|c|}{\multirow{2}{*}{Ad}}                                        & Gap                                                       & \multicolumn{1}{c|}{\textbf{0.129}} & 0.135          & \multicolumn{1}{c|}{\textbf{0.131}} & 0.140          & \multicolumn{1}{c|}{\textbf{0.130}} & 0.142          & \multicolumn{1}{c|}{\textbf{0.131}} & 0.134          \\
\multicolumn{1}{|c|}{}                                                           & Violations                                                & \multicolumn{1}{c|}{\textbf{0.017}} & 0.055          & \multicolumn{1}{c|}{\textbf{0.004}} & 0.006          & \multicolumn{1}{c|}{\textbf{0.000}} & 0.013          & \multicolumn{1}{c|}{\textbf{0.000}} & 0.004          \\ \hline
\multicolumn{1}{|c|}{\multirow{2}{*}{Average}}                                   & Gap                                                       & \multicolumn{1}{c|}{0.070}          & \textbf{0.067} & \multicolumn{1}{c|}{0.070}          & \textbf{0.068} & \multicolumn{1}{c|}{\textbf{0.069}} & 0.070          & \multicolumn{1}{c|}{0.069}          & \textbf{0.068} \\
\multicolumn{1}{|c|}{}                                                           & Violations                                                & \multicolumn{1}{c|}{\textbf{0.065}} & 0.256          & \multicolumn{1}{c|}{\textbf{0.050}} & 0.173          & \multicolumn{1}{c|}{\textbf{0.027}} & 0.124          & \multicolumn{1}{c|}{\textbf{0.007}} & 0.046          \\ \hline
\end{tabular}
\end{table*}

%% file: supp_tables/SSL-feasible-only-figures.tex
\begin{figure*}[ht!]
    \centering
    \begin{subfigure}{0.24\textwidth}
        \centering
        \includegraphics[width=\textwidth]{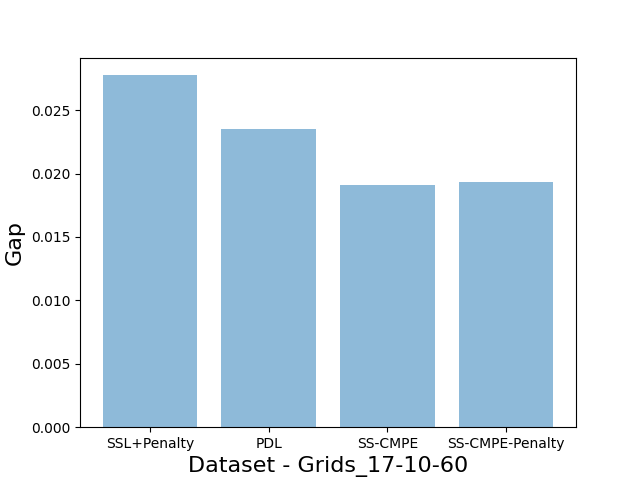}
    \end{subfigure}
    \hfill
    \begin{subfigure}{0.24\textwidth}
        \centering
        \includegraphics[width=\textwidth]{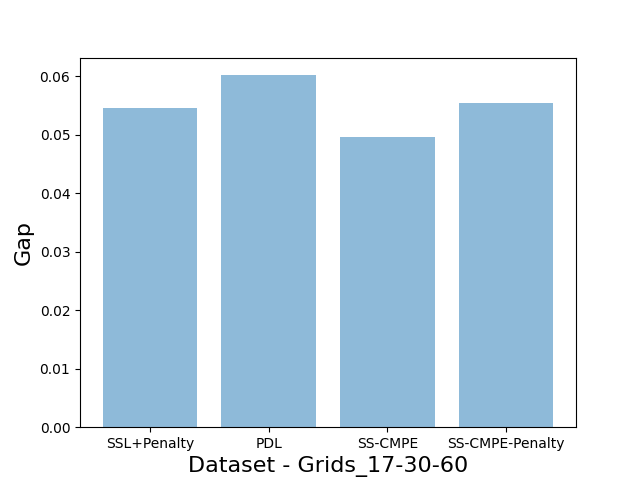}
    \end{subfigure}
    \hfill
    \begin{subfigure}{0.24\textwidth}
        \centering
        \includegraphics[width=\textwidth]{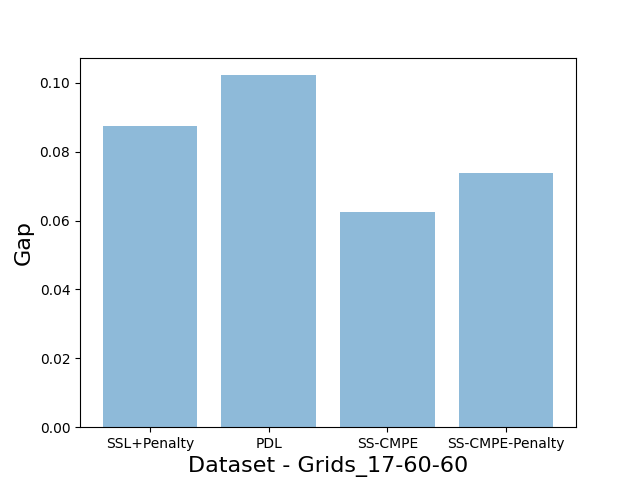}
    \end{subfigure}
    \hfill
    \begin{subfigure}{0.24\textwidth}
        \centering
        \includegraphics[width=\textwidth]{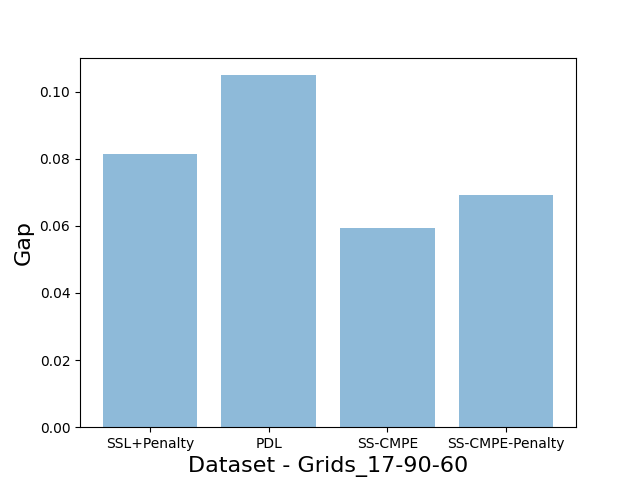}
    \end{subfigure}
    
    \vspace{0.5cm} 
    
    \begin{subfigure}{0.24\textwidth}
        \centering
        \includegraphics[width=\textwidth]{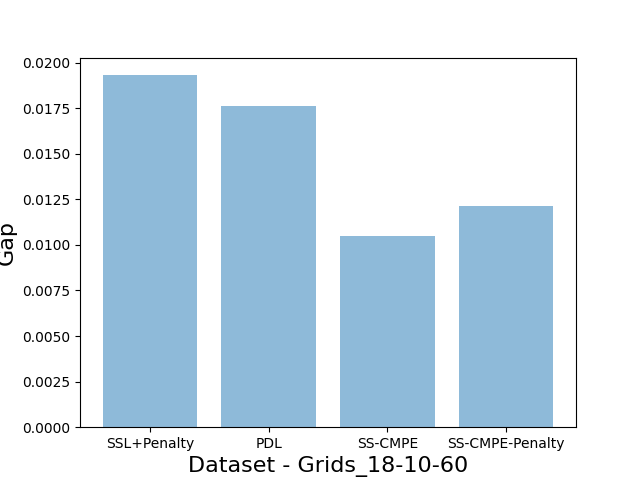}
    \end{subfigure}
    \hfill
    \begin{subfigure}{0.24\textwidth}
        \centering
        \includegraphics[width=\textwidth]{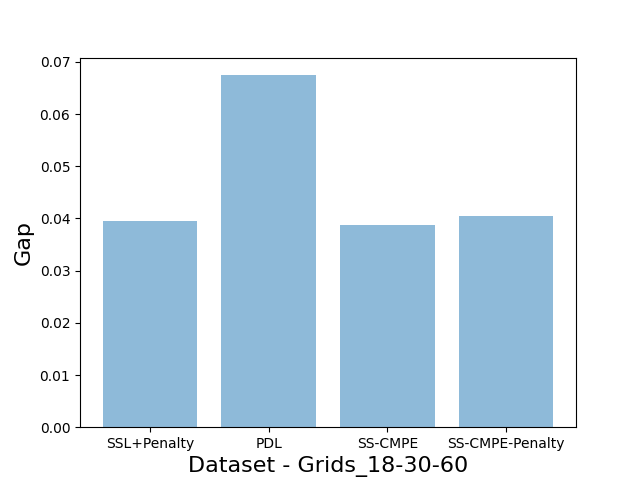}
    \end{subfigure}
    \hfill
    \begin{subfigure}{0.24\textwidth}
        \centering
        \includegraphics[width=\textwidth]{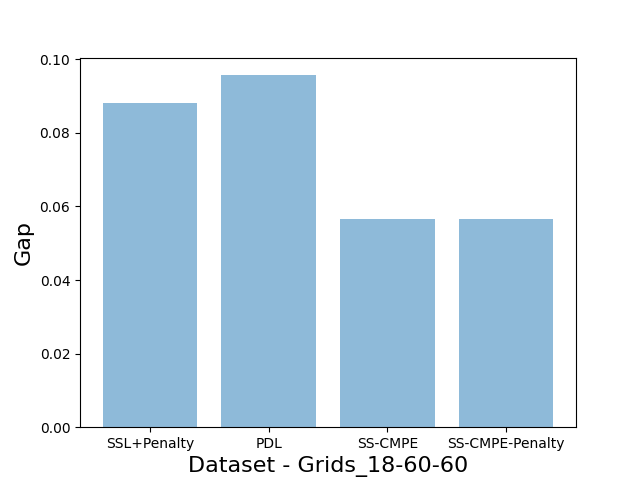}
    \end{subfigure}
    \hfill
    \begin{subfigure}{0.24\textwidth}
        \centering
        \includegraphics[width=\textwidth]{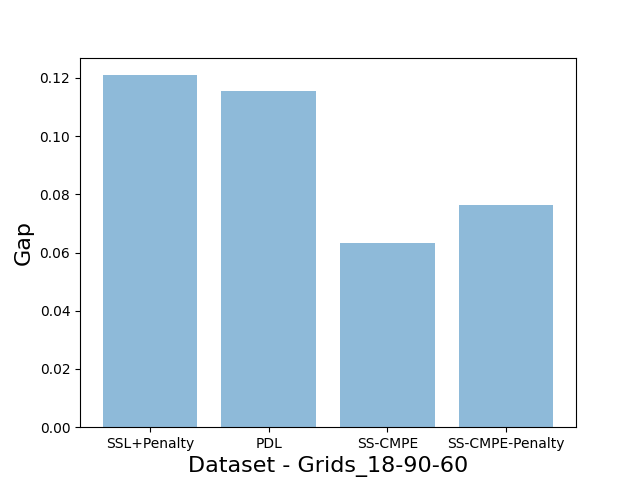}
    \end{subfigure}
    
    \vspace{0.5cm} 
    
    \begin{subfigure}{0.24\textwidth}
        \centering
        \includegraphics[width=\textwidth]{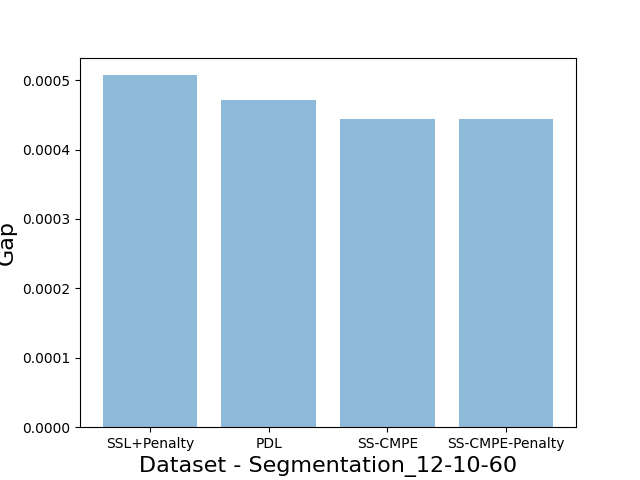}
    \end{subfigure}
    \hfill
    \begin{subfigure}{0.24\textwidth}
        \centering
        \includegraphics[width=\textwidth]{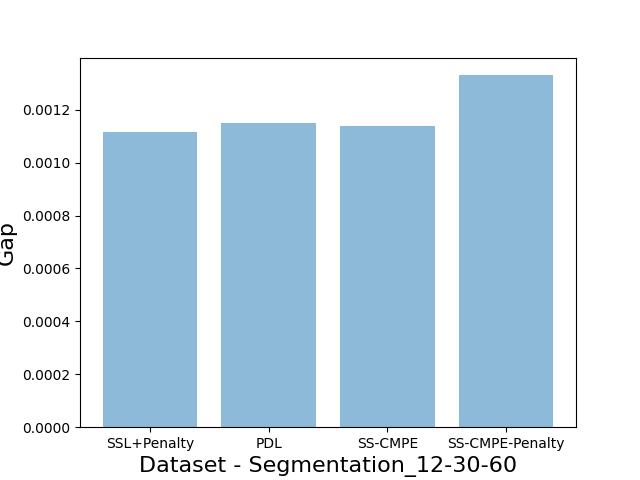}
    \end{subfigure}
    \hfill
    \begin{subfigure}{0.24\textwidth}
        \centering
        \includegraphics[width=\textwidth]{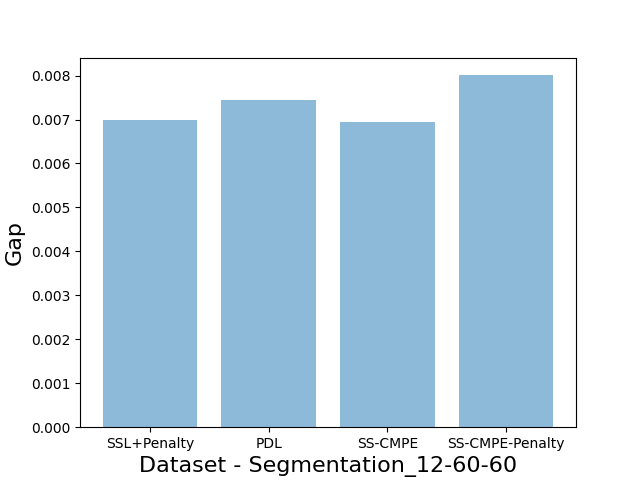}
    \end{subfigure}
    \hfill
    \begin{subfigure}{0.24\textwidth}
        \centering
        \includegraphics[width=\textwidth]{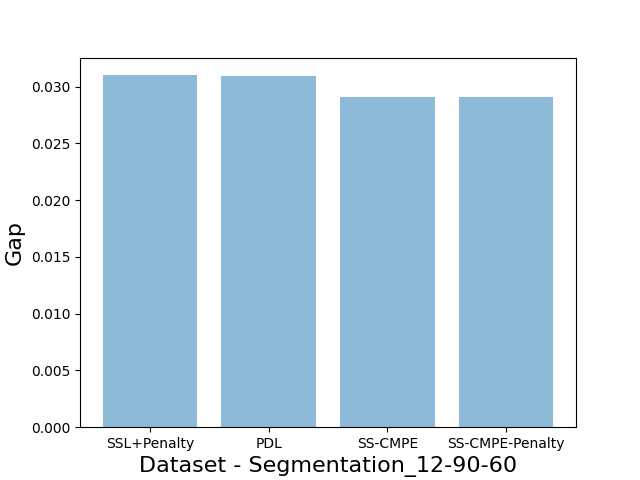}
    \end{subfigure}
    
    \vspace{0.5cm} 
    
    \begin{subfigure}{0.24\textwidth}
        \centering
        \includegraphics[width=\textwidth]{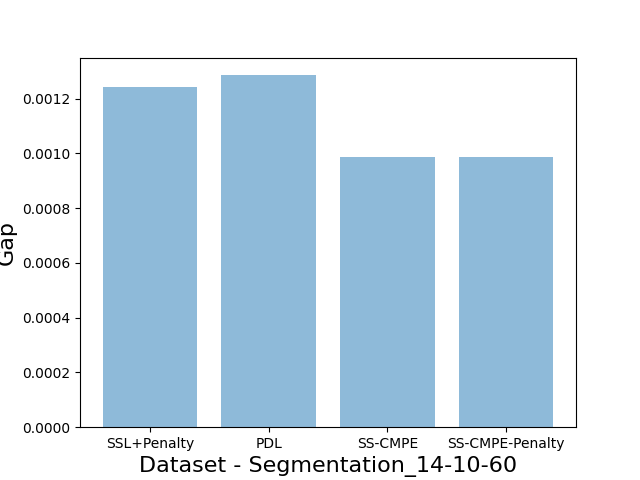}
    \end{subfigure}
    \hfill
    \begin{subfigure}{0.24\textwidth}
        \centering
        \includegraphics[width=\textwidth]{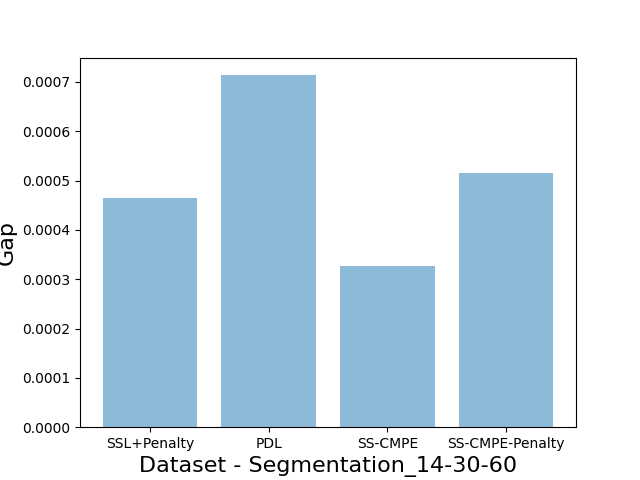}
    \end{subfigure}
    \hfill
    \begin{subfigure}{0.24\textwidth}
        \centering
        \includegraphics[width=\textwidth]{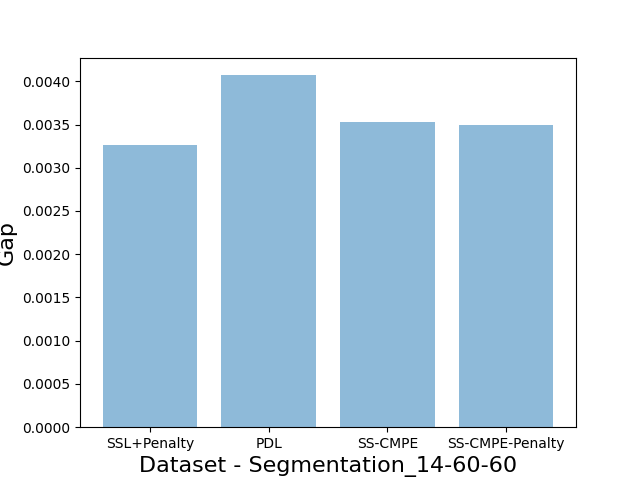}
    \end{subfigure}
    \hfill
    \begin{subfigure}{0.24\textwidth}
        \centering
        \includegraphics[width=\textwidth]{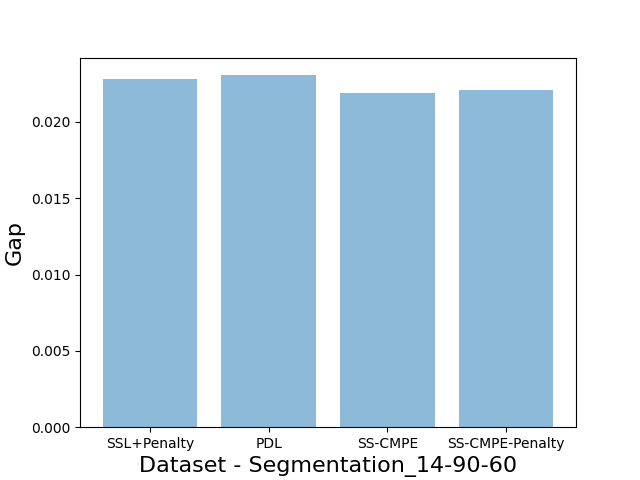}
    \end{subfigure}
    
    \vspace{0.5cm} 
    
    \begin{subfigure}{0.24\textwidth}
        \centering
        \includegraphics[width=\textwidth]{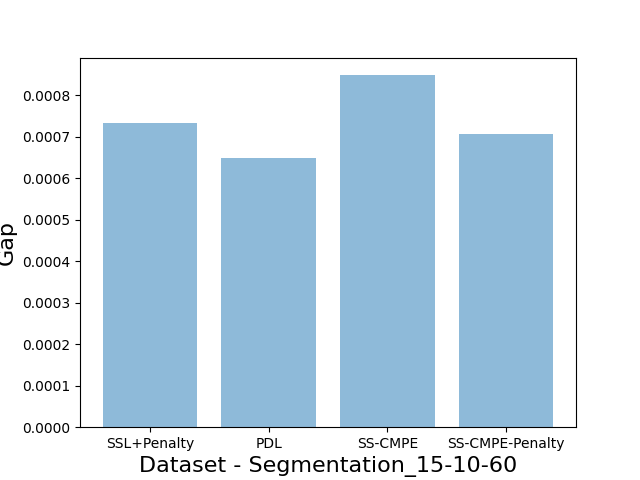}
    \end{subfigure}
    \hfill
    \begin{subfigure}{0.24\textwidth}
        \centering
        \includegraphics[width=\textwidth]{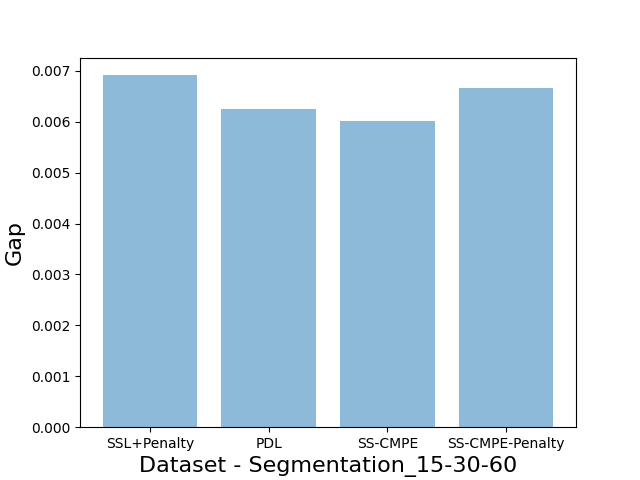}
    \end{subfigure}
    \hfill
    \begin{subfigure}{0.24\textwidth}
        \centering
        \includegraphics[width=\textwidth]{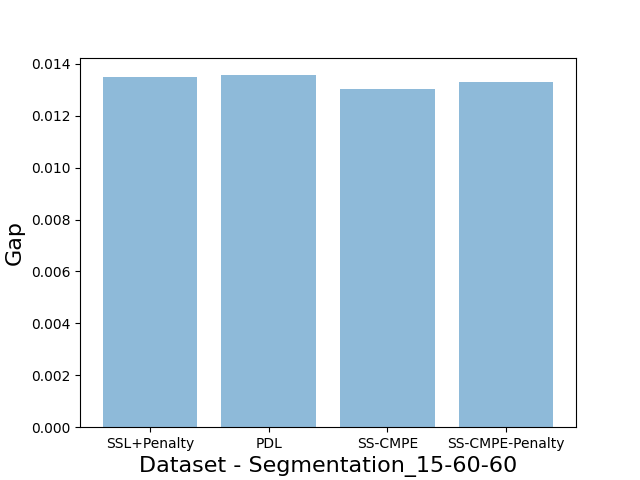}
    \end{subfigure}
    \hfill
    \begin{subfigure}{0.24\textwidth}
        \centering
        \includegraphics[width=\textwidth]{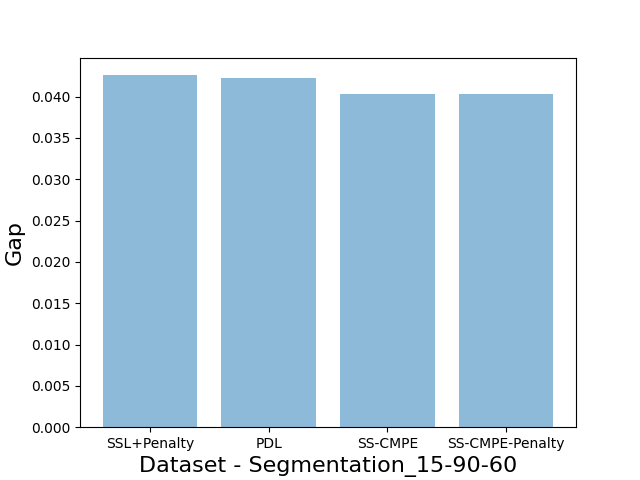}
    \end{subfigure}
    
    \vspace{0.5cm} 
    
    
    \caption{Illustration of the optimality gap for self-supervised methods (on feasible examples only) for all approaches. Lower is better.}
    \label{fig:supp_only_feasible}
\end{figure*}

%% file: supp_tables/seg12-q.tex
\begin{table*}[tbh]
\centering
\caption{Average gap and constraint violations over test samples for models applied to the Segmentation12 Dataset for different q values. The plot displays the mean values of the average gap and constraint violations, with standard deviations denoted by ±}
\vspace{0.35cm}

\label{tab:sup-seg12-q}
\begin{tabular}{|cc|l|l|l|l|}
\hline
\multicolumn{2}{|c|}{q} & 10 & 30 & 60 & 90 \\ \hline
\multicolumn{2}{|c|}{ILP Obj} & 491.150 & 476.654 & 467.913 & 461.967 \\ \hline
\multicolumn{1}{|c|}{\multirow{2}{*}{\mae}} & Gap & 0.064 ± 0.051 & 0.061 ± 0.049 & 0.053 ± 0.042 & 0.052 ± 0.040 \\
\multicolumn{1}{|c|}{} & Violations & 0.569 ± 0.495 & 0.545 ± 0.498 & 0.430 ± 0.495 & 0.131 ± 0.337 \\ \hline
\multicolumn{1}{|c|}{\multirow{2}{*}{\mse}} & Gap & 0.054 ± 0.042 & 0.053 ± 0.042 & 0.051 ± 0.041 & 0.051 ± 0.041 \\
\multicolumn{1}{|c|}{} & Violations & 0.776 ± 0.417 & 0.580 ± 0.494 & 0.486 ± 0.500 & 0.186 ± 0.390 \\ \hline
\multicolumn{1}{|c|}{\multirow{2}{*}{\maepen}} & Gap & 0.064 ± 0.050 & 0.061 ± 0.049 & 0.059 ± 0.046 & 0.052 ± 0.043 \\
\multicolumn{1}{|c|}{} & Violations & 0.511 ± 0.500 & 0.500 ± 0.500 & 0.348 ± 0.476 & 0.140 ± 0.347 \\ \hline
\multicolumn{1}{|c|}{\multirow{2}{*}{\msepen}} & Gap & 0.054 ± 0.042 & 0.052 ± 0.041 & 0.051 ± 0.040 & 0.051 ± 0.042 \\
\multicolumn{1}{|c|}{} & Violations & 0.651 ± 0.477 & 0.505 ± 0.500 & 0.486 ± 0.500 & 0.186 ± 0.390 \\ \hline
\multicolumn{1}{|c|}{\multirow{2}{*}{\sslpen}} & Gap & 0.054 ± 0.043 & 0.052 ± 0.040 & 0.051 ± 0.041 & 0.051 ± 0.040 \\
\multicolumn{1}{|c|}{} & Violations & 0.790 ± 0.407 & 0.622 ± 0.485 & 0.486 ± 0.500 & 0.186 ± 0.390 \\ \hline
\multicolumn{1}{|c|}{\multirow{2}{*}{\pdl}} & Gap & 0.063 ± 0.050 & 0.052 ± 0.041 & 0.052 ± 0.041 & 0.051 ± 0.041 \\
\multicolumn{1}{|c|}{} & Violations & 0.545 ± 0.498 & 0.622 ± 0.485 & 0.517 ± 0.500 & 0.163 ± 0.369 \\ \hline
\multicolumn{1}{|c|}{\multirow{2}{*}{\sscmpe}} & Gap & 0.057 ± 0.044 & 0.051 ± 0.040 & 0.053 ± 0.043 & 0.050 ± 0.041 \\
\multicolumn{1}{|c|}{} & Violations & 0.503 ± 0.293 & 0.346 ± 0.485 & 0.257 ± 0.437 & 0.104 ± 0.305 \\ \hline
\multicolumn{1}{|c|}{\multirow{2}{*}{\sscmpepen}} & Gap & 0.058 ± 0.043 & 0.052 ± 0.040 & 0.053 ± 0.043 & 0.051 ± 0.041 \\
\multicolumn{1}{|c|}{} & Violations & 0.152 ± 0.359 & 0.166 ± 0.372 & 0.084 ± 0.277 & 0.021 ± 0.143 \\ \hline
\end{tabular}
\end{table*}

%% file: supp_tables/seg14-q.tex
\begin{table*}[!ht]
\centering
\caption{Average gap and constraint violations over test samples for models applied to the Segmentation14 Dataset for different q values. The plot displays the mean values of the average gap and constraint violations, with standard deviations denoted by ±.}
\vspace{0.35cm}

\label{tab:sup-seg14-q}
\begin{tabular}{|cc|c|c|c|c|}
\hline
\multicolumn{2}{|c|}{q} & 10 & 30 & 60 & 90 \\ \hline
\multicolumn{2}{|c|}{ILP Obj} & 493.647 & 482.837 & 476.145 & 470.485 \\ \hline
\multicolumn{1}{|c|}{\multirow{2}{*}{\mae}} & Gap & 0.067 ± 0.051 & 0.062 ± 0.048 & 0.062 ± 0.047 & 0.051 ± 0.040 \\
\multicolumn{1}{|c|}{} & Violations & 0.691 ± 0.462 & 0.623 ± 0.485 & 0.435 ± 0.496 & 0.271 ± 0.444 \\ \hline
\multicolumn{1}{|c|}{\multirow{2}{*}{\mse}} & Gap & 0.051 ± 0.039 & 0.049 ± 0.038 & 0.048 ± 0.037 & 0.053 ± 0.041 \\
\multicolumn{1}{|c|}{} & Violations & 0.810 ± 0.392 & 0.818 ± 0.386 & 0.606 ± 0.489 & 0.252 ± 0.434 \\ \hline
\multicolumn{1}{|c|}{\multirow{2}{*}{\maepen}} & Gap & 0.065 ± 0.051 & 0.061 ± 0.048 & 0.060 ± 0.046 & 0.054 ± 0.043 \\
\multicolumn{1}{|c|}{} & Violations & 0.693 ± 0.461 & 0.616 ± 0.487 & 0.410 ± 0.492 & 0.207 ± 0.405 \\ \hline
\multicolumn{1}{|c|}{\multirow{2}{*}{\msepen}} & Gap & 0.049 ± 0.039 & 0.049 ± 0.039 & 0.049 ± 0.037 & 0.054 ± 0.041 \\
\multicolumn{1}{|c|}{} & Violations & 0.810 ± 0.392 & 0.803 ± 0.397 & 0.601 ± 0.490 & 0.215 ± 0.411 \\ \hline
\multicolumn{1}{|c|}{\multirow{2}{*}{\sslpen}} & Gap & 0.061 ± 0.047 & 0.050 ± 0.038 & 0.050 ± 0.039 & 0.053 ± 0.042 \\
\multicolumn{1}{|c|}{} & Violations & 0.590 ± 0.492 & 0.618 ± 0.486 & 0.394 ± 0.489 & 0.207 ± 0.405 \\ \hline
\multicolumn{1}{|c|}{\multirow{2}{*}{\pdl}} & Gap & 0.058 ± 0.045 & 0.068 ± 0.051 & 0.056 ± 0.043 & 0.054 ± 0.043 \\
\multicolumn{1}{|c|}{} & Violations & 0.507 ± 0.500 & 0.414 ± 0.493 & 0.403 ± 0.491 & 0.207 ± 0.405 \\ \hline
\multicolumn{1}{|c|}{\multirow{2}{*}{\sscmpe}} & Gap & 0.050 ± 0.038 & 0.047 ± 0.037 & 0.048 ± 0.037 & 0.051 ± 0.040 \\
\multicolumn{1}{|c|}{} & Violations & 0.502 ± 0.295 & 0.444 ± 0.401 & 0.309 ± 0.497 & 0.150 ± 0.358 \\ \hline
\multicolumn{1}{|c|}{\multirow{2}{*}{\sscmpepen}} & Gap & 0.050 ± 0.039 & 0.048 ± 0.038 & 0.050 ± 0.037 & 0.052 ± 0.042 \\
\multicolumn{1}{|c|}{} & Violations & 0.134 ± 0.340 & 0.088 ± 0.284 & 0.066 ± 0.248 & 0.046 ± 0.211 \\ \hline
\end{tabular}
\end{table*}

%% file: supp_tables/seg15-q.tex
\begin{table*}[ht]
\centering
\caption{Average gap and constraint violations over test samples for models applied to the Segmentation15 Dataset for different q values. The plot displays the mean values of the average gap and constraint violations, with standard deviations denoted by ±.}
\vspace{0.35cm}

\label{tab:sup-seg15-q}
\begin{tabular}{|cc|c|c|c|c|}
\hline
\multicolumn{2}{|c|}{q} & 10 & 30 & 60 & 90 \\ \hline
\multicolumn{2}{|c|}{ILP Obj} & 531.436 & 520.647 & 516.797 & 514.276 \\ \hline
\multicolumn{1}{|c|}{\multirow{2}{*}{\mae}} & Gap & 0.053 ± 0.043 & 0.054 ± 0.043 & 0.053 ± 0.042 & 0.052 ± 0.041 \\
\multicolumn{1}{|c|}{} & Violations & 0.570 ± 0.495 & 0.417 ± 0.493 & 0.248 ± 0.432 & 0.061 ± 0.239 \\ \hline
\multicolumn{1}{|c|}{\multirow{2}{*}{\mse}} & Gap & 0.051 ± 0.038 & 0.052 ± 0.041 & 0.052 ± 0.040 & 0.053 ± 0.041 \\
\multicolumn{1}{|c|}{} & Violations & 0.833 ± 0.373 & 0.715 ± 0.452 & 0.450 ± 0.498 & 0.076 ± 0.265 \\ \hline
\multicolumn{1}{|c|}{\multirow{2}{*}{\maepen}} & Gap & 0.056 ± 0.043 & 0.054 ± 0.042 & 0.053 ± 0.041 & 0.053 ± 0.042 \\
\multicolumn{1}{|c|}{} & Violations & 0.616 ± 0.486 & 0.457 ± 0.498 & 0.265 ± 0.441 & 0.075 ± 0.263 \\ \hline
\multicolumn{1}{|c|}{\multirow{2}{*}{\msepen}} & Gap & 0.052 ± 0.040 & 0.051 ± 0.040 & 0.051 ± 0.040 & 0.052 ± 0.041 \\
\multicolumn{1}{|c|}{} & Violations & 0.833 ± 0.373 & 0.715 ± 0.452 & 0.450 ± 0.498 & 0.076 ± 0.265 \\ \hline
\multicolumn{1}{|c|}{\multirow{2}{*}{\sslpen}} & Gap & 0.053 ± 0.042 & 0.059 ± 0.047 & 0.052 ± 0.041 & 0.054 ± 0.043 \\
\multicolumn{1}{|c|}{} & Violations & 0.676 ± 0.468 & 0.360 ± 0.480 & 0.274 ± 0.446 & 0.097 ± 0.295 \\ \hline
\multicolumn{1}{|c|}{\multirow{2}{*}{\pdl}} & Gap & 0.051 ± 0.040 & 0.053 ± 0.043 & 0.051 ± 0.040 & 0.054 ± 0.043 \\
\multicolumn{1}{|c|}{} & Violations & 0.698 ± 0.459 & 0.461 ± 0.499 & 0.392 ± 0.488 & 0.086 ± 0.280 \\ \hline
\multicolumn{1}{|c|}{\multirow{2}{*}{\sscmpe}} & Gap & 0.051 ± 0.040 & 0.050 ± 0.041 & 0.052 ± 0.040 & 0.051 ± 0.040 \\
\multicolumn{1}{|c|}{} & Violations & 0.366 ± 0.474 & 0.298 ± 0.499 & 0.225 ± 0.417 & 0.059 ± 0.236 \\ \hline
\multicolumn{1}{|c|}{\multirow{2}{*}{\sscmpepen}} & Gap & 0.052 ± 0.040 & 0.052 ± 0.042 & 0.052 ± 0.041 & 0.051 ± 0.041 \\
\multicolumn{1}{|c|}{} & Violations & 0.060 ± 0.238 & 0.060 ± 0.238 & 0.049 ± 0.215 & 0.001 ± 0.032 \\ \hline
\end{tabular}
\end{table*}

%% file: supp_tables/grids17-q.tex
\begin{table*}[]
\centering
\caption{Average gap and constraint violations over test samples for models applied to the Grids17 Dataset for different q values. The plot displays the mean values of the average gap and constraint violations, with standard deviations denoted by ±.}
\vspace{0.35cm}

\label{tab:sup-grids17-q}
\begin{tabular}{|cc|c|c|c|c|}
\hline
\multicolumn{2}{|c|}{q} & 10 & 30 & 60 & 90 \\ \hline
\multicolumn{2}{|c|}{ILP Obj} & 2892.585191 & 2884.506703 & 2877.311872 & 2878.147272 \\ \hline
\multicolumn{1}{|c|}{\multirow{2}{*}{\mae}} & Gap & 0.119 ± 0.078 & 0.133 ± 0.088 & 0.125 ± 0.084 & 0.114 ± 0.078 \\
\multicolumn{1}{|c|}{} & Violations & 0.565 ± 0.496 & 0.376 ± 0.485 & 0.122 ± 0.328 & 0.020 ± 0.140 \\ \hline
\multicolumn{1}{|c|}{\multirow{2}{*}{\mse}} & Gap & 0.054 ± 0.041 & 0.057 ± 0.045 & 0.067 ± 0.054 & 0.075 ± 0.056 \\
\multicolumn{1}{|c|}{} & Violations & 0.314 ± 0.464 & 0.152 ± 0.359 & 0.044 ± 0.205 & 0.013 ± 0.111 \\ \hline
\multicolumn{1}{|c|}{\multirow{2}{*}{\maepen}} & Gap & 0.129 ± 0.081 & 0.144 ± 0.089 & 0.137 ± 0.087 & 0.127 ± 0.085 \\
\multicolumn{1}{|c|}{} & Violations & 0.534 ± 0.499 & 0.376 ± 0.485 & 0.142 ± 0.350 & 0.019 ± 0.137 \\ \hline
\multicolumn{1}{|c|}{\multirow{2}{*}{\msepen}} & Gap & 0.054 ± 0.041 & 0.059 ± 0.045 & 0.069 ± 0.054 & 0.058 ± 0.044 \\
\multicolumn{1}{|c|}{} & Violations & 0.304 ± 0.460 & 0.125 ± 0.331 & 0.044 ± 0.205 & 0.002 ± 0.045 \\ \hline
\multicolumn{1}{|c|}{\multirow{2}{*}{\sslpen}} & Gap & 0.104 ± 0.060 & 0.079 ± 0.056 & 0.093 ± 0.059 & 0.087 ± 0.058 \\
\multicolumn{1}{|c|}{} & Violations & 0.149 ± 0.357 & 0.013 ± 0.113 & 0.004 ± 0.067 & 0.026 ± 0.159 \\ \hline
\multicolumn{1}{|c|}{\multirow{2}{*}{\pdl}} & Gap & 0.086 ± 0.056 & 0.087 ± 0.058 & 0.109 ± 0.063 & 0.112 ± 0.062 \\
\multicolumn{1}{|c|}{} & Violations & 0.043 ± 0.202 & 0.003 ± 0.055 & 0.001 ± 0.022 & 0.014 ± 0.118 \\ \hline
\multicolumn{1}{|c|}{\multirow{2}{*}{\sscmpe}} & Gap & 0.072 ± 0.051 & 0.071 ± 0.052 & 0.066 ± 0.048 & 0.063 ± 0.049 \\
\multicolumn{1}{|c|}{} & Violations & 0.146 ± 0.353 & 0.032 ± 0.176 & 0.024 ± 0.152 & 0.001 ± 0.022 \\ \hline
\multicolumn{1}{|c|}{\multirow{2}{*}{\sscmpepen}} & Gap & 0.073 ± 0.052 & 0.069 ± 0.051 & 0.073 ± 0.052 & 0.066 ± 0.050 \\
\multicolumn{1}{|c|}{} & Violations & 0.035 ± 0.185 & 0.013 ± 0.113 & 0.002 ± 0.039 & 0.001 ± 0.022 \\ \hline
\end{tabular}
\end{table*}

%% file: supp_tables/grids18-q.tex
\begin{table*}[]
\centering
\caption{Average gap and constraint violations over test samples for models applied to the Grids18 Dataset for different q values. The plot displays the mean values of the average gap and constraint violations, with standard deviations denoted by ±.}
\vspace{0.35cm}

\label{tab:sup-grids18-q}
\begin{tabular}{|cc|c|c|c|c|}
\hline
\multicolumn{2}{|c|}{q} & 10 & 30 & 60 & 90 \\ \hline
\multicolumn{2}{|c|}{ILP Obj} & 4185.600003 & 4166.310635 & 4158.737261 & 4167.11386 \\ \hline
\multicolumn{1}{|c|}{\multirow{2}{*}{\mae}} & Gap & 0.138 ± 0.087 & 0.142 ± 0.091 & 0.122 ± 0.081 & 0.102 ± 0.073 \\
\multicolumn{1}{|c|}{} & Violations & 0.606 ± 0.489 & 0.389 ± 0.488 & 0.178 ± 0.383 & 0.019 ± 0.138 \\ \hline
\multicolumn{1}{|c|}{\multirow{2}{*}{\mse}} & Gap & 0.056 ± 0.044 & 0.060 ± 0.047 & 0.069 ± 0.056 & 0.064 ± 0.050 \\
\multicolumn{1}{|c|}{} & Violations & 0.332 ± 0.471 & 0.194 ± 0.395 & 0.178 ± 0.383 & 0.015 ± 0.122 \\ \hline
\multicolumn{1}{|c|}{\multirow{2}{*}{\maepen}} & Gap & 0.143 ± 0.087 & 0.154 ± 0.093 & 0.145 ± 0.092 & 0.114 ± 0.080 \\
\multicolumn{1}{|c|}{} & Violations & 0.551 ± 0.498 & 0.364 ± 0.481 & 0.172 ± 0.378 & 0.021 ± 0.145 \\ \hline
\multicolumn{1}{|c|}{\multirow{2}{*}{\msepen}} & Gap & 0.061 ± 0.047 & 0.064 ± 0.049 & 0.065 ± 0.050 & 0.133 ± 0.082 \\
\multicolumn{1}{|c|}{} & Violations & 0.210 ± 0.407 & 0.087 ± 0.282 & 0.025 ± 0.158 & 0.025 ± 0.158 \\ \hline
\multicolumn{1}{|c|}{\multirow{2}{*}{\sslpen}} & Gap & 0.115 ± 0.065 & 0.074 ± 0.055 & 0.093 ± 0.060 & 0.123 ± 0.065 \\
\multicolumn{1}{|c|}{} & Violations & 0.060 ± 0.238 & 0.182 ± 0.386 & 0.013 ± 0.115 & 0.005 ± 0.074 \\ \hline
\multicolumn{1}{|c|}{\multirow{2}{*}{\pdl}} & Gap & 0.105 ± 0.063 & 0.126 ± 0.067 & 0.101 ± 0.062 & 0.118 ± 0.064 \\
\multicolumn{1}{|c|}{} & Violations & 0.097 ± 0.295 & 0.001 ± 0.032 & 0.003 ± 0.050 & 0.005 ± 0.071 \\ \hline
\multicolumn{1}{|c|}{\multirow{2}{*}{\sscmpe}} & Gap & 0.064 ± 0.047 & 0.072 ± 0.053 & 0.060 ± 0.046 & 0.065 ± 0.049 \\
\multicolumn{1}{|c|}{} & Violations & 0.200 ± 0.400 & 0.029 ± 0.169 & 0.000 ± 0.000 & 0.001 ± 0.032 \\ \hline
\multicolumn{1}{|c|}{\multirow{2}{*}{\sscmpepen}} & Gap & 0.067 ± 0.051 & 0.067 ± 0.051 & 0.078 ± 0.055 & 0.075 ± 0.053 \\
\multicolumn{1}{|c|}{} & Violations & 0.017 ± 0.129 & 0.019 ± 0.137 & 0.002 ± 0.039 & 0.000 ± 0.000 \\ \hline
\end{tabular}
\end{table*}


%% file: supp_tables/ad.tex
\begin{table*}[]
\caption{Average gap and constraint violations over test samples for models applied to the AD Dataset for different q values. The plot displays the mean values of the average gap and constraint violations, with standard deviations denoted by ±. }
\vspace{0.35cm}

\label{tab:ad-q}
\centering
\begin{tabular}{|cc|c|c|c|c|}
\hline
\multicolumn{2}{|c|}{q}                                               & 10            & 30            & 60            & 90            \\ \hline
\multicolumn{2}{|c|}{ILP Obj} & 2535.424      & 2526.023      & 2521.957      & 2519.917      \\ \hline
\multicolumn{1}{|c|}{\multirow{2}{*}{\mae}}         & Gap             & 0.288 ± 0.061 & 0.276 ± 0.063 & 0.283 ± 0.061 & 0.270 ± 0.063 \\ \cline{2-6} 
\multicolumn{1}{|c|}{}                              & Violations      & 0.671 ± 0.470 & 0.457 ± 0.498 & 0.257 ± 0.437 & 0.046 ± 0.210 \\ \hline
\multicolumn{1}{|c|}{\multirow{2}{*}{\mse}}         & Gap             & 0.204 ± 0.061 & 0.201 ± 0.063 & 0.204 ± 0.061 & 0.213 ± 0.059 \\ \cline{2-6} 
\multicolumn{1}{|c|}{}                              & Violations      & 0.336 ± 0.472 & 0.063 ± 0.243 & 0.069 ± 0.254 & 0.013 ± 0.111 \\ \hline
\multicolumn{1}{|c|}{\multirow{2}{*}{\mae+Penalty}} & Gap             & 0.294 ± 0.062 & 0.290 ± 0.066 & 0.277 ± 0.061 & 0.274 ± 0.061 \\ \cline{2-6} 
\multicolumn{1}{|c|}{}                              & Violations      & 0.600 ± 0.490 & 0.468 ± 0.499 & 0.271 ± 0.444 & 0.039 ± 0.194 \\ \hline
\multicolumn{1}{|c|}{\multirow{2}{*}{\msepen}}      & Gap             & 0.216 ± 0.061 & 0.213 ± 0.063 & 0.220 ± 0.061 & 0.229 ± 0.060 \\ \cline{2-6} 
\multicolumn{1}{|c|}{}                              & Violations      & 0.085 ± 0.278 & 0.041 ± 0.197 & 0.021 ± 0.142 & 0.005 ± 0.074 \\ \hline
\multicolumn{1}{|c|}{\multirow{2}{*}{\sslpen}}         & Gap             & 0.135 ± 0.055 & 0.140 ± 0.057 & 0.142 ± 0.054 & 0.134 ± 0.055 \\ \cline{2-6} 
\multicolumn{1}{|c|}{}                              & Violations      & 0.244 ± 0.430 & 0.143 ± 0.350 & 0.054 ± 0.226 & 0.005 ± 0.074 \\ \hline
\multicolumn{1}{|c|}{\multirow{2}{*}{\pdl}}         & Gap             & 0.148 ± 0.056 & 0.152 ± 0.056 & 0.146 ± 0.055 & 0.139 ± 0.054 \\ \cline{2-6} 
\multicolumn{1}{|c|}{}                              & Violations      & 0.055 ± 0.228 & 0.006 ± 0.080 & 0.013 ± 0.113 & 0.004 ± 0.063 \\ \hline
\multicolumn{1}{|c|}{\multirow{2}{*}{\sscmpe}}      & Gap             & 0.135 ± 0.055 & 0.131 ± 0.057 & 0.131 ± 0.055 & 0.131 ± 0.054 \\ \cline{2-6} 
\multicolumn{1}{|c|}{}                              & Violations      & 0.102 ± 0.302 & 0.025 ± 0.155 & 0.005 ± 0.071 & 0.003 ± 0.055 \\ \hline
\multicolumn{1}{|c|}{\multirow{2}{*}{\sscmpepen}}   & Gap             & 0.129 ± 0.054 & 0.136 ± 0.057 & 0.130 ± 0.054 & 0.133 ± 0.054 \\ \cline{2-6} 
\multicolumn{1}{|c|}{}                              & Violations      & 0.017 ± 0.127 & 0.004 ± 0.063 & 0.000 ± 0.000 & 0.000 ± 0.000 \\ \hline
\end{tabular}
\end{table*}

%% file: supp_tables/bbc.tex
\begin{table*}[]
\centering
\caption{Average gap and constraint violations over test samples for models applied to the BBC Dataset for different q values. The plot displays the mean values of the average gap and constraint violations, with standard deviations denoted by ±.}
\vspace{0.35cm}

\label{bbc-q}
\begin{tabular}{|cc|c|c|c|c|}
\hline
\multicolumn{2}{|c|}{q}                                     & 10            & 30            & 60            & 90            \\ \hline
\multicolumn{2}{|c|}{ILP Obj} & 890.289       & 880.270       & 875.668       & 872.118       \\ \hline
\multicolumn{1}{|c|}{\multirow{2}{*}{\mae}}       & Gap             & 0.053 ± 0.037 & 0.046 ± 0.034 & 0.043 ± 0.032 & 0.045 ± 0.034 \\ \cline{2-6} 
\multicolumn{1}{|c|}{}                            & Violations      & 0.779 ± 0.415 & 0.624 ± 0.485 & 0.414 ± 0.493 & 0.165 ± 0.371 \\ \hline
\multicolumn{1}{|c|}{\multirow{2}{*}{\mse}}       & Gap             & 0.036 ± 0.027 & 0.036 ± 0.028 & 0.037 ± 0.028 & 0.037 ± 0.029 \\ \cline{2-6} 
\multicolumn{1}{|c|}{}                            & Violations      & 0.924 ± 0.265 & 0.854 ± 0.354 & 0.578 ± 0.494 & 0.204 ± 0.403 \\ \hline
\multicolumn{1}{|c|}{\multirow{2}{*}{\maepen}}    & Gap             & 0.047 ± 0.036 & 0.044 ± 0.034 & 0.041 ± 0.031 & 0.040 ± 0.031 \\ \cline{2-6} 
\multicolumn{1}{|c|}{}                            & Violations      & 0.657 ± 0.475 & 0.557 ± 0.497 & 0.384 ± 0.486 & 0.151 ± 0.358 \\ \hline
\multicolumn{1}{|c|}{\multirow{2}{*}{\msepen}}    & Gap             & 0.036 ± 0.028 & 0.036 ± 0.028 & 0.037 ± 0.030 & 0.037 ± 0.029 \\ \cline{2-6} 
\multicolumn{1}{|c|}{}                            & Violations      & 0.919 ± 0.272 & 0.854 ± 0.354 & 0.578 ± 0.494 & 0.204 ± 0.403 \\ \hline
\multicolumn{1}{|c|}{\multirow{2}{*}{\sslpen}}       & Gap             & 0.041 ± 0.032 & 0.042 ± 0.032 & 0.038 ± 0.030 & 0.039 ± 0.030 \\ \cline{2-6} 
\multicolumn{1}{|c|}{}                            & Violations      & 0.516 ± 0.500 & 0.393 ± 0.488 & 0.408 ± 0.492 & 0.130 ± 0.336 \\ \hline
\multicolumn{1}{|c|}{\multirow{2}{*}{\pdl}}       & Gap             & 0.043 ± 0.033 & 0.051 ± 0.036 & 0.045 ± 0.034 & 0.044 ± 0.033 \\ \cline{2-6} 
\multicolumn{1}{|c|}{}                            & Violations      & 0.336 ± 0.472 & 0.160 ± 0.366 & 0.149 ± 0.356 & 0.029 ± 0.169 \\ \hline
\multicolumn{1}{|c|}{\multirow{2}{*}{\sscmpe}}    & Gap             & 0.038 ± 0.029 & 0.043 ± 0.032 & 0.042 ± 0.033 & 0.040 ± 0.031 \\ \cline{2-6} 
\multicolumn{1}{|c|}{}                            & Violations      & 0.316 ± 0.495 & 0.239 ± 0.427 & 0.108 ± 0.310 & 0.044 ± 0.206 \\ \hline
\multicolumn{1}{|c|}{\multirow{2}{*}{\sscmpepen}} & Gap             & 0.038 ± 0.030 & 0.043 ± 0.032 & 0.044 ± 0.033 & 0.040 ± 0.031 \\ \cline{2-6} 
\multicolumn{1}{|c|}{}                            & Violations      & 0.074 ± 0.263 & 0.067 ± 0.250 & 0.056 ± 0.229 & 0.002 ± 0.045 \\ \hline
\end{tabular}
\end{table*}

%% file: supp_tables/c20ng.tex
\begin{table*}[]
\centering
\caption{Average gap and constraint violations over test samples for models applied to the 20 Newsgroup Dataset for different q values. The plot displays the mean values of the average gap and constraint violations, with standard deviations denoted by ±.}
\vspace{0.35cm}

\label{c20ng-q}
\begin{tabular}{|cc|c|c|c|c|}
\hline
\multicolumn{2}{|c|}{q}                                     & 10            & 30            & 60            & 90            \\ \hline
\multicolumn{2}{|c|}{ILP Obj} & 928.386       & 924.439       & 923.173       & 921.754       \\ \hline
\multicolumn{1}{|c|}{\multirow{2}{*}{\mae}}       & Gap             & 0.044 ± 0.034 & 0.046 ± 0.036 & 0.047 ± 0.035 & 0.048 ± 0.037 \\ \cline{2-6} 
\multicolumn{1}{|c|}{}                            & Violations      & 0.470 ± 0.499 & 0.176 ± 0.381 & 0.049 ± 0.215 & 0.001 ± 0.022 \\ \hline
\multicolumn{1}{|c|}{\multirow{2}{*}{\mse}}       & Gap             & 0.050 ± 0.038 & 0.053 ± 0.039 & 0.051 ± 0.038 & 0.051 ± 0.037 \\ \cline{2-6} 
\multicolumn{1}{|c|}{}                            & Violations      & 0.639 ± 0.480 & 0.403 ± 0.491 & 0.142 ± 0.349 & 0.008 ± 0.089 \\ \hline
\multicolumn{1}{|c|}{\multirow{2}{*}{\maepen}}    & Gap             & 0.044 ± 0.035 & 0.047 ± 0.036 & 0.047 ± 0.036 & 0.047 ± 0.035 \\ \cline{2-6} 
\multicolumn{1}{|c|}{}                            & Violations      & 0.455 ± 0.498 & 0.181 ± 0.386 & 0.046 ± 0.210 & 0.001 ± 0.022 \\ \hline
\multicolumn{1}{|c|}{\multirow{2}{*}{\msepen}}    & Gap             & 0.046 ± 0.036 & 0.047 ± 0.035 & 0.046 ± 0.036 & 0.044 ± 0.034 \\ \cline{2-6} 
\multicolumn{1}{|c|}{}                            & Violations      & 0.573 ± 0.495 & 0.384 ± 0.486 & 0.161 ± 0.367 & 0.015 ± 0.122 \\ \hline
\multicolumn{1}{|c|}{\multirow{2}{*}{\sslpen}}    & Gap             & 0.045 ± 0.036 & 0.046 ± 0.036 & 0.045 ± 0.035 & 0.046 ± 0.035 \\ \cline{2-6} 
\multicolumn{1}{|c|}{}                            & Violations      & 0.386 ± 0.487 & 0.139 ± 0.346 & 0.024 ± 0.152 & 0.002 ± 0.039 \\ \hline
\multicolumn{1}{|c|}{\multirow{2}{*}{\pdl}}       & Gap             & 0.043 ± 0.035 & 0.046 ± 0.036 & 0.046 ± 0.036 & 0.046 ± 0.036 \\ \cline{2-6} 
\multicolumn{1}{|c|}{}                            & Violations      & 0.278 ± 0.448 & 0.129 ± 0.335 & 0.028 ± 0.165 & 0.001 ± 0.032 \\ \hline
\multicolumn{1}{|c|}{\multirow{2}{*}{\sscmpe}}    & Gap             & 0.043 ± 0.034 & 0.045 ± 0.035 & 0.044 ± 0.034 & 0.044 ± 0.034 \\ \cline{2-6} 
\multicolumn{1}{|c|}{}                            & Violations      & 0.317 ± 0.465 & 0.086 ± 0.280 & 0.019 ± 0.137 & 0.001 ± 0.032 \\ \hline
\multicolumn{1}{|c|}{\multirow{2}{*}{\sscmpepen}} & Gap             & 0.044 ± 0.033 & 0.045 ± 0.035 & 0.046 ± 0.035 & 0.045 ± 0.034 \\ \cline{2-6} 
\multicolumn{1}{|c|}{}                            & Violations      & 0.069 ± 0.254 & 0.054 ± 0.227 & 0.007 ± 0.083 & 0.001 ± 0.022 \\ \hline
\end{tabular}
\end{table*}

%% file: supp_tables/cwebkb.tex
\begin{table*}[]
\centering
\caption{Average gap and constraint violations over test samples for models applied to the Webkb Dataset for different q values. The plot displays the mean values of the average gap and constraint violations, with standard deviations denoted by ±.}
\vspace{0.35cm}

\label{cwebkb-q}
\begin{tabular}{|cc|c|c|c|c|}
\hline
\multicolumn{2}{|c|}{q}                                     & 10            & 30            & 60            & 90            \\ \hline
\multicolumn{2}{|c|}{ILP Obj} & 828.463       & 824.361       & 825.517       & 823.917       \\ \hline
\multicolumn{1}{|c|}{\multirow{2}{*}{\mae}}       & Gap             & 0.057 ± 0.043 & 0.057 ± 0.043 & 0.058 ± 0.043 & 0.062 ± 0.044 \\ \cline{2-6} 
\multicolumn{1}{|c|}{}                            & Violations      & 0.613 ± 0.487 & 0.502 ± 0.500 & 0.249 ± 0.433 & 0.046 ± 0.210 \\ \hline
\multicolumn{1}{|c|}{\multirow{2}{*}{\mse}}       & Gap             & 0.065 ± 0.046 & 0.069 ± 0.046 & 0.066 ± 0.045 & 0.065 ± 0.045 \\ \cline{2-6} 
\multicolumn{1}{|c|}{}                            & Violations      & 0.695 ± 0.461 & 0.473 ± 0.499 & 0.210 ± 0.407 & 0.018 ± 0.133 \\ \hline
\multicolumn{1}{|c|}{\multirow{2}{*}{\maepen}}    & Gap             & 0.058 ± 0.042 & 0.058 ± 0.042 & 0.059 ± 0.043 & 0.061 ± 0.044 \\ \cline{2-6} 
\multicolumn{1}{|c|}{}                            & Violations      & 0.471 ± 0.499 & 0.395 ± 0.489 & 0.211 ± 0.408 & 0.041 ± 0.197 \\ \hline
\multicolumn{1}{|c|}{\multirow{2}{*}{\msepen}}    & Gap             & 0.054 ± 0.042 & 0.054 ± 0.041 & 0.054 ± 0.040 & 0.054 ± 0.040 \\ \cline{2-6} 
\multicolumn{1}{|c|}{}                            & Violations      & 0.584 ± 0.493 & 0.378 ± 0.485 & 0.174 ± 0.380 & 0.018 ± 0.131 \\ \hline
\multicolumn{1}{|c|}{\multirow{2}{*}{\sslpen}}    & Gap             & 0.058 ± 0.043 & 0.057 ± 0.041 & 0.057 ± 0.042 & 0.057 ± 0.042 \\ \cline{2-6} 
\multicolumn{1}{|c|}{}                            & Violations      & 0.360 ± 0.480 & 0.217 ± 0.413 & 0.082 ± 0.274 & 0.015 ± 0.120 \\ \hline
\multicolumn{1}{|c|}{\multirow{2}{*}{\pdl}}       & Gap             & 0.063 ± 0.045 & 0.060 ± 0.044 & 0.058 ± 0.042 & 0.057 ± 0.042 \\ \cline{2-6} 
\multicolumn{1}{|c|}{}                            & Violations      & 0.149 ± 0.357 & 0.096 ± 0.295 & 0.056 ± 0.229 & 0.013 ± 0.115 \\ \hline
\multicolumn{1}{|c|}{\multirow{2}{*}{\sscmpe}}    & Gap             & 0.059 ± 0.043 & 0.058 ± 0.043 & 0.056 ± 0.042 & 0.056 ± 0.042 \\ \cline{2-6} 
\multicolumn{1}{|c|}{}                            & Violations      & 0.169 ± 0.374 & 0.050 ± 0.218 & 0.026 ± 0.159 & 0.004 ± 0.063 \\ \hline
\multicolumn{1}{|c|}{\multirow{2}{*}{\sscmpepen}} & Gap             & 0.062 ± 0.045 & 0.061 ± 0.044 & 0.057 ± 0.042 & 0.053 ± 0.040 \\ \cline{2-6} 
\multicolumn{1}{|c|}{}                            & Violations      & 0.074 ± 0.263 & 0.029 ± 0.169 & 0.001 ± 0.032 & 0.001 ± 0.022 \\ \hline
\end{tabular}
\end{table*}

%% file: supp_tables/dna.tex
\begin{table*}[]
\centering
\caption{Average gap and constraint violations over test samples for models applied to the DNA Dataset for different q values. The plot displays the mean values of the average gap and constraint violations, with standard deviations denoted by ±.}
\vspace{0.35cm}

\label{dna-q}
\begin{tabular}{|cc|c|c|c|c|}
\hline
\multicolumn{2}{|c|}{q}                                     & 10            & 30            & 60            & 90            \\ \hline
\multicolumn{2}{|c|}{ILP Obj} & 222.848       & 221.635       & 221.114       & 220.625       \\ \hline
\multicolumn{1}{|c|}{\multirow{2}{*}{\mae}}       & Gap             & 0.138 ± 0.109 & 0.142 ± 0.109 & 0.136 ± 0.109 & 0.141 ± 0.111 \\ \cline{2-6} 
\multicolumn{1}{|c|}{}                            & Violations      & 0.444 ± 0.497 & 0.448 ± 0.497 & 0.286 ± 0.452 & 0.114 ± 0.317 \\ \hline
\multicolumn{1}{|c|}{\multirow{2}{*}{\mse}}       & Gap             & 0.138 ± 0.112 & 0.140 ± 0.112 & 0.139 ± 0.111 & 0.139 ± 0.110 \\ \cline{2-6} 
\multicolumn{1}{|c|}{}                            & Violations      & 0.506 ± 0.500 & 0.565 ± 0.496 & 0.322 ± 0.467 & 0.113 ± 0.317 \\ \hline
\multicolumn{1}{|c|}{\multirow{2}{*}{\maepen}}    & Gap             & 0.140 ± 0.111 & 0.136 ± 0.106 & 0.143 ± 0.111 & 0.137 ± 0.113 \\ \cline{2-6} 
\multicolumn{1}{|c|}{}                            & Violations      & 0.444 ± 0.497 & 0.448 ± 0.497 & 0.286 ± 0.452 & 0.114 ± 0.317 \\ \hline
\multicolumn{1}{|c|}{\multirow{2}{*}{\msepen}}    & Gap             & 0.135 ± 0.109 & 0.140 ± 0.112 & 0.141 ± 0.111 & 0.143 ± 0.115 \\ \cline{2-6} 
\multicolumn{1}{|c|}{}                            & Violations      & 0.434 ± 0.496 & 0.494 ± 0.500 & 0.281 ± 0.450 & 0.089 ± 0.285 \\ \hline
\multicolumn{1}{|c|}{\multirow{2}{*}{\sslpen}}       & Gap             & 0.140 ± 0.115 & 0.141 ± 0.111 & 0.146 ± 0.116 & 0.143 ± 0.118 \\ \cline{2-6} 
\multicolumn{1}{|c|}{}                            & Violations      & 0.048 ± 0.214 & 0.062 ± 0.241 & 0.014 ± 0.118 & 0.004 ± 0.067 \\ \hline
\multicolumn{1}{|c|}{\multirow{2}{*}{\pdl}}       & Gap             & 0.140 ± 0.113 & 0.141 ± 0.113 & 0.139 ± 0.112 & 0.144 ± 0.120 \\ \cline{2-6} 
\multicolumn{1}{|c|}{}                            & Violations      & 0.287 ± 0.452 & 0.129 ± 0.335 & 0.003 ± 0.055 & 0.006 ± 0.077 \\ \hline
\multicolumn{1}{|c|}{\multirow{2}{*}{\sscmpe}}    & Gap             & 0.138 ± 0.113 & 0.138 ± 0.108 & 0.137 ± 0.106 & 0.139 ± 0.109 \\ \cline{2-6} 
\multicolumn{1}{|c|}{}                            & Violations      & 0.046 ± 0.210 & 0.017 ± 0.129 & 0.012 ± 0.109 & 0.008 ± 0.089 \\ \hline
\multicolumn{1}{|c|}{\multirow{2}{*}{\sscmpepen}} & Gap             & 0.139 ± 0.116 & 0.139 ± 0.113 & 0.140 ± 0.112 & 0.139 ± 0.113 \\ \cline{2-6} 
\multicolumn{1}{|c|}{}                            & Violations      & 0.013 ± 0.115 & 0.002 ± 0.045 & 0.001 ± 0.022 & 0.001 ± 0.022 \\ \hline
\end{tabular}
\end{table*}

%% file: arXiV_main.bbl
\begin{thebibliography}{}

\bibitem[Achterberg, 2009]{achterberg2009scip}
Achterberg, T. (2009).
\newblock Scip: solving constraint integer programs.
\newblock {\em Mathematical Programming Computation}, 1:1--41.

\bibitem[Achterberg et~al., 2008]{achterberg2008constraint}
Achterberg, T., Berthold, T., Koch, T., and Wolter, K. (2008).
\newblock Constraint integer programming: A new approach to integrate cp and mip.
\newblock In {\em Integration of AI and OR Techniques in Constraint Programming for Combinatorial Optimization Problems: 5th International Conference, CPAIOR 2008 Paris, France, May 20-23, 2008 Proceedings 5}, pages 6--20. Springer.

\bibitem[Choi and Darwiche, 2011]{choi_relax_2010}
Choi, A. and Darwiche, A. (2011).
\newblock Relax, compensate and then recover.
\newblock In Onada, T., Bekki, D., and McCready, E., editors, {\em New Frontiers in Artificial Intelligence}, pages 167--180, Berlin, Heidelberg. Springer Berlin Heidelberg.

\bibitem[Choi et~al., 2012]{choi_same-decision_2012}
Choi, A., Xue, Y., and Darwiche, A. (2012).
\newblock Same-decision probability: {A} confidence measure for threshold-based decisions.
\newblock {\em International Journal of Approximate Reasoning}, 53(9):1415--1428.

\bibitem[Choi et~al., 2020]{choi2020probabilistic}
Choi, Y., Vergari, A., and Van~den Broeck, G. (2020).
\newblock Probabilistic circuits: A unifying framework for tractable probabilistic models.
\newblock {\em UCLA. URL: http://starai. cs. ucla. edu/papers/ProbCirc20. pdf}.

\bibitem[Cui et~al., 2022]{cui2022variational}
Cui, Z., Wang, H., Gao, T., Talamadupula, K., and Ji, Q. (2022).
\newblock Variational message passing neural network for maximum-a-posteriori (map) inference.
\newblock In {\em Uncertainty in Artificial Intelligence}, pages 464--474. PMLR.

\bibitem[Darwiche, 2009]{darwiche09}
Darwiche, A. (2009).
\newblock {\em {Modeling and Reasoning with Bayesian Networks}}.
\newblock Cambridge University Press.

\bibitem[Darwiche and Hirth, 2020]{darwiche_reasons_2020}
Darwiche, A. and Hirth, A. (2020).
\newblock On the reasons behind decisions.
\newblock In {\em Twenty Fourth European Conference on Artificial Intelligence}, volume 325 of {\em Frontiers in Artificial Intelligence and Applications}, pages 712--720. {IOS} Press.

\bibitem[Darwiche and Hirth, 2023]{darwiche2023complete}
Darwiche, A. and Hirth, A. (2023).
\newblock On the (complete) reasons behind decisions.
\newblock {\em Journal of Logic, Language and Information}, 32(1):63--88.

\bibitem[Dechter, 1999]{dechter99}
Dechter, R. (1999).
\newblock {Bucket elimination: A unifying framework for reasoning}.
\newblock {\em Artificial Intelligence}, 113:41--85.

\bibitem[Dechter and Rish, 2003]{dechter2003mini}
Dechter, R. and Rish, I. (2003).
\newblock Mini-buckets: A general scheme for bounded inference.
\newblock {\em Journal of the ACM (JACM)}, 50(2):107--153.

\bibitem[Donti et~al., 2021]{donti2021dc3}
Donti, P.~L., Rolnick, D., and Kolter, J.~Z. (2021).
\newblock Dc3: A learning method for optimization with hard constraints.
\newblock {\em arXiv preprint arXiv:2104.12225}.

\bibitem[Elidan and Globerson, 2010]{elidan_2010_2010}
Elidan, G. and Globerson, A. (2010).
\newblock {\em The 2010 {UAI} {Approximate} Inference Challenge}.
\newblock Published: Available online at: http://www.cs.huji.ac.il/project/UAI10/index.php.

\bibitem[Fioretto et~al., 2020]{fioretto2020predicting}
Fioretto, F., Mak, T.~W., and Van~Hentenryck, P. (2020).
\newblock Predicting ac optimal power flows: Combining deep learning and lagrangian dual methods.
\newblock In {\em Proceedings of the AAAI conference on artificial intelligence}, volume~34, pages 630--637.

\bibitem[Gilmer et~al., 2017]{gilmer2017neural}
Gilmer, J., Schoenholz, S.~S., Riley, P.~F., Vinyals, O., and Dahl, G.~E. (2017).
\newblock Neural message passing for quantum chemistry.
\newblock In {\em International conference on machine learning}, pages 1263--1272. PMLR.

\bibitem[Globerson and Jaakkola, 2007]{globerson2007fixing}
Globerson, A. and Jaakkola, T. (2007).
\newblock Fixing max-product: Convergent message passing algorithms for map lp-relaxations.
\newblock {\em Advances in neural information processing systems}, 20:553--560.

\bibitem[Gogate, 2014]{uai14competition}
Gogate, V. (2014).
\newblock Results of the 2014 {UAI} competition.
\newblock \url{https://personal.utdallas.edu/~vibhav.gogate/uai14-competition/index.html}.

\bibitem[Gogate, 2016]{uai16competition}
Gogate, V. (2016).
\newblock Results of the 2016 {UAI} competition.
\newblock \url{https://personal.utdallas.edu/~vibhav.gogate/uai16-competition/index.html}.

\bibitem[Gurobi~Optimization, 2021]{gurobi2021gurobi}
Gurobi~Optimization, L. (2021).
\newblock Gurobi optimizer reference manual.

\bibitem[Horst and Tuy, 1996]{horst1996global}
Horst, R. and Tuy, H. (1996).
\newblock {\em Global Optimization: Deterministic Approaches}.
\newblock Springer Berlin Heidelberg.

\bibitem[Ihler et~al., 2012]{ihler2012join}
Ihler, A.~T., Flerova, N., Dechter, R., and Otten, L. (2012).
\newblock Join-graph based cost-shifting schemes.
\newblock {\em arXiv preprint arXiv:1210.4878}.

\bibitem[Kingma and Ba, 2014]{Kingma2014Dec}
Kingma, D.~P. and Ba, J. (2014).
\newblock {Adam: A Method for Stochastic Optimization}.
\newblock {\em arXiv}.

\bibitem[Koller and Friedman, 2009]{koller2009probabilistic}
Koller, D. and Friedman, N. (2009).
\newblock {\em Probabilistic graphical models: principles and techniques}.
\newblock MIT press.

\bibitem[Komodakis et~al., 2007]{komodakis2007mrf}
Komodakis, N., Paragios, N., and Tziritas, G. (2007).
\newblock Mrf optimization via dual decomposition: Message-passing revisited.
\newblock In {\em 2007 IEEE 11th International Conference on Computer Vision}, pages 1--8. IEEE.

\bibitem[Kotary et~al., 2021]{kotary2021learning}
Kotary, J., Fioretto, F., and Hentenryck, P.~V. (2021).
\newblock Learning hard optimization problems: A data generation perspective.
\newblock {\em arXiv preprint arXiv: Arxiv-2106.02601}.

\bibitem[Kuck et~al., 2020]{kuck2020belief}
Kuck, J., Chakraborty, S., Tang, H., Luo, R., Song, J., Sabharwal, A., and Ermon, S. (2020).
\newblock Belief propagation neural networks.
\newblock {\em Advances in Neural Information Processing Systems}, 33:667--678.

\bibitem[LeCun and Cortes, 2010]{lecun-mnisthandwrittendigit-2010}
LeCun, Y. and Cortes, C. (2010).
\newblock {MNIST} handwritten digit database.

\bibitem[Liu and Cherian, 2023]{Liu2023jan}
Liu, T. and Cherian, A. (2023).
\newblock Learning a constrained optimizer: A primal method.
\newblock In {\em AAAI 2023 Bridge on Constraint Programming and Machine Learning}.

\bibitem[Lowd and Davis, 2010]{Lowd2010Dec_1}
Lowd, D. and Davis, J. (2010).
\newblock {Learning Markov Network Structure with Decision Trees}.
\newblock In {\em {2010 IEEE International Conference on Data Mining}}, pages 334--343. IEEE.

\bibitem[Marinescu and Dechter, 2009]{marinescu&dechter09}
Marinescu, R. and Dechter, R. (2009).
\newblock {Memory intensive AND/OR search for combinatorial optimization in graphical models.}
\newblock {\em AI Journal}, 173(16-17):1492--1524.

\bibitem[Marinescu and Dechter, 2012]{marinescu&dechter12}
Marinescu, R. and Dechter, R. (2012).
\newblock {Best-First AND/OR Search for Most Probable Explanations}.
\newblock {\em CoRR}, abs/1206.5268.

\bibitem[Nellikkath and Chatzivasileiadis, 2021]{nellikkath2021physicsinformed}
Nellikkath, R. and Chatzivasileiadis, S. (2021).
\newblock Physics-informed neural networks for ac optimal power flow.
\newblock {\em arXiv preprint arXiv: Arxiv-2110.02672}.

\bibitem[Nocedal and Wright, 2006]{nocedal}
Nocedal, J. and Wright, S.~J. (2006).
\newblock {\em Numerical Optimization}.
\newblock Springer, New York, NY, USA, 2e edition.

\bibitem[Park and Van~Hentenryck, 2022]{park2022self}
Park, S. and Van~Hentenryck, P. (2022).
\newblock Self-supervised primal-dual learning for constrained optimization.
\newblock {\em arXiv preprint arXiv:2208.09046}.

\bibitem[Rahman et~al., 2014]{rahman2014cutset}
Rahman, T., Kothalkar, P., and Gogate, V. (2014).
\newblock Cutset networks: A simple, tractable, and scalable approach for improving the accuracy of chow-liu trees.
\newblock In {\em Machine Learning and Knowledge Discovery in Databases: European Conference, ECML PKDD 2014, Nancy, France, September 15-19, 2014. Proceedings, Part II 14}, pages 630--645. Springer.

\bibitem[Rahman et~al., 2021]{rahman2021novel}
Rahman, T., Rouhani, S., and Gogate, V. (2021).
\newblock Novel upper bounds for the constrained most probable explanation task.
\newblock {\em Advances in Neural Information Processing Systems}, 34:9613--9624.

\bibitem[Rouhani et~al., 2018]{rouhaniRG18}
Rouhani, S., Rahman, T., and Gogate, V. (2018).
\newblock Algorithms for the nearest assignment problem.
\newblock In {\em Proceedings of the Twenty-Seventh International Joint Conference on Artificial Intelligence, {IJCAI} 2018, July 13-19, 2018, Stockholm, Sweden}, pages 5096--5102. ijcai.org.

\bibitem[Rouhani et~al., 2020]{rouhani2020novel}
Rouhani, S., Rahman, T., and Gogate, V. (2020).
\newblock A novel approach for constrained optimization in graphical models.
\newblock {\em Advances in Neural Information Processing Systems}, 33:11949--11960.

\bibitem[Satorras and Welling, 2021]{satorras2021neural}
Satorras, V.~G. and Welling, M. (2021).
\newblock Neural enhanced belief propagation on factor graphs.
\newblock In {\em International Conference on Artificial Intelligence and Statistics}, pages 685--693. PMLR.

\bibitem[Sherali and Adams, 2009]{Sherali&Adams09}
Sherali, H.~D. and Adams, W.~P. (2009).
\newblock A reformulation-linearization technique (rlt) for semi-infinite and convex programs under mixed 0-1 and general discrete restrictions.
\newblock {\em Discrete Applied Mathematics}, 157(6):1319--1333.
\newblock Reformulation Techniques and Mathematical Programming.

\bibitem[Sherali and Tuncbilek, 1992]{sherali1992global}
Sherali, H.~D. and Tuncbilek, C.~H. (1992).
\newblock A global optimization algorithm for polynomial programming problems using a reformulation-linearization technique.
\newblock {\em Journal of Global Optimization}, 2:101--112.

\bibitem[Ucla-Starai, 2023]{Ucla-Starai2023May}
Ucla-Starai (2023).
\newblock {Density-Estimation-Datasets}.
\newblock [Online; accessed 17. May 2023].

\bibitem[Van~Haaren and Davis, 2012]{VanHaaren2012_2}
Van~Haaren, J. and Davis, J. (2012).
\newblock {Markov Network Structure Learning: A Randomized Feature Generation Approach}.
\newblock {\em AAAI}, 26(1):1148--1154.

\bibitem[Wainwright et~al., 2005]{wainwright2005map}
Wainwright, M.~J., Jaakkola, T.~S., and Willsky, A.~S. (2005).
\newblock Map estimation via agreement on trees: message-passing and linear programming.
\newblock {\em IEEE transactions on information theory}, 51(11):3697--3717.

\bibitem[Wu et~al., 2020]{wu2020map}
Wu, B., Shen, L., Zhang, T., and Ghanem, B. (2020).
\newblock Map inference via l2 sphere linear program reformulation.
\newblock {\em International Journal of Computer Vision}, 128(7):1913--1936.

\bibitem[Zamzam and Baker, 2019]{zamzam2019learning}
Zamzam, A. and Baker, K. (2019).
\newblock Learning optimal solutions for extremely fast ac optimal power flow.
\newblock {\em arXiv preprint arXiv: Arxiv-1910.01213}.

\bibitem[Zhang et~al., 2020]{zhang2020factor}
Zhang, Z., Wu, F., and Lee, W.~S. (2020).
\newblock Factor graph neural networks.
\newblock {\em Advances in Neural Information Processing Systems}, 33:8577--8587.

\end{thebibliography}
